\documentclass{article}
\usepackage{multirow}
\usepackage{subfigure}
\usepackage[utf8]{inputenc}
\usepackage{amsmath, amsthm}
\usepackage{hyperref}
\usepackage{bbm}
\usepackage{url}
\usepackage{amssymb}
\usepackage{wrapfig}
\usepackage[overload]{empheq}
\usepackage{cleveref} 
\usepackage{CJK}
\usepackage{indentfirst}
\usepackage{mathtools}
\usepackage{amsmath}
\usepackage{amssymb}
\usepackage{mathtools}
\usepackage{amsthm}
\usepackage{comment}
\usepackage{color}
\usepackage{algorithm}
\usepackage{algorithmicx}
\usepackage{algpseudocode}
\usepackage{braket}
\usepackage{relsize}
\usepackage{adjustbox}
\usepackage{lmodern}
\theoremstyle{plain}
\usepackage{pifont}
\usepackage{float}
\usepackage{enumitem}
\usepackage{natbib}

\usepackage{microtype}
\usepackage{graphicx}
\usepackage{subfigure}
\usepackage{booktabs} 
\usepackage{tikz}

\usepackage[final]{neurips_2025}

\usepackage[utf8]{inputenc} 
\usepackage[T1]{fontenc}    
\usepackage{hyperref}       
\usepackage{url}            
\usepackage{booktabs}       
\usepackage{amsfonts}       
\usepackage{nicefrac}       
\usepackage{microtype}      
\usepackage{xcolor}         

\theoremstyle{plain}
\newtheorem{theorem}{Theorem}[section]

\newtheorem{lemma}[theorem]{Lemma}

\theoremstyle{definition}
\newtheorem{definition}[theorem]{Definition}
\newtheorem{assumption}[theorem]{Assumption}
\theoremstyle{remark}

\newcommand{\norm}[1]{\left \lVert #1 \right\rVert }
\newcommand{\bignorm}[1]{\left\lVert#1\right\rVert}

\newcommand{\abs}[1]{\left | #1 \right | }

\newcommand{\E}{\mathbb{E}}
\newcommand{\kp}{\mathsf P}
\newcommand{\cp}{\mathcal{P}}
\newcommand{\e}{\mathbf e}
\newcommand{\mcs}{\mathcal{S}}
\newcommand{\mca}{\mathcal{A}}

\newcommand{\mE}{\mathbb{E}}

\DeclareMathOperator*{\cO}{\mathcal{O}}

\title{Finite-Sample Analysis of Policy Evaluation for \\
Robust Average Reward Reinforcement Learning}

\author{%
  Yang Xu  \\
  Purdue University\\
  West Lafayette, IN 47907, USA \\
  \texttt{xu1720@purdue.edu} \\
  \And
  Washim Uddin Mondal \\
  Indian Institute of Technology Kanpur \\
  Kanpur, UP, India 208016 \\
  \texttt{wmondal@iitk.ac.in} \\
  \AND
  Vaneet Aggarwal \\
  Purdue University\\
  West Lafayette, IN 47907, USA \\
  \texttt{vaneet@purdue.edu} \\
}

\begin{document}

\maketitle

\begin{abstract}
We present the first finite-sample analysis of policy evaluation in robust average-reward Markov Decision Processes (MDPs). Prior work in this setting have established only asymptotic convergence guarantees, leaving open the question of sample complexity. In this work, we address this gap by showing that the robust Bellman operator is a contraction under a carefully constructed semi-norm, and developing a stochastic approximation framework with controlled bias. Our approach builds upon Multi-Level Monte Carlo (MLMC) techniques to estimate the robust Bellman operator efficiently. To overcome the infinite expected sample complexity inherent in standard MLMC, we introduce a truncation mechanism based on a geometric distribution, ensuring a finite expected sample complexity while maintaining a small bias that decays exponentially with the truncation level. Our method achieves the order-optimal sample complexity of $\tilde{\mathcal{O}}(\epsilon^{-2})$ for robust policy evaluation and robust average reward estimation, marking a significant advancement in robust reinforcement learning theory.
\end{abstract}

\section{Introduction}

Reinforcement learning (RL) has achieved notable success in domains such as robotics \cite{tang2025deep}, finance \citep{hambly2023recent}, healthcare \citep{yu2021reinforcement}, transportation \cite{al2019deeppool}, and large language models \citep{srivastava2025technical} by enabling agents to learn optimal decision-making strategies through interaction with an environment. However, in many real-world applications, direct interaction is impractical due to safety concerns, high costs, or limited data  collection budgets \cite{sunderhauf2018limits, hofer2021sim2real}. This challenge is particularly evident in scenarios where agents are trained in simulated environments before being deployed in the real world, such as in robotic control and autonomous driving. The mismatch between simulated and real environments, known as the simulation-to-reality gap, often leads to performance degradation when the learned policy encounters unmodeled uncertainties. Robust reinforcement learning (robust RL) addresses this challenge by formulating the learning problem as an optimization over an uncertainty set of transition probabilities, ensuring reliable performance under worst-case conditions. In this work, we focus on the problem of evaluating the robust value function and robust average reward for a given policy using only data sampled from a simulator (nominal model), aiming to enhance generalization and mitigate the impact of transition uncertainty in real-world deployment. 

Reinforcement learning problems under infinite time horizons are typically studied under two primary reward formulations: the discounted-reward setting, where future rewards are exponentially discounted, and the average-reward setting, which focuses on optimizing long-term performance. While the discounted-reward formulation is widely used, it may lead to myopic policies that underperform in applications requiring sustained long-term efficiency, such as queueing systems, inventory management, and network control. In contrast, the average-reward setting is more suitable for environments where decisions impact long-term operational efficiency. Despite its advantages, robust reinforcement learning under the average-reward criterion remains largely unexplored. Existing works on robust average-reward RL primarily provide asymptotic guarantees \cite{wang2023robust,wang2023model, wang2024robust}, lacking algorithms with finite-time performance bounds. This gap highlights the need for principled approaches that ensure robustness against model uncertainties while maintaining strong long-term performance guarantees.

Solving the robust average-reward reinforcement learning problem is significantly more challenging than its non-robust counterpart, with the primary difficulty arising in policy evaluation. Specifically, the goal is to compute the worst-case value function and worst-case average reward over an entire uncertainty set of transition models while having access only to samples from a nominal transition model. In this paper, we investigate three types of uncertainty sets: Contamination uncertainty sets, total variation (TV) distance uncertainty sets, and Wasserstein distance uncertainty sets. Unlike the standard average-reward setting, where value functions and average rewards can be estimated directly from observed trajectories \citep{wei2020model,bai2024regret,ganesh2025regret,ganesh2025order,ganesh2025sharper}, the robust setting introduces an additional layer of complexity due to the need to optimize against adversarial transitions. Consequently, conventional approaches based on direct estimation such as \cite{wei2020model,bai2024regret,ganesh2025regret,ganesh2025order,ganesh2025sharper} immediately fail, as they do not account for the worst-case nature of the problem. Overcoming this challenge requires new algorithmic techniques that can infer the worst-case dynamics using only limited samples from the nominal model.

\subsection{Challenges and Contributions}

A common approach to policy evaluation in robust RL is to solve the corresponding robust Bellman operator. However, robust average-reward RL presents additional difficulties compared to the robust discounted-reward setting. In the discounted case, the presence of a discount factor induces a contraction property in the robust Bellman operator \cite{wang2022policy, zhou2024natural}, facilitating stable iterative updates. In contrast, the average-reward Bellman operator lacks a contraction property with respect to any norm even in the non-robust setting \cite{zhang2021finite}, making standard fixed-point analysis inapplicable. Due to this fundamental limitation, existing works on robust average-reward RL such as  \cite{wang2023model} rely on asymptotic techniques, primarily leveraging ordinary differential equation (ODE) analysis to examine the behavior of temporal difference (TD) learning. These methods exploit the asymptotic stability of the corresponding ODE \cite{borkar2023stochastic} to establish almost sure convergence but fail to provide finite-sample performance guarantees. Addressing this limitation requires novel analytical tools and algorithmic techniques capable of providing explicit finite-sample bounds for robust policy evaluation and optimization.

In this work, we first establish and exploit a key structural property of the robust average-reward Bellman operator with uncertainty set $\cp$ under the ergodicity of the nominal model: it is a contraction under some semi-norm, denoted as $\|\cdot\|_{\cp}$, where the detailed construction is specified in Theorem \ref{thm:robust_seminorm-contraction} and \eqref{eq:psemeinormbrief}. Constructing $\|\cdot\|_{\cp}$ is not straightforward, because ergodicity alone only guarantees that the chain mixes over multiple steps and fails to produce a single‐step contraction for familiar measures such as the span semi-norm. To overcome this, we group together all the worst‐case transition dynamics under uncertainty into one compact family of linear mappings, and observe that their “worst‐case gain” over any number of steps stays strictly below $1$.  From this we build an extremal norm, which by construction shrinks every non‐constant component by the same fixed factor in a single step.  Finally, we add a small “quotient” correction that exactly annihilates constant shifts, producing a semi-norm that vanishes only on constant functions but still inherits the one‐step shrinkage. The above construction yields a uniform, strict contraction for the robust Bellman operator.

This fundamental result above enables the use of stochastic approximation techniques similar to \cite{zhang2021finite} to analyze and bound the error in policy evaluation, overcoming the lack of a standard contraction property that has hindered prior finite-sample analyses. Building on this insight, we develop a novel stochastic approximation framework tailored to the robust average-reward setting. Our approach simultaneously estimates both the robust value function and the robust average reward, leading to an efficient iterative procedure for solving the robust Bellman equation. A critical challenge in this framework under TV and Wasserstein distance uncertainty sets is accurately estimating the worst-case transition effects, which requires computing the support function of the uncertainty set. While previous works \cite{blanchet2015unbiased,blanchet2019unbiased, wang2023model} have leveraged Multi-Level Monte Carlo (MLMC) for this task, their MLMC-based estimators suffer from infinite expected sample complexity due to the unbounded nature of the required geometric sampling, leading to only asymptotic convergence. To address this, we introduce a truncation mechanism based on a truncated geometric distribution, ensuring that the sample complexity remains finite while maintaining an exponentially decaying bias. With these techniques, we derive the first finite-sample complexity guarantee for policy evaluation in robust average-reward RL, achieving an optimal $\tilde{\mathcal{O}}(\epsilon^{-2})$ sample complexity bound. The main contributions of this paper are summarized as
follows:

\begin{itemize}[leftmargin=*]
    \item We prove that under the ergodicity assumption of the nominal model, the robust average-reward Bellman operator is a contraction with respect to a suitably constructed semi-norm (Theorem \ref{thm:robust_seminorm-contraction}). This key result enables the application of stochastic approximation techniques for policy evaluation.
    
    \item We prove the convergence of stochastic approximation under the semi-norm contraction and under i.i.d. with noise with non-zero bias (Theorem \ref{thm:informalbiasedSA}) as an intermediate result.
    
    \item We develop an efficient method for computing estimates for the robust Bellman operator under TV distance and Wasserstein distance uncertainty sets. By modifying MLMC with a truncated geometric sampling scheme, we ensure finite expected sample complexity while keeping variance controlled and bias decaying exponentially with truncation level (Theorem \ref{thm:sample-complexity}-\ref{thm:linear-variance}).
    
    \item We propose a novel temporal difference learning method that iteratively updates the robust value function and the robust average reward, facilitating efficient policy evaluation in robust average-reward RL. We establish the first non-asymptotic sample complexity result for policy evaluation in robust average-reward RL, proving an order-optimal $\tilde{\mathcal{O}}(\epsilon^{-2})$ complexity for policy evaluation (Theorem \ref{thm:Vresult}), along with a $\tilde{\mathcal{O}}(\epsilon^{-2})$ complexity for robust average-reward estimation (Theorem \ref{thm:gresult}).
\end{itemize}

\section{Related Work}
The theoretical guarantees of robust average-reward reinforcement learning have been studied by the following works. \cite{wang2023robust} takes a model-based perspective, approximating robust average-reward MDPs with discounted MDPs and proving uniform convergence of the robust discounted value function as the discount factor approaches one, employing dynamic programming and Blackwell optimality arguments to characterize optimal policies. \cite{wang2023model} proposes a model-free approach by developing robust relative value iteration (RVI) TD and Q-learning algorithms, proving their almost sure convergence using stochastic approximation, martingale theory, and Multi-Level Monte Carlo estimators to handle non-linearity in the robust Bellman operator. While these studies provide fundamental insights into robust average-reward RL, they do not establish explicit convergence rate guarantees due to the lack of contraction properties in the robust Bellman operator. In addition, \cite{sunpolicy2024, wang2025provable} study the policy optimization of average-reward robust MDPs assuming direct queries of the sub-gradient information.

Policy evaluation in robust discounted-reward reinforcement learning with finite sample guarantees has been extensively studied, with the key recent works \cite{wang2022policy, zhou2024natural, li2022first, kumar2023policy, kuang2022learning} focusing on solving the robust Bellman equation by finding its fixed-point solution. This approach is made feasible by the contraction property of the robust Bellman operator under the sup-norm, which arises due to the presence of a discount factor $\gamma < 1$. However, this fundamental approach does not directly extend to the robust average-reward setting, where the absence of a discount factor removes the contraction property under any norm. As a result, existing robust discounted methods cannot be applied in the robust average-reward RL setting.

Recently, a growing body of concurrent work has established finite-sample guarantees for robust average-reward reinforcement learning. Model-based approaches include \citep{roch2025reduction, chen2025sample}, and \citep{roch2025finite} develops a model-free value-iteration method under contamination and $\ell_p$-ball uncertainty sets. While these results significantly advance the area, the specific problem of policy evaluation in robust average-reward MDPs has not yet been addressed in terms of sample complexity. Our work targets this gap.

\section{Formulation}

\subsection{Robust average-reward MDPs.} \label{sec:ramdp}
For a robust MDP with state space $\mcs$ and action space $\mca$ while  $|\mcs|=S$ and $|\mca|=A$, the transition kernel is assumed to be in some uncertainty set $\mathcal{P}$. At each time step, the environment transits to the next state according to an arbitrary transition kernel $\kp\in\cp$. In this paper, we focus on the $(s,a)$-rectangular compact uncertainty set \cite{nilim2004robustness,iyengar2005robust}, i.e., $\mathcal{P}=\bigotimes_{s,a} \mathcal{P}^a_s$, where $\mathcal{P}^a_s \subseteq \Delta(\mcs)$, and $\Delta$ denotes the  probability simplex. Popular uncertainty sets include those defined by the contamination model \cite{hub65,wang2022policy},  total variation \cite{lim2013reinforcement}, and Wasserstein distance \cite{gao2023distributionally}.

We investigate the worst-case average-reward over the uncertainty set of MDPs. Specifically, define the  robust average-reward of a policy $\pi$ as 
\begin{align}\label{eq:Vdef}
    g^\pi_\cp(s)\triangleq \min_{\kappa\in\bigotimes_{n\geq 0} \mathcal{P}} \lim_{T\to\infty}\mathbb{E}_{\pi,\kappa}\left[\frac{1}{T}\sum^{T-1}_{t=0}r_t|S_0=s\right],
\end{align}
where $\kappa=(\mathsf P_0,\mathsf P_1...)\in\bigotimes_{n\geq 0} \mathcal{P}$. It was shown in \cite{wang2023robust} that the worst case under the time-varying model is equivalent to the one under the stationary model:
\begin{align}\label{eq:5}
    g^\pi_\cp(s)= \min_{\kp\in\mathcal{P}} \lim_{T\to\infty}\mathbb{E}_{\pi,\kp}\left[\frac{1}{T}\sum^{T-1}_{t=0}r_t|S_0=s\right].
\end{align}
Therefore, we limit our focus to the stationary model. We refer to the minimizers of \eqref{eq:5} as the worst-case transition kernels for the policy $\pi$, and denote the set of all possible worst-case transition kernels by $\Omega^\pi_g$, i.e., $\Omega^\pi_g \triangleq \{\kp\in\cp: g^\pi_\kp=g^\pi_\cp \}$, where $g^\pi_\kp$ denotes the average reward of policy $\pi$ under the single transition $\kp\in\cp$:
\begin{align} \label{eq:nonrobustg}
    g_\kp^\pi(s)\triangleq \lim_{T\to\infty} \mE_{\pi,\kp}\bigg[\frac{1}{T}\sum^{T-1}_{n=0} r_t|S_0=s \bigg].
\end{align}

We focus on the model-free setting, where only samples from the nominal MDP denoted as $\tilde{\kp}$ (the centroid of the uncertainty set) are available. We investigate the problem of robust policy evaluation and robust average reward estimation, which means for a given policy $\pi$, we aim to estimate the robust value function and the robust average reward. Throughout this paper, we make the following standard assumption regarding the structure of the induced Markov chain.

\begin{assumption}\label{ass:sameg}
    The Markov chain induced by $\pi$ is irreducible and aperiodic for the nominal model $\tilde{\kp}$. 
\end{assumption}

In contrast to many current works on robust average-reward RL \cite{wang2023robust, wang2023model, wang2024robust,sunpolicy2024,roch2025finite}, Assumption \ref{ass:sameg} requires only that the center of the uncertainty set be irreducible and aperiodic. We note that when the radius of uncertainty sets is small enough, Assumption \ref{ass:sameg} can ensure that $\kp^\pi$ is irreducible and aperiodic for all $\kp \in \cp$. This ensures that, under any transition model within the uncertainty set, the policy $\pi$ induces a single recurrent communicating class. A well-known result in average-reward MDPs states that under Assumption \ref{ass:sameg}, the average reward is independent of the starting state, i.e., for any $\kp\in\cp$ and all $s,s' \in \mcs$, we have  $ g^\pi_\kp(s) = g^\pi_\kp(s')$. Thus, we can drop the dependence on the initial state and simply write $g^\pi_\kp$ as the robust average reward. We now formally define the robust value function $ V^\pi_{\kp_V}$ by connecting it with the following robust Bellman equation: 

\begin{theorem}[Robust Bellman Equation, Theorem 3.1 in \cite{wang2023model}]\label{thm:robust Bellman} 
If $(g,V)$ is a solution to the robust Bellman equation
\begin{equation}\label{eq:bellman}
    V(s) = \sum_{a} \pi(a|s) \big(r(s,a) - g + \sigma_{\cp^a_s}(V) \big), \quad \forall s \in \mathcal{S},
\end{equation}
where $\sigma_{\cp^a_s}(V) = \min_{p\in\cp^a_s} p^\top V$ is denoted as the support function, then the scalar $g$ corresponds to the robust average reward, i.e., $g = g^\pi_\cp$, and the worst-case transition kernel $\kp_V$ belongs to the set of minimizing transition kernels, i.e., $\kp_V \in \Omega^\pi_g$, where 
$\Omega^\pi_g \triangleq \{ \kp \in \cp : g^\pi_\kp = g^\pi_\cp \} $. Furthermore, the function $V$ is unique up to an additive constant, where if $V$ is a solution to the Bellman equation, then we have $ V = V^\pi_{\kp_V} + c \mathbf{e}$,    where $c \in \mathbb{R}$ and $\mathbf{e}$ is the all-ones vector in $\mathbb{R}^{S}$, and $ V^\pi_{\kp_V}$ is defined as the relative value function of the policy $\pi$ under the single transition $\kp_V$ as follows:
\begin{align}\label{eq:relativevaluefunction}
    V^\pi_{\kp_V}(s)\triangleq \mE_{\pi,\kp_V}\bigg[\sum^\infty_{t=0} (r_t-g^\pi_{\kp_V})|S_0=s \bigg].
\end{align}

\end{theorem}
Theorem \ref{thm:robust Bellman} implies that the robust Bellman equation \eqref{eq:bellman} identifies both the worst‐case average reward $g$ and a corresponding value function $V$ that is determined only up to an additive constant. In particular, $\sigma_{\cp^a_s}(V)$ represents the worst-case transition effect over the uncertainty set $\cp^a_s$. Unlike the robust discounted case, where the contraction property of the Bellman operator under the sup-norm enables straightforward fixed-point iteration, the robust average-reward Bellman equation does not induce contraction under any norm, making direct iterative methods inapplicable. Throughout the paper, we denote $\e$ as the all-ones vector in $\mathbb{R}^S$.  We now characterize the explicit forms of $\sigma_{\cp^a_s}(V)$ for different compact uncertainty sets as follows:

\noindent \textbf{Contamination Uncertainty Set}\label{sec:con}
The contamination uncertainty models outliers or rare faults \citep{chen2016general}. Specifically, the $\delta$-contamination uncertainty set is
$
    \cp^a_s=\{(1-\delta)\tilde{\kp}^a_s+\delta q: q\in\Delta(\mcs) \}, 
$
where $0<\delta<1$ is the radius. Under this uncertainty set, the support function can be computed as 
\begin{equation}\label{eq:contamination}
    \sigma_{\cp^a_s}(V)=(1-\delta)(\tilde{\kp}^a_s)^\top V+\delta \min_s V(s),
\end{equation}
and this is linear in the nominal transition kernel $\tilde{\kp^a_s}$. 

\noindent \textbf{Total Variation Uncertainty Set.}
The total variation (TV) distance uncertainty set models categorical misspecification or discretization error \cite{ho2021partial}, and is characterized as
$
    \cp^a_s=\{q\in\Delta(|\mcs|): \frac{1}{2}\|q-\tilde{\kp}^a_s\|_1\leq \delta \},
$
define $\| \cdot \|_{\mathrm{sp}}$ as the span semi-norm and the support function can be computed using its dual function \cite{iyengar2005robust}: 
\begin{align}\label{eq:tv dual}
    \sigma_{\cp^a_s}(V)=\max_{\mu \geq \mathbf{0}}\big((\tilde{\kp}^a_s)^\top (V-\mu)-\delta \|V-\mu\|_{\mathrm{sp}}  \big).
\end{align}
\textbf{Wasserstein Distance Uncertainty Sets.}
The Wasserstein distance uncertainty Models smooth model drift when states have a geometry \citep{clement2021first}. Consider the metric space $(\mathcal{S},d)$ by defining some distance metric $d$. For some parameter $l\in[1,\infty)$ and two distributions $p,q\in\Delta(\mathcal{S})$, define the $l$-Wasserstein distance between them as 
$W_l(q,p)=\inf_{\mu\in\Gamma(p,q)}\|d\|_{\mu,l}$, where $\Gamma(p,q)$ denotes the distributions over $\mathcal{S}\times\mathcal{S}$ with marginal distributions $p,q$, and $\|d\|_{\mu,l}=\big(\mE_{(X,Y)\sim \mu}\big[d(X,Y)^l\big]\big)^{1/l}$. The Wasserstein distance uncertainty set is then defined as 
\begin{equation}
    \cp^a_s=\left\{q\in\Delta(\mcs): W_l(\tilde{\kp}^a_s,q)\leq \delta \right\}.
\end{equation}
The support function w.r.t. the Wasserstein distance set, can be calculated as follows \cite{gao2023distributionally}:
\begin{equation}\label{eq:wd dual}
    \sigma_{\cp^a_s}(V)=\sup_{\lambda\geq 0}\left(-\lambda\delta^l+\mE_{\tilde{\kp}^a_{s}}\big[\inf_{y}\big(V(y)+\lambda d(S,y)^l \big)\big] \right).
\end{equation}

\subsection{Robust Bellman Operator}

Motivated by Theorem \ref{thm:robust Bellman}, we define the robust Bellman operator, which forms the basis for our policy evaluation procedure.

\begin{definition}[Robust Bellman Operator, \cite{wang2023model}]
The robust Bellman operator $\mathbf{T}_g$ is defined as:
\begin{equation} \label{eq:bellmanoperator}
    \mathbf{T}_g(V)(s) = \sum_{a} \pi(a|s) \big[ r(s,a) - g +  \sigma_{\cp^a_s}(V) \big], \quad \forall s \in \mathcal{S}.
\end{equation}
\end{definition}

The operator $\mathbf{T}_g$ transforms a value function $V$ by incorporating the worst-case transition effect. A key challenge in solving the robust Bellman equation is that $\mathbf{T}_g$ does not satisfy contraction under standard norms, preventing the use of conventional fixed-point iteration. To cope with this problem, we establish that $\mathbf{T}_g$ is a contraction under some constructed semi-norm. This allows us to further develop provably efficient stochastic approximation algorithms.

\section{Semi-Norm Contraction of Robust Bellman Operators} \label{contractionsection}

Under Assumption \ref{ass:sameg}, we are able to establish the semi-norm contraction property. For motivation, we first establish the semi-norm contraction property of the non-robust average-reward Bellman operator for a policy $\pi$ under transition $\kp$ defined as follows:
\begin{equation} \label{eq:bellmanoperator_nonrobust}
    \mathbf{T}_g^{\kp}(V)(s) = \sum_{a} \pi(a|s) \big[ r(s,a) - g +  \sum_{s'} \kp(s'|s,a)V(s')\big], \quad \forall s \in \mathcal{S}.
\end{equation}

\begin{lemma} \label{lem:seminorm-contraction}
Let $\mathcal{S}$ be a finite state space, and let $\pi$ be a stationary policy. If the Markov chain induced by $\pi$ under the transition $\kp$ is irreducible and aperiodic, there exists a semi-norm $\|\cdot\|_{\kp}$ with kernel $\{c \e : c \in \mathbb{R}\}$ and a constant $\beta \in (0,1)$ such that for all $V_1, V_2 \in \mathbb{R}^{S}$ and any $g \in \mathbb{R}$,
\begin{equation}
\bigl\| \mathbf{T}_g^{\kp}(V_1) -  \mathbf{T}_g^{\kp}(V_2)\bigr\|_{\kp}
\leq \beta \|V_1 - V_2\|_{\kp}.
\end{equation}
\end{lemma}

\paragraph{Proof Sketch}
Under ergodicity, the one‐step transition matrix (denoted as $\kp^\pi$) has a unique stationary distribution $d^\pi$, define the stationary projector $E=\e^\top d^\pi$, then the fluctuation matrix (defined as $Q^\pi=\kp^\pi-E$) has all eigenvalues strictly inside the unit circle.  Standard finite‐dimensional theory (via the discrete Lyapunov equation \cite{horn2012matrix}) would produce a norm $\|\cdot\|_{Q}$ on $\mathbb{R}^S$ such that there is a constant $\alpha\in(0,1)$ such that for any $x \in \mathbb{R}^S$, $\|Q^\pi x\|_{Q}\leq\alpha\|x\|_{Q}$. We then build the semi-norm as follows:
\begin{equation}
\|x\|_{\kp}=\|Q^\pi x\|_{Q}+\epsilon \inf_{c\in\mathbb{R}}\|x-c\e\|_{Q},
\quad
0<\epsilon<1-\alpha,
\end{equation}
so that its kernel is exactly the constant vectors (the second term vanishes only on shifts of $\e$) and the first term enforces a one‐step shrinkage by $\beta=\alpha+\epsilon<1$.  A short calculation then shows
\(\|\kp^\pi x\|_{\kp}\leq \beta\|x\|_{\kp}\), yielding the desired contraction, which leads to the overall result.

The concrete proof of Lemma \ref{lem:seminorm-contraction} including the detailed construction of the semi-norm $\|\cdot\|_\kp$ is in Appendix \ref{proofspan-contraction}, where the  properties of irreducible and aperiodic finite state Markov chain are utilized. Thus, we show the (non-robust) average-reward Bellman operator $\mathbf{T}_g^{\kp}$ is a strict contraction under $\|\cdot\|_\kp$. Based on the above motivations, we now formally establish the contraction property of the robust average-reward Bellman operator by leveraging Lemma \ref{lem:seminorm-contraction} and the compactness of the uncertainty sets.

\begin{theorem} \label{thm:robust_seminorm-contraction}
Under Assumption \ref{ass:sameg}, if $\cp$ is compact, with certain restrictions on the radius of the uncertainty sets, there exists a semi-norm $\|\cdot\|_{\cp}$ with kernel $\{c \e : c \in \mathbb{R}\}$ such that the robust Bellman operator $\mathbf{T}_g$ is a contraction. Specifically, there exist $\gamma\in(0,1)$ such that
\begin{equation} \label{eq:contractiongamma}
\|\mathbf{T}_g(V_1) - \mathbf{T}_g(V_2)\|_{\cp}
\;\le\;\gamma\,\|V_1 - V_2\|_{\cp},
\;\forall \; V_1,V_2\in\mathbb R^{ S},\,g\in\mathbb R.
\end{equation}
\end{theorem}

\paragraph{Proof Sketch}
For any $\kp\in\cp$, the one‐step transition matrix $\kp^\pi$ has a unique stationary projector $E_\kp$ due to ergodicity. Since $\cp$ is compact, the family of fluctuation matrices $\{Q_\kp^\pi = \kp^\pi - E_\kp : \kp\in\cp\}$  has joint spectral radius strictly less than $1$.  By Lemma \ref{lem:bergerlemmaIV}in \cite{berger1992bounded}, one is able to construct an ``extremal norm'' (denoted as $\|\cdot\|_{\rm ext}$) under which every $Q_\kp^\pi$ contracts by a uniform factor $\alpha \in (0,1)$.  Mimicking the non-robust case in Lemma \ref{lem:seminorm-contraction}, we similarly define

\begin{equation} \label{eq:psemeinormbrief}
\|x\|_{\cp}
=\sup_{\kp \in\cp}\|Q_\kp^\pi x\|_{\rm ext}
+\epsilon \inf_{c\in\mathbb{R}}\|x-c\e\|_{\rm ext},
\quad
0<\epsilon<1-\alpha.
\end{equation}
The supremum term zeros out if $x \in \{c\e:c\in \mathbb{R}\}$ , and it inherits the uniform one‐step shrinkage by $\alpha$.  Adding the small quotient term fixes the kernel without spoiling $\gamma=\alpha+\epsilon<1$, so one shows at once
\begin{equation}
\|\mathbf{T}^\kp_g(V_1)-\mathbf{T}^\kp_g(V_2)\|_{\cp}\leq \gamma\|V_1-V_2\|_{\cp} \quad \text{for all} \; \kp\in\cp
\end{equation}
The above leads to the desired results.

The concrete proof of Theorem \ref{thm:robust_seminorm-contraction} along with the detailed construction of the semi-norm $\|\cdot\|_\cp$ and the specific radius restrictions on various uncertainty sets are in Appendix \ref{proofrobust-span-contraction}. 
 Since all the uncertainty sets listed in Section \ref{sec:ramdp} are closed and bounded in a real vector space, these uncertainty sets are all compact and satisfy the contraction property in Theorem \ref{thm:robust_seminorm-contraction}. We also note that the contraction factor $\gamma$ relates to the joint spectral gap of the family $\{Q^\pi_\kp : \kp \in\cp\}$.

\section{Efficient Estimators for Uncertainty Sets} \label{QueriesSection}

To utilize the contraction property in Section \ref{contractionsection} to obtain convergence rate results, our idea is perform the following iterative stochastic approximation:  
\begin{equation} \label{eq:generalSA}
V_{t+1}(s) \leftarrow V_t(s) + \eta_t \left( \hat{\mathbf{T}}_{g}(V_t)(s) - V_t(s) \right), \quad \forall s \in \mathcal{S}
\end{equation}
where the learning rate $\eta_t$ would be specified in Section \ref{robustTD}. The detailed analysis and complexities of the general stochastic approximation in the form of \eqref{eq:generalSA} is provided in 
 Appendix \ref{seminormcontractionwithbias}. Theorem \ref{thm:informalbiasedSA} implies that if $\hat{\mathbf{T}}_{g}(V)$, being an estimator of $\mathbf{T}_{g}(V)$, could be constructed with bounded variance and small bias, $V_t$ converges to a solution of the Bellman equation in \eqref{eq:bellman}. However, the challenge of constructing our desired $\hat{\mathbf{T}}_{g}(V)$ lies in the construction of the support function estimator $\hat{\sigma}_{\cp^a_s}(V)$.

In this section, we aim to construct an estimator $\hat{\sigma}_{\cp^a_s}(V)$ for all $s \in \mathcal{S}$ and $a \in \mathcal{A}$ in various uncertainty sets. Recall that the support function ${\sigma}_{\cp^a_s}(V)$ represents the worst-case transition effect over the uncertainty set $\cp^a_s$ as defined in the robust Bellman equation in Theorem \ref{thm:robust Bellman}. The explicit forms of ${\sigma}_{\cp^a_s}(V)$ for different uncertainty sets were characterized in \eqref{eq:contamination}-\eqref{eq:wd dual}. Our goal in this section is to construct efficient estimators $\hat{\sigma}_{\cp^a_s}(V)$ that approximates ${\sigma}_{\cp^a_s}(V)$ while maintaining controlled variance and finite sample complexity.

\paragraph{Linear Contamination Uncertainty Set}
Recall that the $\delta$-contamination uncertainty set is
$
    \cp^a_s=\{(1-\delta)\tilde{\kp}^a_s+\delta q: q\in\Delta(\mcs) \}, 
$
where $0<\delta<1$ is the radius. Since the support function can be computed by \eqref{eq:contamination} and the expression is linear in the nominal transition kernel $\tilde{\kp}^a_s$. A direct approach is to use the transition to the subsequent state to construct our estimator:
\begin{align}\label{eq:contaminationquery}
    \hat{\sigma}_{\cp^a_s}(V)\triangleq (1-\delta) V(s')+\delta\min_x V(x),
\end{align}
where $s'$ is a subsequent state sample after $(s,a)$. Hence, the sample complexity of \eqref{eq:contaminationquery} is just one. Lemma \ref{lem:wangthmD1} from \cite{wang2023model} states that $\hat{\sigma}_{\cp^a_s}(V)$ obtained by \eqref{eq:contaminationquery} is unbiased and has bounded variance as follows:
\begin{equation}
        \E\left[\hat{\sigma}_{\cp^a_s}(V)\right] = {\sigma}_{\cp^a_s}(V), \quad \text{and} \quad \mathrm{Var}(\hat{\sigma}_{\cp^a_s}(V)) \leq  \|V\|^2
\end{equation}

\paragraph{Nonlinear Contamination Sets}

Regarding TV and Wasserstein distance uncertainty sets, they have a nonlinear relationship between the nominal distribution $\tilde{\kp}^a_s$ and the support function ${\sigma}_{\cp^a_s}(V)$. Previous works such as \citep{blanchet2015unbiased,blanchet2019unbiased, wang2023model} have proposed a Multi-Level Monte-Carlo (MLMC) method for obtaining an unbiased estimator  of ${\sigma}_{\cp^a_s}(V)$ with bounded variance. However, their approaches require drawing $2^{N+1}$ samples where $N$ is sampled from a geometric distribution $\mathrm{Geom}(\Psi)$ with parameter $\Psi \in (0,0.5)$. This operation would need infinite samples in expectation for obtaining each single estimator as $\mathbb{E}[2^{N+1}] = \sum^{\infty}_{N=0} 2^{N+1} \Psi(1-\Psi)^N  = \sum^{\infty}_{N=0} 2\Psi(2-2\Psi)^N \rightarrow \infty$.
To handle the above problem, we aim to provide an estimator $\hat{\sigma}_{\cp^a_s}(V)$ with finite sample complexity and small enough bias. We construct a truncated-MLMC estimator under geometric sampling with parameter $\Psi=0.5$ as shown in Algorithm \ref{alg:sampling}.

\begin{algorithm}[htb]
\caption{Truncated MLMC Estimator for TV and Wasserstein Uncertainty Sets}
\label{alg:sampling}
\textbf{Input}: $s\in \mathcal{S}$, $a\in\mathcal{A}$,  Max level $N_{\max}$, Value function $V$
\begin{algorithmic}[1] 
\State Sample $N \sim \mathrm{Geom}(0.5)$
\State $N' \leftarrow \min \{N, N_{\max}\}$
\State Collect $2^{N'+1}$ i.i.d. samples of $\{s'_i\}^{2^{N'+1}}_{i=1}$ with $s'_i \sim \tilde{\kp}^a_s$ for each $i$
\State $\hat{\kp}^{a,E}_{s,N'+1} \leftarrow \frac{1}{2^{N'}}\sum_{i=1}^{2^{N'}} \mathbbm{1}_{\{s'_{2i}\}}$
\State $\hat{\kp}^{a,O}_{s,N'+1} \leftarrow \frac{1}{2^{N'}}\sum_{i=1}^{2^{N'}} \mathbbm{1}_{\{s'_{2i-1}\}}$
\State $\hat{\kp}^{a}_{s,N'+1}\leftarrow\frac{1}{2^{N'+1}}\sum_{i=1}^{2^{N'+1}} \mathbbm{1}_{\{s'_i\}}$
\State$\hat{\kp}^{a,1}_{s,N'+1} \leftarrow \mathbbm{1}_{\{s'_1\}}$
\If{TV} Obtain $\sigma_{\hat{\kp}^{a,1}_{s,N'+1}}(V), \sigma_{\hat{\kp}^{a}_{s,N'+1}}(V), \sigma_{\hat{\kp}^{a,E}_{s,N'+1}}(V), \sigma_{\hat{\kp}^{a,O}_{s,N'+1}}(V)$ from \eqref{eq:tv dual}
\ElsIf{Wasserstein} Obtain $\sigma_{\hat{\kp}^{a,1}_{s,N'+1}}(V), \sigma_{\hat{\kp}^{a}_{s,N'+1}}(V), \sigma_{\hat{\kp}^{a,E}_{s,N'+1}}(V), \sigma_{\hat{\kp}^{a,O}_{s,N'+1}}(V)$ from \eqref{eq:wd dual}
\EndIf
\State $\Delta_{N'}(V)\leftarrow \sigma_{\hat{\kp}^{a}_{s,N'+1}}(V)-\frac{1}{2}\Bigl[ \sigma_{\hat{\kp}^{a,E}_{s,N'+1}}(V)+  \sigma_{\hat{\kp}^{a,O}_{s,N'+1}}(V)
\Bigr]$
\State $\hat{\sigma}_{\cp^a_s}(V)\leftarrow\sigma_{\hat{\kp}^{a,1}_{s,N'+1}}(V)+\frac{\Delta_{N'}(V)}{  \mathbb{P}(N' = n) },
\text{where }
p'(n) = \mathbb{P}(N' = n)$
\Return $\hat{\sigma}_{\cp^a_s}(V)$
\end{algorithmic}
\end{algorithm}

In particular, if $n<N_{\max}$, then $\{N'=n\}=\{N=n\}$ with probability $(\tfrac12)^{n+1}$, while $\{N'=N_{\max}\}$ has probability $\sum_{m=N_{\max}}^\infty (1/2)^{m+1} = 2^{-N_{\max}}$. After obtaining $N'$, Algorithm \ref{alg:sampling} then collects a set of $2^{N'+1}$ i.i.d. samples from the nominal transition model to construct empirical estimators for different transition distributions. The core of the approach lies in computing the support function estimates for TV and Wasserstein uncertainty sets using a correction term $\Delta_{N'}(V)$, which accounts for the bias introduced by truncation. This correction ensures that the final estimator maintains a low bias while achieving a finite sample complexity. This truncation technique has been widely used in prior work across different settings such as \citep{wang2024model,ganesh2025sharper,xu2025accelerating}. We now present several crucial properties of Algorithm \ref{alg:sampling}.

\begin{theorem}[Finite Sample Complexity]
\label{thm:sample-complexity}
Under Algorithm \ref{alg:sampling}, denote $M = 2^{N'+1}$
as the random number of samples (where $N'=\min\{N,N_{\max}\}$).  Then
\begin{equation}
\mathbb{E}[M]=N_{\max}+2=\mathcal{O}(N_{\max}).
\end{equation}
\end{theorem}
The proof of Theorem \ref{thm:sample-complexity} is in Appendix \ref{proof:sample-complexity}, which demonstrates that setting the geometric sampling parameter to $\Psi=0.5$  ensures that the expected number of samples follows a linear growth pattern rather than an exponential one. This choice precisely cancels out the effect of the exponential sampling inherent in the truncated MLMC estimator, preventing infinite expected sample complexity. This result shows that the expected number of queries grows only linearly with $N_{\max}$, ensuring that the sampling cost remains manageable even for large truncation levels. The key factor enabling this behavior is setting the geometric distribution parameter to $0.5,$ which balances the probability mass across different truncation levels, preventing an exponential increase in sample complexity.

\begin{theorem}[Exponentially Decaying Bias]
\label{thm:exp-bias}
Let $\hat{\sigma}_{\cp^a_s}(V)$ be the estimator of ${\sigma}_{\cp^a_s}(V)$ obtained from Algorithm \ref{alg:sampling} then under TV uncertainty set, we have:
\begin{equation}
\abs{\mathbb{E}\bigl[\hat{\sigma}_{\cp^a_s}(V) - {\sigma}_{\cp^a_s}(V)\bigr] } \leq
6(1+\frac{1}{\delta}) 2^{-\frac{N_{\max}}{2}}\|V\|_{\mathrm{sp}}
\end{equation}
where $\delta$ denotes the radius of TV distance. Under Wasserstein uncertainty set, we have:
\begin{equation}
\abs{\mathbb{E}\bigl[\hat{\sigma}_{\cp^a_s}(V) - {\sigma}_{\cp^a_s}(V)\bigr] } \leq
6\cdot 2^{-\frac{N_{\max}}{2}}\|V\|_{\mathrm{sp}}
\end{equation}
\end{theorem}
Theorem \ref{thm:exp-bias} establishes that the bias of the truncated MLMC estimator decays exponentially with $N_{\max}$, ensuring that truncation does not significantly affect accuracy. This result follows from observing that the deviation introduced by truncation can be expressed as a sum of differences between support function estimates at different level, and each of which is controlled by the $\ell_1$-distance between transition distributions. Thus, we can use binomial concentration property to ensure the exponentially decaying bias.

The proof of Theorem \ref{thm:exp-bias} is in Appendix \ref{proof:exp-bias}. One important lemma used in the proof is the following Lemma \ref{lem:LipschitzTV}, where we show the Lipschitz property for both TV and Wasserstein distance uncertainty sets.
\begin{lemma}
\label{lem:LipschitzTV}
For any $p,q \in \Delta(\mathcal{S})$, let $\mathcal{P}_{TV}$ and $\mathcal{Q}_{TV}$ denote the TV distance uncertainty set with radius $\delta$ centering at $p$ and $q$ respectively, and let $\mathcal{P}_{W}$ and $\mathcal{Q}_{W}$ denote the Wasserstein distance uncertainty set with radius $\delta$ centering at $p$ and $q$ respectively. Then for any value function $V$, we have:
\begin{equation} \label{eq:TVWlipschitz} 
|\sigma_{\mathcal{P}_{TV}} (V) - \sigma_{\mathcal{Q}_{TV}} (V)| \leq (1+\frac{1}{\delta})\|V\|_{\mathrm{sp}}\|p-q\|_1 \; \text{and} \; |\sigma_{\mathcal{P}_{W}} (V) - \sigma_{\mathcal{Q}_{W}} (V)| \leq \|V\|_{\mathrm{sp}}\|p-q\|_1 
\end{equation}
\end{lemma}
We refer the proof of Theorem \ref{thm:exp-bias} to Appendix \ref{proof:LipschitzTV}.

\begin{theorem}[Linear Variance]
\label{thm:linear-variance}
Let $\hat{\sigma}_{\cp^a_s}(V)$ be the estimator of ${\sigma}_{\cp^a_s}(V)$ obtained from Algorithm \ref{alg:sampling} then under TV distance uncertainty set, we have:
\begin{equation}
 \mathrm{Var}(\hat{\sigma}_{\cp^a_s}(V)) \leq  3\|V\|_{\mathrm{sp}}^2 + 144(1+\frac{1}{\delta})^2\|V\|_{\mathrm{sp}}^2 N_{\max}
\end{equation}
and under Wasserstein distance uncertainty set, we have:
\begin{equation}
 \mathrm{Var}(\hat{\sigma}_{\cp^a_s}(V)) \leq  3\|V\|_{\mathrm{sp}}^2 + 144\|V\|_{\mathrm{sp}}^2 N_{\max}
\end{equation}
\end{theorem}

Theorem \ref{thm:linear-variance} establishes that the variance of the truncated MLMC estimator grows linearly with $N_{\max}$, ensuring that the estimator remains stable even as the truncation level increases.
The proof of Theorem \ref{thm:linear-variance} is in Appendix \ref{proof:linear-variance}, which follows from bounding the second moment of the estimator by analyzing the variance decomposition 
across different MLMC levels. Specifically, by expressing the estimator in terms of successive refinements of the transition model, we show that the variance accumulates additively across levels due to the binomial concentration property.
\section{Robust Average-Reward TD Learning} \label{robustTD}
Equipped with the methods of constructing $\hat{\sigma}_{\cp^a_s}(V)$ for all $s \in \mathcal{S}$ and $a \in \mathcal{A}$, we now present the formal algorithm for robust policy evaluation and robust average reward for a given policy $\pi$ in Algorithm \ref{alg:RobustTD}. Algorithm \ref{alg:RobustTD} presents a robust temporal difference (TD) learning method for policy evaluation in robust average-reward MDPs. This algorithm builds upon the truncated MLMC estimator (Algorithm~\ref{alg:sampling}) and the biased stochastic approximation framework in Section \ref{seminormcontractionwithbias}, ensuring both efficient 
sample complexity and finite-time convergence guarantees.

The algorithm is divided into two main phases. The first phase (Lines 1-7) estimates the robust value function. The noisy Bellman operator is computed using the estimator $\hat{\sigma}_{\mathcal{P}_s^a}(V_t)$ obtained depending on the uncertainty set type. Then the iterative update follows a stochastic approximation scheme with stepsize $\eta_t$, ensuring convergence while maintaining stability. Finally, the value function is centered at an anchor state $s_0$ to remove the ambiguity due to its additive invariance. The second phase (Lines 8-14) estimates the robust average reward by utilizing $V_T$ from the output of the first phase. The expected Bellman residual  $\delta_t(s)$ is computed across all states and averaging it to obtain $\bar{\delta}_t$. A separate stochastic approximation update with stepsize $\beta_t$ is then applied to refine $g_t$, ensuring convergence to the robust worst-case average reward. By combining these two phases, Algorithm~\ref{alg:RobustTD} provides an efficient and provably 
convergent method for robust policy evaluation under average-reward criteria, marking 
a significant advancement over prior methods that only provided asymptotic guarantees. 

\begin{algorithm}[htb]
\caption{Robust Average-Reward TD}
\label{alg:RobustTD}
\textbf{Input}: Policy $\pi$, Initial values $V_0$, $g_0=0$, Stepsizes $\eta_t$, $\beta_t$, Max level $N_{\max}$, Anchor state $s_0\in\mcs$
\begin{algorithmic}[1] 
\For {$t = 0,1,\ldots, T-1$}
\For {each $(s,a)\in\mcs\times\mca$} 
\If {Contamination} Sample $\hat{\sigma}_{\cp^a_s}(V_t)$ according to \eqref{eq:contaminationquery}
\ElsIf{TV or Wasserstein} Sample $\hat{\sigma}_{\cp^a_s}(V_t)$ according to Algorithm \ref{alg:sampling}
\EndIf
\EndFor
\State $\hat{\mathbf{T}}_{g_0}(V_t)(s) \leftarrow \sum_{a} \pi(a|s) \big[ r(s,a) - g_0 +  \hat{\sigma}_{\cp^a_s}(V_t) \big], \quad \forall s \in \mathcal{S}$
\State  $V_{t+1}(s) \leftarrow V_t(s) + \eta_t \left( \hat{\mathbf{T}}_{g_0}(V_t)(s) - V_t(s) \right), \quad \forall s \in \mathcal{S}$
\State  $V_{t+1}(s) = V_{t+1}(s) - V_{t+1}(s_0), \quad \forall s \in \mathcal{S}$
\EndFor
\For {$t = 0,1,\ldots, T-1$}
\For {each $(s,a)\in\mcs\times\mca$} 
\If {Contamination} Sample $\hat{\sigma}_{\cp^a_s}(V_t)$ according to \eqref{eq:contaminationquery}
\ElsIf{TV or Wasserstein} Sample $\hat{\sigma}_{\cp^a_s}(V_t)$ according to Algorithm \ref{alg:sampling}
\EndIf
\EndFor
\State $\hat{\delta}_t(s) \leftarrow \sum_{a}\pi(a|s) \big[ r(s,a) +  \hat{\sigma}_{\cp^a_s}(V_T) \big]- V_T(s)  , \quad \forall s \in \mathcal{S}$
\State $\bar{\delta}_t \leftarrow \frac{1}{S}\sum_s \hat{\delta}_t(s)$
\State $g_{t+1} \leftarrow g_t + \beta_t(\bar{\delta}_t-g_t)$
\EndFor \quad
\Return $V_T$, $g_T$
\end{algorithmic}
\end{algorithm}

To derive the sample complexity of robust policy evaluation, we utilize the semi-norm contraction property of the Bellman operator in Theorem \ref{thm:robust_seminorm-contraction}, and fit Algorithm \ref{alg:RobustTD} into the general biased stochastic approximation result in Theorem \ref{thm:informalbiasedSA} while incorporating the bias analysis characterized in Section \ref{QueriesSection}. Since each phase of Algorithm \ref{alg:RobustTD} contains a loop of length $T$ with all the states and actions updated together, the total samples needed for the entire algorithm in expectation is $2SAT \E[N_{\max}]$, where $\E[N_{\max}]$ is one for contamination uncertainty sets and is $\cO(N_{\max})$ from Theorem \ref{thm:sample-complexity} for TV and Wasserstein distance uncertainty sets.

\begin{theorem} \label{thm:Vresult}
   If $V_t$ is generated by Algorithm \ref{alg:RobustTD} and satisfying Assumption \ref{ass:sameg}, then if the stepsize $\eta_t \coloneqq \cO(\frac{1}{t})$, we require a sample complexity of $\cO\left(\frac{SAt^2_{\mathrm{mix}}}{\epsilon^2(1-\gamma)^2} \right)$ for contamination uncertainty set and a sample complexity of $\tilde{\cO}\left(\frac{SAt^2_{\mathrm{mix}}}{\epsilon^2(1-\gamma)^2} \right)$ for TV and Wasserstein distance uncertainty set to ensure an $\epsilon$ convergence of $V_T$. Moreover, these results are order-optimal in terms of $\epsilon$.
\end{theorem}
\begin{theorem} \label{thm:gresult}
    If $g_t$ is generated by Algorithm \ref{alg:RobustTD} and satisfying Assumption \ref{ass:sameg}, then if the stepsize $\beta_t \coloneqq \cO(\frac{1}{t})$, we require a sample complexity of $\tilde{\cO}\left(\frac{SAt^2_{\mathrm{mix}}}{\epsilon^2(1-\gamma)^2} \right)$ for all contamination, TV, and  Wasserstein distance uncertainty set to ensure an $\epsilon$ convergence of $g_T$.
\end{theorem}

The formal version of Theorems \ref{thm:Vresult} and \ref{thm:gresult} along with the proofs are in Appendix \ref{proof:VGresults}. Theorem \ref{thm:Vresult} provides the order-optimal sample complexity 
of $\tilde{\cO}(\epsilon^{-2})$ for Algorithm \ref{alg:RobustTD} to achieve an $\epsilon$-accurate estimate of $V_T$. Although Theorem \ref{thm:Vresult} claims order-optimal in terms of $\epsilon$, we do not claim tightness in $S$, $A$ and $\gamma$, and treat sharpening these dependencies as open. The proof of Theorem \ref{thm:gresult} extends the analysis of Theorem \ref{thm:Vresult} to robust average reward estimation. The key difficulty lies in controlling the propagation of error from value function estimates to reward estimation. By again leveraging the contraction property and appropriately tuning stepsizes, we establish an $\tilde{\cO}(\epsilon^{-2})$ complexity bound for robust average reward estimation.

\section{Conclusion}
This paper provides the first finite-sample analysis for policy evaluation in robust average-reward MDPs, 
bridging a gap where only asymptotic guarantees existed. By introducing a biased stochastic approximation framework and leveraging the properties of various uncertainty sets, we establish finite-time convergence under biased noise. Our algorithm achieves an order-optimal sample complexity of $\tilde{\mathcal{O}}(\epsilon^{-2})$ for policy evaluation, despite the added complexity of robustness.

A crucial step in our analysis is proving that the robust Bellman operator is contractive under our constructed semi-norm $\|\cdot\|_\cp$, ensuring the validity of stochastic approximation updates. We further develop a truncated 
Multi-Level Monte Carlo estimator that efficiently computes worst-case value functions under total variation and Wasserstein uncertainty, while keeping bias and variance controlled. One limitation of this work is that the results require ergodicity to hold in the setting, as stated in Assumption \ref{ass:sameg}. Additionally, scaling the algorithm and results in the paper via function approximations remains an important open problem.
\section*{Acknowledgments}
We would like to thank Zaiwei Chen of Purdue University for assistance with identifying relevant literature and for constructive feedback. We also thank Zijun Chen and Nian Si of The Hong Kong University of Science and Technology for helpful discussions regarding Appendix~\ref{seminormcontraction}. Finally, we thank the anonymous reviewers for insightful comments that substantially improved the paper.

\bibliographystyle{plain}
\bibliography{main}

\newpage

\appendix
\onecolumn
\appendix

\section{Semi-Norm Contraction Property of the Bellman Operator} \label{seminormcontraction}

\subsection{Proof of Lemma \ref{lem:seminorm-contraction}} \label{proofspan-contraction}

For any $V_1, V_2 \in \mathbb{R}^{S}$ and define $\Delta = V_1 - V_2$.  Denote $\kp^\pi$ as the transition matrix under policy $\pi$ and the unique stationary distribution $d^\pi$,  and denote $E$ as the matrix with all rows being identical to $d^\pi$. We further define $Q^\pi = \kp^\pi - E$. Thus, we would have,
\begin{equation} \label{eq:non-robust_first_step}
    \mathbf{T}_g^\kp(V_1)(s)- \mathbf{T}_g^\kp(V_2)(s)=\sum_{s' \in \mathcal{S}}
\kp^\pi(s'| s)
\bigl[V_1(s') - V_2(s')\bigr]=\kp^\pi \,\Delta (s).
\end{equation}
which implies
\begin{equation} \label{eq:PQE}
    \mathbf{T}_g^\kp(V_1) - \mathbf{T}_g^\kp(V_2)
= \kp^\pi \Delta  = Q^\pi \Delta + E \Delta
\end{equation}
We now discuss the detailed construction of the semi-norm $\|\cdot\|_\kp$.
Since $\kp^\pi$ is ergodic, according to the Perron–Frobenius theorem, $\kp^\pi$ has an eigenvalue $\lambda_1=1$ of algebraic multiplicity exactly one, with corresponding right eigenvector $\e$. Moreover, all other eigenvalues  $\lambda_2 \geq \ldots \geq \lambda_S$ of $\kp^\pi$ satisfies $|\lambda_i| <1$ for all $i \in \{2,\ldots,S\}$.
\begin{lemma} \label{lem:Qspectrum}
    All eigenvalues of $Q^\pi$ lies strictly inside the unit circle.
    \begin{proof}
        Since $E=\e d^{\pi \top}$, $E$ is a rank‑one projector onto the span of $\e$. Hence the spectrum of $E$ is $\{1,0,\ldots, 0\}$. In addition, we can show $\kp^\pi$ and $E$ commute by
        \begin{align}
            \kp^\pi E = \kp^\pi (\e (d^{\pi })^\top) = (\e (d^{\pi })^\top) = E \\
            E \kp^\pi = \e ((d^{\pi })^\top \kp^\pi) = \e (d^{\pi })^\top =E
        \end{align}
        Thus, by the Schur's theorem, $\kp$ and $E$ are simultaneously upper triangularizable. In a common triangular basis, the diagonals of $\kp$ and $E$ list their eigenvalues in descending orders, which are $\{\lambda_1, \lambda_2,\ldots,\lambda_S\}$ and $\{1,0,\ldots, 0\}$ respectively. Thus, in that same basis, $Q^\pi = \kp^\pi - E$ is also triangular, with diagonal entries being $\{\lambda_1-1,\lambda_2-0,\ldots, \lambda_S-0\}$. Since $\lambda_1=1$, we have the spectrum of $Q^\pi$ is exactly $\{ \lambda_2,\ldots, \lambda_S,0\}$. Since we already have  $|\lambda_i| <1$ for all $i \in \{2,\ldots,S\}$, we conclude the proof.
    \end{proof}
\end{lemma}

Define $\rho(\cdot)$ to be the spectral radius of a matrix, then Lemma \ref{lem:Qspectrum} implies that $\rho(Q^\pi)<1$.  Hence by equivalence of norms in $\mathbb R^{|\mathcal S|}$ it is possible to construct a vector norm $\|\cdot\|_Q$ so that the induced operator norm of $Q^\pi$ is less than $1$, specifically
\begin{equation}\label{eq:definingalpha}
0 \leq \rho(Q^\pi) \leq \bigl\|Q^\pi\bigr\|_Q \leq \alpha < 1.
\end{equation}
A concrete construction example is to leverage the discrete-Lyapunov equation \citep{horn2012matrix} of solving $M$ on the space of symmetric matrices for any $\rho(Q^\pi) < \alpha < 1$ as follows:
\begin{equation} \label{eq:LyapunovEq}
     Q^{\pi\top} MQ^\pi - \alpha^2 M = -I
\end{equation}
Define $B\coloneqq \alpha^{-1}Q^\pi$, then $\rho(B)= \alpha^{-1}\rho(Q^\pi)<1$. We can express $M$ in the form of Neumann series as
\begin{align} \label{eq:Neumann}
     M &= \alpha^{-2} I + \alpha^{-4} (Q^\pi)^\top Q^\pi + \alpha^{-6} \left((Q^\pi)^\top\right)^2 (Q^\pi)^2 + \dots \nonumber \\
     &= \sum_{k=0}^{\infty} \alpha^{-2(k+1)} \left((Q^\pi)^\top\right)^k (Q^\pi)^k \nonumber \\
     &= \alpha^{-2} \sum_{k=0}^{\infty}  \left(B^\top\right)^k B^k.
\end{align}

We now show that $M$ is bounded. Write $B=SJS^{-1}$, where $J=\mathrm{diag}\big(J_{m_1}(\lambda_1),\dots,J_{m_r}(\lambda_r)\big)$ is the Jordan normal form. By the Jordan block power formula \citep{horn2012matrix},
\begin{equation}
J_m(\lambda) = \lambda I_m + N_m,
\end{equation}
with $N_m$ the nilpotent matrix having ones on the first superdiagonal and $N_m^m=0$.
Then $B^k=SJ^kS^{-1}$ and $J^k=\mathrm{diag}\big(J_{m_1}(\lambda_1)^k,\dots,J_{m_r}(\lambda_r)^k\big)$. For each block and each integer $k\ge m$, by the binomial theorem we have $N_m^m=0$ and
\begin{equation}
J_m(\lambda)^k
= (\lambda I_m + N_m)^k
= \sum_{j=0}^{m-1}\binom{k}{j}\,\lambda^{\,k-j}\,N_m^{\,j}.
\end{equation}
 
For $k\ge m$, use $\binom{k}{j}\le \frac{k^j}{j!}$ and factor $|\lambda|^{\,k}$:
\begin{equation}
\|J_m(\lambda)^k\|_2
\leq |\lambda|^{\,k}\sum_{j=0}^{m-1}\frac{k^j}{j!}|\lambda|^{-j} \|N_m^j\|_2
\leq c_{m,\lambda} k^{m-1}|\lambda|^{k},
\end{equation}
where $c_{m,\lambda}\coloneqq\sum_{j=0}^{m-1}\frac{|\lambda|^{-j}}{j!} \|N_m^j\|_2$. Thus, let $s = \max_i m_i$ be the size of the largest Jordan block of $B$. Since $J^k$ is block diagonal,
\begin{equation}
\|J^k\|_2
\leq \sum_{i=1}^r \|J_{m_i}(\lambda_i)^k\|_2
\leq \Big(\sum_{i=1}^r c_{m_i,\lambda_i}\Big) k^{s-1} \Big(\max_i |\lambda_i|\Big)^{k}
=C_Jk^{s-1}\rho(B)^{k}
\end{equation}
for all $k\ge s$, where $C_J:=\sum_{i=1}^r c_{m_i,\lambda_i}$ and $\rho(B)=\max_i|\lambda_i|$.

Since similarity does not change eigenvalues but may scale norms by the condition number, we can derive that
\[
\|B^k\|_2 = \|S J^k S^{-1}\|_2
\leq \|S\|_2\|S^{-1}\|_2\|J^k\|_2
\leq \kappa(S)\,C_J k^{s-1}\rho(B)^{k} \qquad (k\ge s),
\]
where $\kappa(S):=\|S\|_2\|S^{-1}\|_2$. By choosing the appropriate constant, the same bound holds for all $k\ge 0$:
\[
\exists \; C_B>0 \ \text{such that}\quad \|B^k\|_2 \;\le\; C_B\,k^{s-1}\,\rho(B)^{\,k}\quad (k\ge 0).
\]
In spectral norm, this implies
\begin{equation}
\|(B^\top)^k B^k\|_2
= \|(B^k)^\top B^k\|_2
= \sigma_{\max}(B^k)^2
= \|B^k\|_2^2
\leq C_B^{2}k^{2(s-1)}\rho(B)^{2k}.
\end{equation}
Thus, the scalar series
$
\sum_{k=0}^{\infty} k^{2(s-1)}\,\rho(B)^{\,2k}
$
is in the form of polynomial times geometric with ratio less than $1$, which converges, and the partial sum expression in \eqref{eq:Neumann} converges absolutely as a geometric‐type series.

Also, since each term in \eqref{eq:Neumann} is positive semi‑definite, and the first term $\alpha^{-2}I$ being positive definite, we can conclude that $M$ being the summation is well-defined and is positive definite. Thus, using the positive definite $M$ defined in \eqref{eq:LyapunovEq}, we can define our desired norm $\|\cdot\|_Q$ as 
\begin{equation} \label{eq:definitionofQnorm}
    \|x\|_Q \coloneqq \sqrt{x^\top M x}
\end{equation}
which implies
\begin{equation} \label{eq:Qnormproperty}
    \|Q^\pi\|_Q = \sup_{x \neq \bf{0}} \frac{\|Q^\pi x\|_Q}{\|x\|_Q} = \sup_{x \neq \bf{0}}\frac{\sqrt{(Q^\pi x)^\top M Q^\pi x}}{\sqrt{x^\top M x}} \overset{(a)}{\leq} \alpha < 1
\end{equation}
Where $(a)$ is because for any $x \neq \bf{0}$, from \eqref{eq:LyapunovEq} we have
\begin{equation} \label{eq:QuadraticLyapunov}
    (Qx)^{\pi\top} MQ^\pi x - \alpha^2 x^\top M x = -x^\top x \quad \Rightarrow \quad (Qx)^{\pi\top} MQ^\pi x = \alpha^2 x^\top M x - \|x\|_2^2
\end{equation}
Since $\|x\|_2^2$ is always non-negative dividing both sides of the second equation of \eqref{eq:QuadraticLyapunov} by $x^\top M x$ and further taking the square root on both sides yields the inequality of $(a)$.

Based on the above construction of the norm $\|\cdot\|_Q$, define the operator $\|\cdot\|_\kp$ as
\begin{equation} \label{eq:constructionofPnorm}
    \|x\|_{\kp} \coloneqq \bigl\|Q^\pi x\bigr\|_Q+\epsilon \inf_{c\in\mathbb R}\bigl\|x - c \e\bigr\|_Q
\end{equation}
where $0<\epsilon<1-\alpha$. 
\begin{lemma} \label{lem:validseminorm}
    The operator $\|\cdot\|_\kp$ is a valid semi-norm with kernel being exactly $\{c\e : c\in\mathbb{R}\}$. Furthermore, for all $x\in \mathbb{R}^S$, we have $\bigl\|\kp^\pi x\bigr\|_{\kp} \leq (\alpha+\epsilon)\bigl\| x\bigr\|_{\kp}$.
    \begin{proof}
    Regarding positive homogeneity and nonnegativity,  for any scalar $\lambda$ and $x\in\mathbb{R}^S$,
    \[
    \|\lambda x\|_{\kp}
    = \bigl\|Q^\pi(\lambda x)\bigr\|_Q
      + \epsilon\inf_{c}\bigl\|\lambda x - c\e\bigr\|_Q
    = |\lambda|\bigl\|Q^\pi x\bigr\|_Q
      + \epsilon\,|\lambda|\inf_{c}\bigl\|x - c\e\bigr\|_Q
    =|\lambda|\,\|x\|_{\kp},
    \]
    and clearly $\|x\|_{\kp}\ge0$, with equality only when both $\|Q^\pi x\|_Q=0$ and $\inf_{c}\|x-c\,e\|_Q=0$. Regarding triangle inequality, for any $x,y\in\mathbb{R}^S$,
    \begin{align*}
    \|x+y\|_{\kp}
    &= \bigl\|Q^\pi(x+y)\bigr\|_Q
       + \epsilon\inf_{c}\bigl\|x+y - c\e\bigr\|_Q\\
    &\le \bigl\|Q^\pi x\bigr\|_Q + \bigl\|Q^\pi y\bigr\|_Q
       + \epsilon\inf_{a,b}\bigl\|x - a\e + y - b\e\bigr\|_Q\\
    &\le \bigl\|Q^\pi x\bigr\|_Q + \bigl\|Q^\pi y\bigr\|_Q
       + \epsilon\inf_{a}\bigl\|x - a\e\bigr\|_Q
       + \epsilon\inf_{b}\bigl\|y - b\e\bigr\|_Q\\
    &= \|x\|_{\kp} + \|y\|_{\kp}.
    \end{align*} Regarding the kernel, if $x=k \e$ for some $k\in \mathbb{R}$, then we have
    \begin{align} \label{eq:kernelseminorm}
        \|x\|_\kp &= \|k Q^\pi \e\|_Q +\epsilon\inf_{c}\|k\e-c\e\|_Q \nonumber \\
        &=  \|k (\kp^\pi-E) \e\|_Q +\epsilon\|k \e - k\e\|_Q \nonumber \\
        &= \|k \e - k\e\|_Q +\epsilon\|0\|_Q=0
    \end{align}
    On the other hand, if $x\notin \{c\e : c\in\mathbb{R}\}$, we know that 
    \begin{equation}
        \|x\|_\kp \geq \epsilon\inf_{c}\|x-c\e\|_Q >0
    \end{equation}    
    Thus, the kernel of $\|\cdot\|_\kp$ is exactly $\{c\e : c\in\mathbb{R}\}$. We now show that, for any $x\in\mathbb{R}^S$,
    \begin{align}
    \bigl\|\kp^\pi x\bigr\|_{\kp}
    &= \bigl\|Q^\pi(\kp^\pi x)\bigr\|_Q
      + \epsilon\inf_{c}\bigl\|\kp^\pi x - c\e\bigr\|_Q \nonumber \\
      &= \bigl\|Q^\pi Q^\pi x + Q^\pi E x \bigr\|_Q
      + \epsilon\inf_{c}\bigl\|Q^\pi x - c\e\bigr\|_Q \nonumber \\
    &\leq \alpha \|Q^\pi x\|_Q
       +\epsilon \|Q^\pi x\|_Q \nonumber \\
    & = (\alpha+\epsilon)\|Q^\pi x\|_Q \nonumber \\
    & \leq (\alpha+\epsilon)\|x\|_{\kp}.
    \end{align}
    \end{proof}
\end{lemma}
Let $\beta = \alpha + \epsilon$, by \eqref{eq:definingalpha} and \eqref{eq:constructionofPnorm}, we have $\alpha\in(0,1)$ and $\epsilon \in (0,1-\alpha)$. Thus, $\beta \in (0,1)$ and combining $\beta$ with the semi-norm $\|\cdot\|_\kp$ confirms Lemma \ref{lem:seminorm-contraction}.

\subsection{Proof of Theorem \ref{thm:robust_seminorm-contraction}} \label{proofrobust-span-contraction}

We override the terms $\alpha, \lambda$ and $\epsilon$ from the previous section.  For any $V_1,V_2$ and $s\in\mathcal S$,
 \begin{align} \label{eq:robustbellmanstep1}
         \mathbf{T}_g(V_1)(s) -&  \mathbf{T}_g(V_2)(s) = \sum_{a \in \mathcal{A}} \pi(a|s) [\sigma_{p^a_s}(V_1) -  \sigma_{p^a_s}(V_2)] \nonumber \\
        & = \sum_{a \in \mathcal{A}} \pi(a|s) [ \min_{p\in\cp^a_s} \sum_{s' \in \mathcal{S}} p(s')V_1(s') -   \min_{p\in\cp^a_s} \sum_{s' \in \mathcal{S}} p(s') V_2(s')] \nonumber \\
        & \leq  \sum_{a \in \mathcal{A}} \pi(a|s) \max_{p\in\cp^a_s} \Bigg[ \sum_{s' \in \mathcal{S}} p(s')V_1(s') -  \sum_{s' \in \mathcal{S}} p(s') V_2(s') \Bigg] \nonumber \\
        & \leq \sum_{a}\pi(a|s)\sum_{s'}\Tilde{p}_{(V_1,V_2)}(s'|s,a)\bigl[V_1(s')-V_2(s')\bigr]
    \end{align}
where $\Tilde{p}_{(V_1,V_2)}(\cdot | s,a) = \arg\max_{p\in\cp^a_s} [ \sum_{s' \in \mathcal{S}} p(s')V_1(s') -  \sum_{s' \in \mathcal{S}} p(s') V_2(s') ] $
and each $\Tilde{p}_{(V_1,V_2)}\in\cp$ for all $V_1, V_2$. We now discuss the construction of the desired semi-norm $\|\cdot\|_\cp$. 

\subsubsection{Joint Spectral Radius of $Q_\cp^\pi$}
For any $\kp \in \cp$, denote $\kp^\pi$ as the transition matrix under policy $\pi$ and the unique stationary distribution $d^\pi_\kp$,  and denote $E_\kp$ as the matrix with all rows being identical to $d^\pi_\kp$ (we will provide the conditions for all $\kp^\pi$ having a unique stationary distribution later). We further define the following:
\begin{equation} \label{eq:Qfamilydef}
    Q_\kp^\pi = \kp^\pi - E_\kp \quad \text{and} \quad Q^\pi_\cp = \{Q^\pi_\kp : \kp \in \cp\}. 
\end{equation}
To obtain the desired one-step contraction result under Assumption \ref{ass:sameg} along with proper radius restrictions, we need to show the conditions of the radius under the different uncertainty sets such that the joint spectral radius $\hat{\rho}(Q^\pi_\cp )$ defined in Lemma \ref{lem:bergerlemmaIV} satisfies $\hat{\rho}(Q^\pi_\cp )<1$, which is necessary to establish the desired one-step contraction. We first provide an upper bound of the joint spectral radius as follows:

\begin{lemma} \label{lem:JSRupperbound}
Define the Dobrushin’s coefficient of an $n$ dimensional Markov matrix $P$ as 
\begin{equation} \label{eq:Dobrushindef}
\tau(P)\coloneqq 1 - \min_{i<j}\sum_{s=1}^n \min(P_{is},P_{js}),
\end{equation}
then the joint spectral radius of the family $Q_\cp^\pi$ is upper bounded by the following:
\begin{equation} \label{eq:JSRupperbound}
    \hat{\rho}(Q^\pi_\cp ) \leq \inf_{m \geq 1}\left(\sup_{\kp_i \in \cp}\tau(\kp_1^\pi \cdot \ldots \cdot \kp_m^\pi)\right)^{\frac{1}{m}}
\end{equation}

\begin{proof}
    We start by first connecting $\hat{\rho}(Q^\pi_\cp )$ to the joint spectral radius of the family $\{\kp^\pi : \kp\in\cp\}$. Define  $\mathbb{H}\coloneqq \{x\in \mathbb{R}^S: \e^\top x = 0\}$ to be the zero-sum subspace where the space spanned by $\e$ is removed. Furthermore, choose an orthonormal basis $U=[u_0 \; U_\mathbb{H}] \in \mathbb{R}^{S \times S}$ with 
    $$
    u_0 = \frac{1}{\sqrt{S}} \e, \quad U_{\mathbb{H}}^\top U_{\mathbb{H}} = I_{S-1}, \quad U_{\mathbb{H}} U_{\mathbb{H}}^\top  = \Pi \coloneqq I_S - \frac{1}{S}\e \e^\top, \quad \e^\top U_{\mathbb{H}} = 0
    $$
where $\Pi$ is the orthogonal projector onto $\mathbb{H}$. Since $U$ is orthogonal, $U^\top = U^{-1}$. With the above notations, for any  $Q^\pi_\kp \in Q^\pi_\cp$, we can construct a similar matrix $\tilde{Q}^\pi_\kp$ as
\begin{equation}
\tilde{Q}^\pi_\kp \coloneqq U^{-1} Q^\pi_\kp U
=
\begin{bmatrix}
0 & \alpha_\kp^{\top}\\[2pt]
0 & B_\kp
\end{bmatrix}, \;\;
\text{where} \;\;
B_\kp:=U_{\mathbb{H}}^\top \kp^\pi U_{\mathbb{H}} \in \mathbb{R}^{(S-1)\times(S-1)}.
\end{equation}
Equivalently, define $T_\kp \coloneq \Pi \kp^\pi \big|_{\mathbb{H}} $, which operates entirely on $\mathbb{H}$. Then $B_\kp$ is the matrix of $T_\kp$ in the basis of $U_{\mathbb{H}}$. Since $E_\kp = \e (d_\kp^{\pi })^\top$ and $U_{\mathbb{H}}^\top \e = 0$, we have $U_{\mathbb{H}}^\top E_\kp U_{\mathbb{H}} = 0$. Hence, the lower-right block in $U_{\mathbb{H}}^\top Q^\pi_\kp U_{\mathbb{H}}$ is just $B_\kp$. Consequently, for any sequence \(\kp^\pi_1,\ldots,\kp^\pi_k\),
\begin{equation}
U^{-1}\bigl(Q^\pi_{\kp_k}\cdots Q^\pi_{\kp_1}\bigr)U
=
\begin{bmatrix}
0 & \ast\\[2pt]
0 & B_{\kp_k}\cdots B_{\kp_1}
\end{bmatrix}.
\end{equation}
Hence, by block upper-triangularity \citep{horn2012matrix}, the spectral radius of $Q^\pi_{\kp_k}\cdots Q^\pi_{\kp_1}$ on $\mathbb{R}^S$ equals the spectral radius of $ B_{\kp_k}\cdots B_{\kp_1}$ on $\mathbb{H}$:
\begin{align}\label{eq:rhoEquality}
\rho\bigl(Q^\pi_{\kp_k}\cdots Q^\pi_{\kp_1}\bigr)&=\rho\bigl( B_{\kp_k}\cdots B_{\kp_1}\bigr)\nonumber \\
&=\rho\bigl( U_{\mathbb{H}}^\top \kp_k^\pi U_{\mathbb{H}} \cdots U_{\mathbb{H}}^\top \kp_1^\pi U_{\mathbb{H}}\bigr)\nonumber \\
&=\rho\bigl( U_{\mathbb{H}}^\top \kp_k^\pi \Pi \kp_{k-1}^\pi \Pi \cdots \kp_2^\pi \Pi\kp_1^\pi U_{\mathbb{H}}\bigr)\nonumber \\
&=\rho(T_{\kp_k} \ldots T_{\kp_1}).
\end{align}
Thus, the spectral radius of the family $Q^\pi_{\kp_k}\cdots Q^\pi_{\kp_1}$ on $\mathbb{R}^S$ equals the spectral radius of $ \kp_k^\pi \cdots \kp_1^\pi $ on $\mathbb{H}$:
\begin{equation}\label{eq:rhoonH}
\rho\bigl(Q^\pi_{\kp_k}\cdots Q^\pi_{\kp_1}\bigr) =\rho(T_{\kp_k} \ldots T_{\kp_1})=\rho ( \Pi \kp_k^\pi \cdots \kp_1^\pi  \big|_{\mathbb{H}} ).
\end{equation}
Given \eqref{eq:rhoonH}, we now study the joint spectral radius of $\cp_{\mathbb{H}}=\{\Pi\kp^\pi : \kp\in\cp\}$ on $\mathbb{H}$.

From Lemma \ref{lem:bergerlemmaIV}, we have that for an arbitrary norm $\|\cdot\|_{\mathbb{H}\rightarrow \mathbb{H}}$ on $\mathbb{H}$,
         \begin{equation}
        \hat{\rho}(\cp_{\mathbb{H}}) = \lim_{k\rightarrow \infty} \sup_{\kp_i \in \cp}\|\Pi \kp^\pi_k \ldots \kp^\pi_1\|_{\mathbb{H}\rightarrow \mathbb{H}}^{\frac{1}{k}} = \lim_{k\rightarrow \infty} \sup_{\kp_i \in \cp}\|T_{\kp_k} \ldots T_{\kp_1}\|_{\mathbb{H}\rightarrow \mathbb{H}}^{\frac{1}{k}}.
    \end{equation}
    Divide $k$ into $m$ partitions of blocks with length $q$ and a residue $r$ as $k=qm+r$, where $q,m,r \in \mathbb{N}$, $0 < q \leq k$ and $0 \leq r < q $. Furthermore, let $M_m \coloneqq \sup_{\kp_i \in \cp} \|T_{\kp_m} \ldots T_{\kp_1}\|_{\mathbb{H}\rightarrow \mathbb{H}}$ and $K_{<m} \coloneqq \max_{0\leq r <m}\sup_{\kp_i \in \cp} \|T_{\kp_r} \ldots T_{\kp_1}\|_{\mathbb{H}\rightarrow \mathbb{H}}$, then by the submultiplicity of operator norm, we have that on $\mathbb{H}$,
    \begin{equation}
         \sup_{\kp_i \in \cp}\|T_{\kp_k} \ldots T_{\kp_1}\|_{\mathbb{H}\rightarrow \mathbb{H}} \leq M_m^q K_{<m},
    \end{equation}
    taking power of $\frac{1}{k}$ and let $k\rightarrow \infty$ implies
    \begin{equation}
        \lim_{k\rightarrow \infty} \sup_{\kp_i \in \cp}\|T_{\kp_k} \ldots T_{\kp_1}\|_{\mathbb{H}\rightarrow \mathbb{H}}^{\frac{1}{k}} \leq \lim_{k\rightarrow \infty}  M_m^{\frac{q}{k}} K_{<m}^{\frac{1}{k}}.
    \end{equation}
    Since $q = \lfloor \frac{k}{m} \rfloor$, we have $\lim_{k\rightarrow \infty} \frac{q}{k} = \frac{1}{m}$ and $\lim_{k\rightarrow \infty} \frac{1}{k} = 0$, which suggests that for any positive integer $m$ we have, 
    \begin{equation}
        \lim_{k\rightarrow \infty} \sup_{\kp_i \in \cp}\|T_{\kp_k} \ldots T_{\kp_1}\|_{\mathbb{H}\rightarrow \mathbb{H}}^{\frac{1}{k}} \leq  M_m^{\frac{1}{m}} =  \left( \sup_{\kp_i \in \cp} \|T_{\kp_m} \ldots T_{\kp_1}\|_{\mathbb{H}\rightarrow \mathbb{H}}\right)^{\frac{1}{m}},
    \end{equation}
    which implies for any norm  $\|\cdot\|_{\mathbb{H}\rightarrow \mathbb{H}}$ on $\mathbb{H}$, we have
    \begin{equation}
        \lim_{k\rightarrow \infty} \sup_{\kp_i \in \cp}\|T_{\kp_k} \ldots T_{\kp_1}\|_{\mathbb{H}\rightarrow \mathbb{H}}^{\frac{1}{k}} \leq \inf_{m\geq 1} \left( \sup_{\kp_i \in \cp} \|T_{\kp_m} \ldots T_{\kp_1}\|_{\mathbb{H}\rightarrow \mathbb{H}}\right)^{\frac{1}{m}}.
    \end{equation}
    From \citep{gaubert2015dobrushin}, the Dobrushin’s coefficient is a valid norm (the induced matrix span norm) on the zero-sum subspace $\mathbb{H}$, which yields \eqref{eq:JSRupperbound}.    
\end{proof}
\end{lemma}

Lemma \ref{lem:JSRupperbound} provides a quantitative method to relate the joint spectral radius of the family $Q^\pi_\cp$ and the  Dobrushin’s coefficient of the family $\cp$. Under Assumption \ref{ass:sameg}, we next discuss the radius restrictions of contamination, TV and Wasserstein distance uncertainty sets such that $\hat{\rho}(Q^\pi_\cp )<1$ is satisfied.

\subsubsection{Discussions on Radius Restrictions} \label{radiusrestrictions}

We provide the following Lemma \ref{lem:contaminationradius}-\ref{lem：Wradius}, which quantifies the radius restrictions regarding all three uncertainty sets of interests for obtaining the desired results.

\paragraph{Contamination Uncertainty}
Regarding contamination uncertainty, where the uncertainty set is characterized as 
\[
\cp :=\ \Bigl\{\kp: \forall(s,a), \kp(\cdot | s,a)\ =\ (1-\delta) \tilde{\kp}(\cdot | s,a) + \delta\,q(\cdot | s,a),  q(\cdot | s,a)\in\Delta(\mathcal S)\,\Bigr\},
\quad 0\leq \delta <1,
\]
For a fixed policy $\pi$, the induced state–transition matrix $\kp^\pi$ is expressed as 
\begin{equation} \label{eq:contaminationppi}
\kp^\pi(s,s') \coloneqq \sum_{a}\pi(a| s) \kp(s'| s,a) = (1-\delta) \sum_{a}\pi(a| s) \tilde{\kp}(s'| s,a) + \delta \sum_{a}\pi(a| s) q(s'| s,a)
\end{equation}
Define the induced uncertainty set $\mathcal P^\pi \coloneqq \{\kp^\pi:  \kp \in\cp\}$ and define $
\tilde{\kp}^\pi (s,s') \coloneqq \sum_{a}\pi(a| s) \tilde{\kp}(s'| s,a).$
Then \eqref{eq:contaminationppi} can be expressed as 
\begin{equation}
\mathcal P^\pi\ =\ \Bigl\{ (1-\delta) \tilde{\kp}^\pi\ +\ \delta q^\pi\ :\ q^\pi\ \text{row–stochastic}\ \Bigr\}.
\end{equation}

\begin{lemma} \label{lem:contaminationradius}
    Under the contamination uncertainty set, if the centroid $\tilde{\kp}^\pi$ is irreducible and aperiodic, then the joint spectral radius of $Q_\cp^\pi$ defined in \eqref{eq:Qfamilydef} is strictly less than $1$. Furthermore, $\kp^\pi$ is irreducible and aperiodic for all $\kp \in \cp$.
    \begin{proof}
        Since $\tilde{\kp}^{\pi}$ is irreducible and aperiodic. Then there exists an $m\in\mathbb N$ such that all entries in $(\tilde{\kp}^{\pi})^m$ are strictly positive. For any $\kp^\pi\in \cp^\pi$ we can write $\kp^\pi=(1-\delta)\tilde{\kp}^{\pi}+\delta q^\pi$ with $q^\pi$ being row–stochastic, so by multinomial expansion,
\begin{equation} \label{eq:contaminationprimitivity}
(\kp^\pi)^m = \bigl((1-\delta)\tilde{\kp}^{\pi}+\delta q^\pi\bigr)^m \geq (1-\delta)^m (\tilde{\kp}^{\pi})^m
\quad\text{(entrywise)}.
\end{equation}
Hence $(\kp^\pi)^m$ is strictly positive for all $\kp\in\cp$, which implies every $\kp^\pi \in\mathcal P^\pi$ is primitive with the same exponent $m$, which further implies $\kp^\pi$ is irreducible and aperiodic for all $\kp \in \cp$.

To bound the joint spectral radius, for the same integer $m$, define the $m$–step overlap constant of the centroid as
\[
a_0 \coloneqq \min_{i<j}\ \sum_{s\in\mathcal S}\min \bigl\{(\tilde{\kp}^{\pi})_{is},  (\tilde{\kp}^{\pi})_{js}\bigr\}\ \in(0,1].
\]
For any length–$m$ product $\kp^{\pi}_m\cdots \kp^{\pi}_1$ with $\kp^{\pi}_t\in\mathcal P^\pi$, the same entrywise bound \eqref{eq:contaminationprimitivity} gives
$\kp^{\pi}_m\cdots \kp^{\pi}_1 \geq (1-\delta)^m (\tilde{\kp}^{\pi})^m$, whence for all $i\neq j$, we have
\begin{equation}
\sum_s \min\{(\kp^{\pi}_m\cdots \kp^{\pi}_1)_{is},(\kp^{\pi}_m\cdots \kp^{\pi}_1)_{js}\} \geq (1-\delta)^m a_0.
\end{equation}
By the definition of the Dobrushin's coefficient in \eqref{eq:Dobrushindef}, the above yields
\begin{equation}
\tau(\kp^{\pi}_m\cdots \kp^{\pi}_1)\ \le\ 1-(1-\delta)^m a_0\ <\ 1.
\end{equation}
By Lemma \ref{lem:JSRupperbound},
\[
\hat{\rho}(Q^\pi_\cp ) \leq \inf_{t \geq 1}\left(\sup_{\kp_i \in \cp}\tau(\kp_1^\pi \cdot \ldots \cdot \kp_t^\pi)\right)^{\frac{1}{t}}\le\,\bigl(1-(1-\delta)^m a_0\bigr)^{1/m}\,<\,1.
\]
\end{proof}
\end{lemma}

Therefore, if the center $\tilde{\kp}^\pi$ is primitive and $0\le\delta<1$, without having any additional restrictions on the radius, we have that all induced kernels in $\mathcal P_\pi$ are irreducible and aperiodic. Furthermore,  the joint spectral radius of $Q^\pi_{\mathcal P}$ satisfies $\hat{\rho}(Q^\pi_{\mathcal P})<1$.

\paragraph{Total Variation (TV) Distance Uncertainty}
Regarding TV uncertainty, where the uncertainty set is characterized as 

\[
\cp \coloneqq \Bigl\{\kp :\ \forall(s,a),\ 
\mathrm{TV}\bigl(P(\cdot| s,a), \tilde{\kp}(\cdot| s,a)\bigr)\ \le\ \delta \Bigr\},
\qquad \delta \ge0,
\]
where $\mathrm{TV}(p,q):=\tfrac12\|p-q\|_1$. For a fixed policy $\pi$, the induced state–transition matrix $\kp^\pi$ is expressed as 
\begin{equation} \label{eq:tvppi}
\kp^\pi(s,s') \coloneqq \sum_{a}\pi(a| s) \kp(s'| s,a) 
\end{equation}
Then for each state $s$,
\begin{align}
\mathrm{TV}\bigl(\kp^\pi(s,\cdot), \tilde{\kp}^\pi(s,\cdot)\bigr)
&=\mathrm{TV} \Bigl(\sum_a \pi(a| s)\kp(\cdot | s,a), \sum_a \pi(a | s)\tilde{\kp}(\cdot | s,a)\Bigr)
\nonumber\\
&\leq \sum_a \pi(a | s)\mathrm{TV}\bigl(\kp(\cdot | s,a),\tilde{\kp}(\cdot |  s,a)\bigr) \nonumber\\
& \leq \delta,
\end{align}
by convexity of $\mathrm{TV}(\cdot,\cdot)$ in each argument. Hence
\begin{equation} \label{eq:TVpiuncertainty}
\quad \mathcal P^\pi \subseteq  \Bigl\{\,M\ \text{row–stochastic}:\ \forall s,\ \mathrm{TV}\bigl(M(s,\cdot), \tilde{\kp}^\pi(s,\cdot)\bigr)\le \delta \Bigr\}. \quad
\end{equation}
\begin{lemma} \label{lem:TVradius}
    Under the TV distance uncertainty set, if the centroid $\tilde{\kp}^\pi$ is irreducible and aperiodic, then there exists $m\in\mathbb N$ such that $(\tilde{\kp}^\pi)^m$ is strictly positive. Define $b_0 = \min_{i,s}((\tilde{\kp}^\pi)^m)_{is} >0$, then if the radius satisfies $\delta < \frac{b_0}{m}$, the joint spectral radius of $Q_\cp^\pi$ defined in \eqref{eq:Qfamilydef} is strictly less than $1$. Furthermore, $\kp^\pi$ is irreducible and aperiodic for all $\kp \in \cp$.
\begin{proof}

Define the $m$–step constant $a_0$ as
\[
a_0 = \min_{i<j}\ \sum_{s\in\mathcal{S}} \min\bigl\{(\tilde{\kp}^\pi)^m)_{is},((\tilde{\kp}^\pi)^m)_{js}\bigr\}\ \in(0,1],
\]
then $a_0 \ge S\,b_0$ where $S = |\mathcal{S}|$.

Regarding the joint spectral radius, for any length-$m$ product $\kp^\pi_{m}\cdots \kp^\pi_{1}$ with $\kp^\pi_{t}\in\mathcal \cp^\pi$.
By a telescoping expansion and nonexpansiveness of TV under right–multiplication by a Markov kernel,
\[
\mathrm{TV}\bigl((\kp^\pi_{m}\cdots \kp^\pi_{1})(i,\cdot),(\tilde{\kp}^\pi)^m(i,\cdot)\bigr)\ \leq m\delta\qquad\text{for all rows }i.
\]
Then, for all $i\ne j$,
\begin{align}
\sum_s \min\{(\kp^\pi_{m}\cdots \kp^\pi_{1})_{is},(\kp^\pi_{m}\cdots \kp^\pi_{1})_{js}\} &\geq 1-\mathrm{TV}\bigl((\kp^\pi_{m}\cdots \kp^\pi_{1})(i,\cdot),(\kp^\pi_{m}\cdots \kp^\pi_{1})(j,\cdot)\bigr) \nonumber \\
&\geq a_0 - 2m\delta,
\end{align}
by the triangle inequality in TV. Hence
\begin{equation}
\tau((\kp^\pi_{m}\cdots \kp^\pi_{1}))=1-\min_{i<j}\sum_s \min\{(\kp^\pi_{m}\cdots \kp^\pi_{1})_{is},(\kp^\pi_{m}\cdots \kp^\pi_{1})_{js}\}\ \le\ 1-(a_0-2m\delta).
\end{equation}
By setting $\delta <\ \frac{a_0}{2m} $,
we have $\sup_{\kp_i\in\mathcal P}\tau((\kp^\pi_{m}\cdots \kp^\pi_{1}))<1$, and by Lemma~\ref{lem:JSRupperbound},
\begin{equation}
\hat\rho\bigl(Q^\pi_{\mathcal P}\bigr)
\ \le\ \Bigl(\sup_{\kp_i\in\mathcal P}\tau((\kp^\pi_{m}\cdots \kp^\pi_{1}))\Bigr)^{1/m}
\ \le\ \bigl(1-(a_0-2m\delta)\bigr)^{1/m}\ <\ 1.
\end{equation}
Similarly, the same perturbation bound yields
\[
\min_s (\kp^\pi)^m_{is}\ \ge\ \min_s ((\tilde{\kp}^\pi))^m)_{is} - m \delta \ge\ b_0 - m \delta,
\]
so by setting $ \delta < \frac{b_0}{m}$, we have that $(\kp^\pi)^m$ is strictly positive for every $\kp\in\cp$; hence all induced kernels are irreducible and aperiodic. Since $a_0\ge Sb_0$, we have $\tfrac{a_0}{2m}\ge \tfrac{b_0}{m}$ for $S\ge2$. Therefore the  condition that $\delta < \frac{b_0}{m}$ satisfies both requirements.
\end{proof}
\end{lemma}

\paragraph{Wasserstein Distance Uncertainty}
Regarding Wasserstein uncertainty with $p\geq 1$, let $(\mathcal S,d)$ be the finite metric space. The uncertainty set can be characterized as
\[
\cp \coloneqq \Bigl\{\kp : \forall(s,a),W_p \bigl(P(\cdot | s,a), \tilde{\kp}(\cdot| s,a);d\bigr) \leq \delta\Bigr\}.
\]
 For a fixed policy $\pi$, the induced state–transition matrix $\kp^\pi$ is expressed as 
\begin{equation} \label{eq:Wppi}
\kp^\pi(s,s') \coloneqq \sum_{a}\pi(a| s) \kp(s'| s,a) 
\end{equation}
For each state $s$, by joint convexity of $W_p^p(\cdot,\cdot;d)$,
\[
\begin{aligned}
W_p^p \Bigl(\kp^\pi(s,\cdot), \tilde{\kp}^{\pi}(s,\cdot);d\Bigr)
&= W_p^p \Bigl(\sum_a \pi(a | s)\kp(\cdot| s,a), \sum_a \pi(a| s)\tilde{\kp}(\cdot| s,a);d\Bigr)\\
&\le \sum_a \pi(a| s)\, W_p^p\bigl(\kp(\cdot| s,a),\tilde{\kp}(\cdot| s,a);d\bigr)
 \leq \delta^p,
\end{aligned}
\]
hence $W_p\bigl(\kp^\pi(s,\cdot),\tilde{\kp}^\pi(s,\cdot);d\bigr)\le\delta$ for all $s$, i.e.
\begin{equation}\label{eq:Wpiuncertainty}
 \mathcal P^\pi\ \subseteq  \bigl\{\,M\ \text{row–stochastic}:\ \forall s,\ W_p(M(s,\cdot),\tilde{\kp}^\pi(s,\cdot);d)\leq \delta \bigr\}.
\end{equation}
We now draw connection between \eqref{eq:Wpiuncertainty} and the TV version in \eqref{eq:TVpiuncertainty}. Since the state space is finite, denote $\delta_{\min}:=\min_{x\neq y} d(x,y)>0$. Then, for any distributions $u,v$, we have
\[
W_1(u,v;d)\ \ge\ \delta_{\min}\,\mathrm{TV}(u,v)\qquad\text{and}\qquad W_p(u,v;d)\ \ge\ W_1(u,v;d),
\]
which implies that
\begin{equation} \label{eq:W2TVreduction}
\mathrm{TV}(u,v)\ \le\ \frac{W_1(u,v;d)}{\delta_{\min}}\ \leq \frac{W_p(u,v;d)}{\delta_{\min}}.
\end{equation}
Therefore we can reduce \eqref{eq:Wpiuncertainty} into a TV distance uncertainty set characterized as follows:
\begin{equation} \label{eq:W2TVpiuncertainty}
\quad \mathcal P^\pi \subseteq  \Bigl\{\,M\ \text{row–stochastic}:\ \forall s,\ \mathrm{TV}\bigl(M(s,\cdot), \tilde{\kp}^\pi(s,\cdot)\bigr)\le \frac{\delta}{\delta_{\min}} \Bigr\}.
\end{equation}

\begin{lemma} \label{lem：Wradius}
    Under the Wasserstein distance uncertainty set, if the centroid $\tilde{\kp}^\pi$ is irreducible and aperiodic, then there exists $m\in\mathbb N$ such that $(\tilde{\kp}^\pi)^m$ is strictly positive. Define $b_0 = \min_{i,s}((\tilde{\kp}^\pi)^m)_{is} >0$ and $\delta_{\min}:=\min_{x\neq y} d(x,y)>0$, then if the radius satisfies $\delta < \frac{\delta_{\min}b_0}{m}$, the joint spectral radius of $Q_\cp^\pi$ defined in \eqref{eq:Qfamilydef} is strictly less than $1$. Furthermore, $\kp^\pi$ is irreducible and aperiodic for all $\kp \in \cp$.
\begin{proof}
This is a direct corollary of Lemma \ref{lem:TVradius} under the condition of \eqref{eq:W2TVpiuncertainty}.
\end{proof}
\end{lemma}
\noindent\emph{Remarks.}
(i) If $d$ is normalized so $\delta_{\min}=1$, the thresholds simplify accordingly. 
(ii) One can also argue via $W_p^p\ge \delta_{\min}^p\,\mathrm{TV}$, which gives the alternative (more conservative when $\varepsilon$ is small) choice $r=\delta^p/\delta_{\min}^p$; the linear reduction $r=\delta/\delta_{\min}$ above is sharper and suffices for the bounds.

\subsubsection{Extremal Norm Construction} \label{extremalnormconstruction}

Under the radius conditions of Lemma \ref{lem:contaminationradius}-\ref{lem：Wradius}, we have that :
\begin{equation}
    r^* \coloneqq  \hat{\rho}(Q^\pi_\cp )  < 1
\end{equation}
We follow similar process for constructing our desired semi-norm $\|\cdot\|_\cp$ as in Appendix \ref{proofspan-contraction} by first constructing a norm such that all $Q_\kp^\pi$ are strictly less then one under that norm. We choose $\alpha \in (r^*,1)$ and we follow the approach in \cite{wirth2002generalized} by constructing an extremal norm $\|\cdot\|_{\rm ext}$ as follows:
\begin{equation} \label{eq:extremalnormdef}
    \|x\|_{\rm ext} \coloneqq \sup_{k \geq 0} \sup_{Q_1,\ldots, Q_k \in Q^\pi_\cp}  \alpha^{-k} \|Q_k Q_{k-1} \ldots  Q_1 x \|_2 \quad \text{where} \quad Q^\pi_\cp = \{Q^\pi_\kp : \kp \in \cp\}
\end{equation}
Note that we follow the convention that $\|Q_k Q_{k-1} \ldots  Q_1 x \|_2 = \|x\|_2$ when $k=0$.
\begin{lemma} \label{lem:extremalnormconstruction}
    Under Assumption \ref{ass:sameg} and the radius conditions of Lemma \ref{lem:contaminationradius}-\ref{lem：Wradius}, the operator $\|\cdot\|_{\rm ext}$ is a valid norm with $\|Q_\kp^\pi\|_{\rm ext}<1$ for all $\kp \in \cp$
    \begin{proof}
        We first prove that $\|\cdot\|_{\rm ext}$ is bounded. Following Lemma \ref{lem:bergerlemmaIV} and choosing $\lambda \in (r^*,\alpha)$, then there exist a positive constant $C < \infty$ such that 
        \begin{equation}
             \|Q_k Q_{k-1} \ldots  Q_1 \|_2 \leq C \lambda^k 
        \end{equation}
        Hence for each $k$ and for all $x \in \mathbb{R}^{S}$,
        \begin{equation}
            \alpha^{-k} \|Q_k Q_{k-1} \ldots  Q_1 x \|_2 \leq \alpha^{-k}C\lambda^k \|x\|_2 = C\left(\frac{\lambda}{\alpha}\right)^{-k}\|x\|_2 \longrightarrow 0 \quad \text{as} \quad k \rightarrow \infty
        \end{equation}
        Thus the double supremum in \eqref{eq:extremalnormdef} is over a bounded and vanishing sequence, so $\|\cdot\|_{\rm ext}$ bounded. 

        To check that $\|\cdot\|_{\rm ext}$ is a valid norm, note that if $x = \bf{0}$, $\|x\|_{\rm ext}$ is directly $0$. On the other hand, if $\|x\|_{\rm ext}=0$, we have
        \begin{equation}
            \sup_{k = 0} \sup_{Q_1,\ldots, Q_k \in Q^\pi_\cp}  \alpha^{-k} \|Q_k Q_{k-1} \ldots  Q_1 x \|_2 = \|x\|_2 = 0 \Rightarrow x = \bf{0}
        \end{equation}
        Regarding homogeneity, observe that for any $c\in \mathbb{R}$ and $x \in \mathbb{R}^{S}$,
        \begin{equation}
            \|cx\|_{\rm ext} = \sup_{k \geq 0} \sup_{Q_1,\ldots, Q_k \in Q^\pi_\cp}  \alpha^{-k} \|Q_k Q_{k-1} \ldots  Q_1 (cx) \|_2 = |c| \|x\|_{\rm ext}
        \end{equation}
        Regarding triangle inequality, using $\|Q_k \ldots  Q_1 (x+y) \|_2 \leq \|Q_k \ldots  Q_1 x\|_2+\|Q_k \ldots  Q_1 y \|_2$ for any $x,y \in \mathbb{R}^S$, we obtain,
        \begin{equation}
            \|x+y\|_{\rm ext} = \sup_{k \geq 0} \sup_{Q_1,\ldots, Q_k \in Q^\pi_\cp}  \alpha^{-k} \|Q_k Q_{k-1} \ldots  Q_1 (x+y) \|_2 \leq \|x\|_{\rm ext}+\|y\|_{\rm ext}
        \end{equation}
        For any $\kp \in \cp$, we have
        \begin{align} \label{eq:extremalcontraction}
            \|Q_\kp^\pi x\|_{\rm ext} &= \sup_{k \geq 0} \sup_{Q_1,\ldots, Q_k \in Q^\pi_\cp}  \alpha^{-k} \|Q_k Q_{k-1} \ldots  Q_1 (Q_\kp^\pi x) \|_2 \nonumber \\
            &\leq \sup_{k \geq 1} \sup_{Q_1,\ldots, Q_k \in Q^\pi_\cp}  \alpha^{-(k-1)} \|Q_k Q_{k-1} \ldots  Q_1 x \|_2 \nonumber \\
            &= \alpha \sup_{k \geq 1} \sup_{Q_1,\ldots, Q_k \in Q^\pi_\cp}  \alpha^{-k} \|Q_k Q_{k-1} \ldots  Q_1 x \|_2 \nonumber \\
            &\leq \alpha \sup_{k \geq 0} \sup_{Q_1,\ldots, Q_k \in Q^\pi_\cp}  \alpha^{-k} \|Q_k Q_{k-1} \ldots  Q_1 x \|_2 \nonumber \\
            &= \alpha \|x\|_{\rm ext}
        \end{align}
        Since $\kp$ is arbitrary, \eqref{eq:extremalcontraction} implies that for any $\kp \in \cp$,
        \begin{equation}
            \|Q_\kp^\pi\|_{\rm ext} = \sup_{x\neq \bf{0}}\frac{\|Q_\kp^\pi x\|_{\rm ext}}{\|x\|_{\rm ext}} \leq \alpha < 1
        \end{equation}
    \end{proof}
    \end{lemma}

\subsubsection{Semi-Norm Contraction for Robust Bellman Operator}
    We now follow the same method as \eqref{eq:constructionofPnorm} to construct the semi-norm $\|\cdot\|_\cp$. Define the operator $\|\cdot\|_\cp$ as
\begin{equation} \label{eq:robustseminormconstruction}
    \|x\|_{\cp} \coloneqq \sup_{\kp\in\cp}\bigl\|Q_\kp^\pi x\bigr\|_{\rm ext}+\epsilon \inf_{c\in\mathbb R}\bigl\|x - c \e\bigr\|_{\rm ext}
\end{equation}
where $0<\epsilon<1-\alpha$. 
\begin{lemma} \label{lem:cpnormcontraction}
    The operator $\|\cdot\|_\cp$ is a valid semi-norm with kernel being exactly $\{c\e : c\in\mathbb{R}\}$ under Assumption \ref{ass:sameg} and the radius conditions of Lemma \ref{lem:contaminationradius}-\ref{lem：Wradius}. Furthermore, for all $x\in \mathbb{R}^S$, we have $\bigl\|\kp^\pi x\bigr\|_{\cp} \leq (\alpha+\epsilon)\bigl\| x\bigr\|_{\cp}$ for all $\kp \in \cp$.
    \begin{proof}
    Regarding positive homogeneity and nonnegativity,  for any scalar $\lambda$ and $x\in\mathbb{R}^S$,
    \[
    \|\lambda x\|_{\cp}
    = \sup_{\kp\in\cp}\bigl\|Q_\kp^\pi (\lambda x)\bigr\|_{\rm ext}+\epsilon \inf_{c\in\mathbb R}\bigl\|\lambda x - c \e\bigr\|_{\rm ext}
    = |\lambda| \sup_{\kp\in\cp}\bigl\|Q_\kp^\pi x \bigr\|_{\rm ext}+\epsilon |\lambda| \inf_{c\in\mathbb R}\bigl\| x - c \e\bigr\|_{\rm ext}
    =|\lambda|\|x\|_{\rm ext}
    \]
    and $\|x\|_{\rm ext}\ge0$. Regarding triangle inequality, for any $x,y\in\mathbb{R}^S$, note that for any $\kp \in \cp$,
    \begin{equation}
        \|Q^\pi_\kp (x+y)\|_{\rm ext} \leq \|Q^\pi_\kp x\|_{\rm ext} + \|Q^\pi_\kp y\|_{\rm ext}
    \end{equation}
    Taking supremum over $\kp$ on both sides yields
    \begin{equation}
        \sup_{\kp\in \cp} \|Q^\pi_\kp (x+y)\|_{\rm ext} \leq \sup_{\kp\in \cp} \|Q^\pi_\kp x\|_{\rm ext} + \sup_{\kp\in \cp} \|Q^\pi_\kp y\|_{\rm ext}
    \end{equation}
    Thus, we have
    \begin{align*}
    \|x+y\|_{\cp}
    &= \sup_{\kp\in\cp}\bigl\|Q_\kp^\pi ( x+y)\bigr\|_{\rm ext}+\epsilon \inf_{c}\bigl\| x+y - c \e\bigr\|_{\rm ext}\\
    &\le   \sup_{\kp\in \cp} \|Q^\pi_\kp x\|_{\rm ext} + \sup_{\kp\in \cp} \|Q^\pi_\kp y\|_{\rm ext}
       + \epsilon\inf_{a,b}\bigl\|x - a\e + y - b\e\bigr\|_{\rm ext}\\
    &\le \sup_{\kp\in \cp} \|Q^\pi_\kp x\|_{\rm ext} + \sup_{\kp\in \cp} \|Q^\pi_\kp y\|_{\rm ext}
       + \epsilon\inf_{a}\bigl\|x - a\e\bigr\|_{\rm ext}
       + \epsilon\inf_{b}\bigl\|y - b\e\bigr\|_{\rm ext}\\
    &= \|x\|_{\cp} + \|y\|_{\cp}.
    \end{align*} Regarding the kernel, if $x=k \e$ for some $k\in \mathbb{R}$, then similar to \eqref{eq:kernelseminorm}, we have
    \begin{equation}
        \|x\|_\cp =  \sup_{\kp\in\cp}\bigl\|Q_\kp^\pi ( k \e)\bigr\|_{\rm ext} +\epsilon\inf_{c}\|k\e-c\e\|_{\rm ext}=0+0=0
    \end{equation}
    On the other hand, if $x\notin \{c\e : c\in\mathbb{R}\}$, we know that 
    \begin{equation}
        \|x\|_\cp \geq \epsilon\inf_{c}\|x-c\e\|_{\rm ext} >0
    \end{equation}    
    Thus, the kernel of $\|\cdot\|_\cp$ is exactly $\{c\e : c\in\mathbb{R}\}$. We now show that, for any $x\in\mathbb{R}^S$ and $\kp \in \cp$,
    \begin{align}
    \bigl\|\kp^\pi x\bigr\|_{\cp}
    &= \sup_{Q \in\ Q_\cp^\pi}\bigl\|Q  \kp^\pi x\bigr\|_{\rm ext} +\epsilon\inf_{c}\bigl\|\kp^\pi x - c\e\bigr\|_{\rm ext}
    \nonumber \\
    &=     \sup_{Q \in\ Q_\cp^\pi}\bigl\|Q  Q_\kp^\pi x + Q E_\kp x\bigr\|_{\rm ext} +\epsilon\inf_{c}\bigl\| Q_\kp^\pi x + E_\kp x - c\e\bigr\|_{\rm ext}    \nonumber \\
    &=     \sup_{Q \in\ Q_\cp^\pi}\bigl\|Q  Q_\kp^\pi x \|_{\rm ext} +\epsilon\inf_{c}\bigl\| Q_\kp^\pi x  - c\e\bigr\|_{\rm ext}    \nonumber \\
    &\leq \alpha  \sup_{Q \in\ Q_\cp^\pi}\bigl\|Q x \|_{\rm ext}
       +\epsilon \|Q^\pi x\|_{\rm ext} \nonumber \\
    &= (\alpha+\epsilon)\|Q^\pi x\|_{\rm ext} \nonumber \\
    &\leq (\alpha+\epsilon)\|x\|_{\cp}.
    \end{align}
    \end{proof}
\end{lemma}

Since $\alpha\in(0,1)$ and $\epsilon \in (0,1-\alpha)$. Thus, let $\gamma = \alpha+\epsilon$ then $\gamma \in (0,1)$. Substituting the above result back to \eqref{eq:robustbellmanstep1}, we obtain
\begin{equation}
     \|\mathbf{T}_g(V_1) -  \mathbf{T}_g(V_2) \|_\cp \leq \| \pi(a|s) \Tilde{p}_{(V_1,V_2)} \bigl[V_1(s')-V_2(s')\bigr] \|_\cp \leq \gamma \|V_1 - V_2 \|_\cp
\end{equation}


\section{Biased Stochastic Approximation Convergence Rate}  \label{seminormcontractionwithbias}
In Section \ref{contractionsection}, we established that the robust Bellman operator is a contraction under the semi-norm $\|\cdot\|_\cp$, ensuring that policy evaluation can be analyzed within a well-posed stochastic approximation framework. However, conventional stochastic approximation methods typically assume unbiased noise, where variance diminishes over time without introducing systematic drift. In contrast, the noise in robust policy evaluation under TV and Wasserstein distance uncertainty sets exhibits a small but persistent bias, arising from the estimators of the support functions $\hat{\sigma}_{\cp^a_s}(V)$ (discussed in Section \ref{QueriesSection}). This bias, if not properly addressed, can lead to uncontrolled error accumulation, affecting the reliability of policy evaluation. To address this challenge, this section introduces a novel analysis of biased stochastic approximation, leveraging properties of dual norms to ensure that the bias remains controlled and does not significantly impact the convergence rate. Our results extend prior work on unbiased settings and provide the first explicit finite-time guarantees, which are further used to establish the sample complexity of policy evaluation in robust average-reward RL. Specifically, we analyze the iteration complexity for solving the fixed equivalence class equation $H(x^*) - x^* \in \overline{E}$ where $\overline{E}\coloneqq \{c \mathbf{e} : c \in \mathbb{R}\}$ with $\mathbf{e}$ being the all-ones vector. The stochastic approximation iteration being used is as follows:
\begin{equation}\label{eq:SA-update}
   x^{t+1}=x^t + \eta_t \bigl[\widehat{H}(x^t) - x^t\bigr],
   \quad
   \text{where}\quad
   \widehat{H}(x^t)=H(x^t) + w^t.
\end{equation}
with $\eta_t>0$ being the step-size sequence. We assume that there exist $\gamma \in (0,1)$ such that
\begin{equation} \label{eq:Hcontraction}
     \|H(x) - H(y)\|_{\cp}\leq \gamma\,\|x - y\|_{\cp}, \quad 
  \forall x, y
\end{equation}
We also assume that the noise terms $\omega^t$ are i.i.d. and have bounded bias and variance
\begin{equation}  \label{eq:omegabounded}
    \mathbb{E}[\,\|w^t\|_{\cp}^2 | \mathcal{F}^t] \le A + B\,\|x^t - x^*\|_{\cp}^2  \quad  \text{and}\quad \bigl\|\mathbb{E}[\,w^t | \mathcal{F}^t]\bigr\|_{\cp} \le \varepsilon_{\mathrm{bias}}
\end{equation}

\begin{theorem} \label{thm:informalbiasedSA}
   If $x^t$ is generated by \eqref{eq:SA-update} with all assumptions in \eqref{eq:Hcontraction} and \eqref{eq:omegabounded} satisfied, then if the stepsize $\eta_t \coloneqq \cO(\frac{1}{t})$,
    \begin{equation} \label{eq:biasedSA}
        \mathbb{E}\Bigl[\|x^T - x^*\|^2_{\cp}\Bigr] \leq  \cO\left(\frac{1}{T^2}\right)\|x^0 - x^*\|^2_{\cp} + \cO\left(\frac{A}{(1-\gamma)^2T}\right) +  \cO\left(\frac{x_{\mathrm{sup}} \varepsilon_{\text{bias}} \log T }{1-\gamma} \right)
    \end{equation}
    where  $x_{\mathrm{sup}} \coloneqq \sup_x \|x\|_{\cp}$ is the upper bound of the $\|\cdot\|_{\cp}$ semi-norm for all $x^t$.
\end{theorem}
Theorem \ref{thm:informalbiasedSA} adapts the analysis of \citep{zhang2021finite} and extends it to a biased i.i.d. noise setting. To manage the bias terms, we leverage properties of dual norms to bound the inner product between the error term and the gradient, ensuring that the bias influence remains logarithmic in $T$ rather than growing unbounded, while also carefully structuring the stepsize decay to mitigate long-term accumulation. This results in an extra $\varepsilon_{\mathrm{bias}}$ term with logarithmic dependence of the total iteration $T$. 

We perform analysis of the biased‐noise extension to the semi-norm stochastic approximation (SA) problem by constructing a smooth convex semi-Lyapunov function for forming the negative drift \cite{zhang2021finite, chen2025non} and using properties in dual norms for managing the bias.
\medskip

\subsection{Proof of Theorem \ref{thm:informalbiasedSA}} 
\subsubsection{Setup and Notation.} \label{setups}

In this section, we override the notation of the semi-norm $\|\cdot\|_\cp$ by re-writing it as the norm $\|\cdot\|_\mathcal{N}$ (defined in \eqref{eq:normN}) to the equivalence class of constant vectors. For any norm $\|\cdot\|_c$ and equivalence class ${\overline{E}}$, define the indicator function  $\delta_{\overline{E}}$ as
\begin{equation}
\delta_{\bar{E}}(x) := 
\begin{cases} 
    0 & x \in \bar{E}, \\
    \infty & \text{{otherwise}.}
\end{cases}
\end{equation}
\noindent then by \cite{zhang2021finite}, the semi-norm induced by norm  $\|\cdot\|_c$  and equivalence class $\overline{E}$ is the infimal convolution of $\|\cdot\|_c$ and the indicator function $\delta_{\overline{E}}$ can be defined as follows 
\begin{equation}
    \|x\|_{c,\overline{E}} \coloneqq (\|\cdot\|_c \ast_{\inf} \delta_{\overline{E}})(x) = \inf_y  (\|x-y\|_c + \delta_{\overline{E}}(y))= \inf_{e\in\overline{E}} \|x-e\|_c \quad \forall x,
\end{equation}
where $\ast_{\inf}$ denotes the infimal convolution operator. Throughout the remaining section, we let $\overline{E}\coloneqq \{c \mathbf{e} : c \in \mathbb{R}\}$ with $\mathbf{e}$ being the all-ones vector. Since $\|\cdot\|_\cp$ constructed in \eqref{eq:robustseminormconstruction} is a semi-norm with kernel being $\overline{E}$, we can construct a norm $\|\cdot\|_{\mathcal{N}}$ such that 
\begin{equation} \label{eq:normNproperty}
\|x\|_{\mathcal{N},\overline{E}} \coloneqq  (\|\cdot\|_{\mathcal{N}} \ast_{\inf} \delta_{\overline{E}})(x)=\|x\|_\cp
\end{equation}
We construct $\|\cdot\|_{\mathcal{N}}$ as follows:
\begin{equation} \label{eq:normN}
    \|x\|_{\mathcal{N}} \coloneqq \sup_{\kp\in\cp}\bigl\|Q_\kp^\pi x\bigr\|_{\rm ext}+\epsilon \inf_{c\in\mathbb R}\bigl\|x - c \e\bigr\|_{\rm ext} + \epsilon \|x\|_{\rm ext}
\end{equation}
where $Q_\kp^\pi$ and $\epsilon$ are defined in \eqref{eq:robustseminormconstruction}.
\begin{lemma}
    The operator $\|\cdot\|_{\mathcal{N}}$ defined in \eqref{eq:normN} is a norm satisfying \eqref{eq:normNproperty}.
    \begin{proof}
        We first verify that $\|\cdot\|_{\mathcal{N}}$ is a norm. Regarding positivity, since all terms in \eqref{eq:normN} are non-negative, $\|x\|_{\mathcal{N}} \geq 0$ for all $x\in\mathbb{R}^S$ and $\|\mathbf{0}\|_{\mathcal{N}}=0$. If $x \neq \bf{0}$, since $\|\cdot\|_{\rm ext}$ is a valid norm and $\epsilon>0$, we have
        $$
        \|x\|_{\mathcal{N}} \geq  \epsilon \|x\|_{\rm ext} > 0.
        $$
        Regarding homogeneity,  For any \(\lambda\in \mathbb{R}\), we have
        \begin{align*}
      \|\lambda x\|_{\mathcal{N}}&=\sup_{\kp\in\cp}\bigl\|Q_\kp^\pi (\lambda x)\bigr\|_{\rm ext}+\epsilon \inf_{c}\bigl\|(\lambda x) - c \e\bigr\|_{\rm ext} + \epsilon \|\lambda x\|_{\rm ext} \\
      &= |\lambda| \sup_{\kp\in\cp}\bigl\|Q_\kp^\pi  x \bigr\|_{\rm ext}+\epsilon |\lambda| \inf_{c}\bigl\|x - c \e\bigr\|_{\rm ext} + \epsilon |\lambda| \|x\|_{\rm ext}\\
      &= |\lambda| \|x\|_{\mathcal{N}}
        \end{align*}
      Regarding triangle inequality, for any $x,y\in \mathbb{R}^S$, we have
      \begin{align*}
          \|x+y\|_\mathcal{N} &=\sup_{\kp\in\cp}\bigl\|Q_\kp^\pi (x+y)\bigr\|_{\rm ext}+\epsilon \inf_{c}\bigl\|( x+y) - c \e\bigr\|_{\rm ext} + \epsilon \| x+y\|_{\rm ext} \\
           &\le   \sup_{\kp\in \cp} \|Q^\pi_\kp x\|_{\rm ext} + \sup_{\kp\in \cp} \|Q^\pi_\kp y\|_{\rm ext}
       + \epsilon\inf_{a,b}\bigl\|x - a\e + y - b\e\bigr\|_{\rm ext}+ \epsilon (\|x\|_{\rm ext}+\|y\|_{\rm ext})\\
    &\le \sup_{\kp\in \cp} \|Q^\pi_\kp x\|_{\rm ext} + \sup_{\kp\in \cp} \|Q^\pi_\kp y\|_{\rm ext}
       + \epsilon\inf_{a}\bigl\|x - a\e\bigr\|_{\rm ext}
       + \epsilon\inf_{b}\bigl\|y - b\e\bigr\|_{\rm ext}+ \epsilon \|x\|_{\rm ext}+\|y\|_{\rm ext}\\
    &= \|x\|_{\mathcal{N}} + \|y\|_{\mathcal{N}}.
      \end{align*}
We now show that since \(Q^\pi_\kp \e=0\) for all $\kp\in\cp$, by the definition of infimal convolution, we have that for all $x\in\mathbb{R}^S$,
\begin{align*}
(\|\cdot\|_{\mathcal{N}} \ast_{\inf} \delta_{\overline{E}})(x)
&=\inf_{k\in\mathbb{R}} \|x-k\e\|_\mathcal{N}\\
&=\inf_{k\in\mathbb{R}} \left(\sup_{\kp\in\cp}\bigl\|Q_\kp^\pi x - k Q_\kp^\pi \e \bigr\|_{\rm ext}+\epsilon \inf_{c\in\mathbb R}\bigl\|x - c \e-k\e\bigr\|_{\rm ext} + \epsilon \|x-k\e\|_{\rm ext}\right)\\
&=\inf_{k\in\mathbb{R}} \left(\sup_{\kp\in\cp}\bigl\|Q_\kp^\pi x \bigr\|_{\rm ext}+\epsilon \inf_{c\in\mathbb R}\bigl\|x - c \e\bigr\|_{\rm ext} + \epsilon \|x-k\e\|_{\rm ext}\right)\\
&=\sup_{\kp\in\cp}\bigl\|Q_\kp^\pi x \bigr\|_{\rm ext}+\epsilon \inf_{c\in\mathbb R}\bigl\|x - c \e\bigr\|_{\rm ext} +\inf_{k\in\mathbb{R}} \left( \epsilon \|x-k\e\|_{\rm ext}\right)\\
&=\|x\|_\cp
\end{align*}
    \end{proof}
\end{lemma}

We thus restate our problem of analyzing the iteration complexity for solving the fixed equivalence class equation $H(x^*) - x^* \in \overline{E}$, with the operator \(H:\mathbb{R}^n \to \mathbb{R}^n\) satisfying the contraction property as follows:
\begin{equation}
  \|H(x) - H(y)\|_{\mathcal{N},\overline{E}} \leq \gamma\|x - y\|_{\mathcal{N},\overline{E}},
  \quad
  \gamma\in(0,1),
  \quad
  \forall x, y
\end{equation}
The stochastic approximation iteration being used is as follows
\begin{equation}\label{eq:SA-updategeneral}
   x^{t+1}= x^t + \eta_t \bigl[\widehat{H}(x^t) - x^t\bigr],
   \quad
   \text{where}\quad
   \widehat{H}(x^t)=  H(x^t) + w^t.
\end{equation}
We assume:
\begin{itemize}
\item \(\mathbb{E}[\,\|w^t\|_{\mathcal{N},\overline{E}}^2 | \mathcal{F}^t] \le A + B\|x^t - x^*\|_{\mathcal{N},\overline{E}}^2\) (In the robust average-reward TD case, $B=0$).
\item \(\bigl\|\mathbb{E}[w^t | \mathcal{F}^t]\bigr\|_{\mathcal{N},\overline{E}} \le \varepsilon_{\mathrm{bias}}\).   
\item \(\eta_t>0\) is a chosen stepsize sequence (decreasing or constant).
\end{itemize}

Note that beside the bias in the noise, the above formulation and assumptions are identical to the unbiased setups in Section B of \cite{zhang2021finite}. Thus, we emphasize mostly on managing the bias.

\subsubsection{Semi‐Lyapunov \texorpdfstring{$M_{\overline{E}}(\cdot)$} M and Smoothness.}

By \cite[Proposition\,1--2]{zhang2021finite}, using the Moreau envelope function $M(x)$ in Definition 2.2 of \cite{chen2020finite}, we define
\[
  M_{\overline{E}}(x)=\bigl(M \ast_{\inf} \delta_{\overline{E}}\bigr)(x),
\]
so that there exist \(c_l,c_u>0\) with
\begin{equation} \label{eq:M2span}
  c_l   M_{\overline{E}}(x) \leq  \frac{1}{2}\|x\|_{\mathcal{N},\overline{E}}^2 \leq  c_u  M_{\overline{E}}(x),
\end{equation}
and \(M_{\overline{E}}\) is \(L\)-smooth \emph{w.r.t.\ another semi-norm} \(\|\cdot\|_{s,\overline{E}}\).  
Concretely, \(L\)-smoothness means:
\begin{equation} \label{eq:lsmooth}
  M_{\overline{E}}(y) \leq M_{\overline{E}}(x)+
  \langle \nabla M_{\overline{E}}(x),\,y - x\rangle+\tfrac{L}{2}\,\|\,y-x\|_{s,\overline{E}}^2,
  \quad
  \forall\,x,y.
\end{equation}
Moreover, the gradient of $M_{\overline{E}}$ satisfies $\langle \nabla M_{\overline{E}}(x),\; c\,\mathbf{e}\rangle = 0$ for all $x$, and the dual norm denoted as \(\|\cdot\|_{*,s,\overline{E}}\) is also L-smooth:
\begin{equation}  \label{eq:duallsmooth}
  \| \nabla M_{\overline{E}}(x) - \nabla M_{\overline{E}}(y)\|_{*,s,\overline{E}} \leq 
   L\|\,y-x\|_{s,\overline{E}},
  \quad
  \forall\,x,y.
\end{equation}

Note that since $\|\cdot\|_{s,\overline{E}}$ and $\|\cdot\|_{\mathcal{N},\overline{E}}$ are semi-norms on a finite‐dimensional space with the same kernel, there exist $\rho_1,\rho_2>0$ such that
\begin{equation} \label{eq:normequivalence}
  \rho_1\,\|z\|_{\mathcal{N},\overline{E}} \leq \|z\|_{s,\overline{E}} \leq \rho_2\,\|z\|_{\mathcal{N},\overline{E}},
  \;
  \forall\,z.
\end{equation}
Likewise, their dual norms (denoted \(\|\cdot\|_{*,s,\overline{E}}\) and \(\|\cdot\|_{*,\mathcal{N},\overline{E}}\)) satisfy the following:
\begin{equation} \label{eq:dualnormequivalence}
  \frac{1}{\rho_2}\|z\|_{*,s,\overline{E}} \leq \|z\|_{*,\mathcal{N},\overline{E}} \leq \frac{1}{\rho_1}\|z\|_{*,s,\overline{E}}, \;
  \forall\,z.
\end{equation}

\subsubsection{Formal Statement of Theorem \ref{thm:informalbiasedSA}} \label{appendix4biasedSA}

By $L$‐smoothness w.r.t.\ $\|\cdot\|_{s,\overline{E}}$ in \eqref{eq:lsmooth}, for each $t$,
\begin{equation} \label{eq:Mt+1decomposition}
  M_{\overline{E}}(x^{t+1} - x^*) \leq M_{\overline{E}}(x^t - x^*)+\bigl\langle \nabla M_{\overline{E}}(x^t-x^*),x^{t+1}-x^t \bigr\rangle+\tfrac{L}{2}\|\,x^{t+1}-x^t\|_{s,\overline{E}}^2.
\end{equation}
where $x^{t+1}-x^t = \eta_t[\widehat{H}(x^t)-x^t] = \eta_t[H(x^t) + w^t - x^t]$. Taking expectation of the second term of the RHS of \eqref{eq:Mt+1decomposition} conditioned on the filtration $\mathcal{F}^t$ we obtain,
\begin{align} \label{eq:nablaMt+1decomposition}
    \mathbb{E}[\langle \nabla M_{\overline{E}}(x^t - x^*), &x^{t+1} - x^t \rangle | \mathcal{F}^t] = \eta_t \mathbb{E}[\langle \nabla M_{\overline{E}}(x^t - x^*), H(x^t) - x^t + \omega^t \rangle| \mathcal{F}^t] \nonumber\\
    &= \eta_t\langle \nabla M_{\overline{E}}(x^t - x^*), H(x^t) - x^t \rangle + \eta_t\mathbb{E}[\langle \nabla M_{\overline{E}}(x^t - x^*),\omega^t \rangle | \mathcal{F}^t] \nonumber\\
     &= \eta_t\langle \nabla M_{\overline{E}}(x^t - x^*), H(x^t) - x^t \rangle + \eta_t\langle \nabla M_{\overline{E}}(x^t - x^*), \mathbb{E}[\omega^t | \mathcal{F}^t] \rangle .
\end{align}
To analyze the additional bias term $\langle \nabla M_{\overline{E}}(x^t - x^*), \mathbb{E}[\omega^t | \mathcal{F}^t] \rangle$, we use the fact that for any (semi-)norm $\|\cdot\|$ with dual (semi-)norm $\|\cdot\|_*$ (defined by 
$\|u\|_* = \sup\{\langle u,v\rangle : \|v\|\le1\}$),
we have the general inequality
\begin{equation}\label{eq:dualNormIneq}
  \bigl\langle u,\,v\bigr\rangle\leq\|u\|_{*}\,\|v\|,
  \quad
  \forall\,u,v.
\end{equation}
In the biased noise setting, $u=\nabla M_{\overline{E}}(x^t - x^*)$ and $v=\mathbb{E}[w^t| \mathcal{F}^t]$, with $\|\cdot\|=\|\cdot\|_{\mathcal{N},\overline{E}}$.  So
\begin{equation} \label{eq:inner_product_bound}
  \bigl\langle 
    \nabla M_{\overline{E}}(x^t - x^*),
    \mathbb{E}[w^t| \mathcal{F}^t]
  \bigr\rangle \leq \bigl\|\nabla M_{\overline{E}}(x^t - x^*)\bigr\|_{*,\,\mathcal{N},\overline{E}}\cdot \bigl\|\mathbb{E}[w^t| \mathcal{F}^t]\bigr\|_{\mathcal{N},\overline{E}}.   
\end{equation}
Since $\|\mathbb{E}[w^t| \mathcal{F}^t]\|_{\mathcal{N},\overline{E}} \le \varepsilon_{\text{bias}}$, it remains to bound 
\(\|\nabla M_{\overline{E}}(x^t - x^*)\|_{*,\mathcal{N},\overline{E}}\). By setting 
 $y=0$ in \eqref{eq:duallsmooth}, we get
\begin{equation}  
  \| \nabla M_{\overline{E}}(x) - \nabla M_{\overline{E}}(0)\|_{*,s,\overline{E}} \leq 
   L\|  x \|_{s,\overline{E}},
  \quad
  \forall\,x.
\end{equation}
Thus,
\begin{equation} \label{eq:step1}
     \| \nabla M_{\overline{E}}(x) \|_{*,s,\overline{E}} \leq  \| \nabla M_{\overline{E}}(0) \|_{*,s,\overline{E}} +
   L\| x \|_{s,\overline{E}},
  \quad
  \forall\,x.
\end{equation}
By \eqref{eq:dualnormequivalence}, we know that there exists $\frac{1}{\rho_2}\leq \alpha \leq \frac{1}{ \rho_1}$ such that
\begin{equation} \label{eq:step2}
    \|  \nabla M_{\overline{E}}(x) \|_{*,\mathcal{N},\overline{E}} \leq \alpha   \|  \nabla M_{\overline{E}}(x) \|_{*,s,\overline{E}} 
\end{equation}
Thus, combining \eqref{eq:step1} and \eqref{eq:step2} would give:
\begin{equation} 
     \| \nabla M_{\overline{E}}(x) \|_{*,\mathcal{N},\overline{E}} \leq  \alpha\big(\| \nabla M_{\overline{E}}(0) \|_{*,s,\overline{E}} +
   L\| x \|_{s,\overline{E}}\big),
  \quad
  \forall\,x.
\end{equation}
By \eqref{eq:normequivalence}, we know that $\| x \|_{s,\overline{E}} \leq \| x \|_{\mathcal{N},\overline{E}}$, thus we have:
\begin{equation} 
     \| \nabla M_{\overline{E}}(x) \|_{*,\mathcal{N},\overline{E}} \leq  \alpha\big(\| \nabla M_{\overline{E}}(0) \|_{*,s,\overline{E}} +
   L\rho_2\| x \|_{\mathcal{N},\overline{E}}\big),
  \quad
  \forall\,x.
\end{equation}
Hence, combining the above with \eqref{eq:inner_product_bound}, there exist some 
\begin{equation} \label{eq:G_value}
G=\mathcal{O}\big(\frac{1}{\rho_1} \max\{L\rho_2, \| \nabla M_{\overline{E}}(0) \|_{*,s,\overline{E}}\} \big)
\end{equation}
such that
\begin{equation}  \label{eq:Gepsilonbias}
  \mathbb{E}
  \Bigl[\langle 
    \nabla M_{\overline{E}}(x^t - x^*),\,w^t
  \rangle|\mathcal{F}^t\Bigr]=\bigl\langle 
    \nabla M_{\overline{E}}(x^t - x^*),
    \mathbb{E}[w^t|\mathcal{F}^t]
  \bigr\rangle\leq G
  \bigl(1 + \|x^t - x^*\|_{\mathcal{N},\overline{E}}\bigr)
  \varepsilon_{\text{bias}}.
\end{equation}
Combining \eqref{eq:Gepsilonbias} with \eqref{eq:nablaMt+1decomposition} we obtain
\begin{align} \label{eq:nablaMt+1decomposition2}
     \mathbb{E}[\langle \nabla M_{\overline{E}}(x^t - x^*), x^{t+1} - x^t \rangle | \mathcal{F}^t] \leq \eta_t&\langle \nabla M_{\overline{E}}(x^t - x^*), H(x^t) - x^t \rangle \nonumber \\
     &+ \eta_tG\varepsilon_{\text{bias}}  \Bigl(1 + \|x^t - x^*\|_{\mathcal{N},\overline{E}}\Bigr)
\end{align}
To bound the first term in the RHS of \eqref{eq:nablaMt+1decomposition2}, note that
\begin{align} \label{eq:nablaMtdecomposition}
    \langle \nabla M_{\bar{E}}(x^t - x^*), H(x^t) - x^t \rangle 
    &= \langle \nabla M_{\bar{E}}(x^t - x^*), H(x^t) - x^* + x^* - x^t \rangle \nonumber\\
    &\overset{(a)}{\leq}M_{\bar{E}}(H(x^t) - x^*) - M_{\bar{E}}(x^t - x^*) \nonumber\\
    &\overset{(b)}{\leq} \frac{1}{2 c_l} \| H(x^t) - H(x^*) \|^2_{c, \bar{E}} - M_{\bar{E}}(x^t - x^*)\nonumber\\
    &\overset{(c)}{\leq} \frac{\gamma^2}{2 c_l} \| x^t - x^* \|^2_{c, \bar{E}} - M_{\bar{E}}(x^t - x^*) \nonumber\\
    &\leq \left( \frac{\gamma^2 c_u}{c_l} - 1 \right) M_{\bar{E}}(x^t - x^*) \nonumber\\
    &\leq -(1 - \gamma \sqrt{c_u / c_l}) M_{\bar{E}}(x^t - x^*),
\end{align}
where \((a)\) follows from the convexity of $ M_{\bar{E}} $, \((b)\) follows from \( x^* \) belonging to a fixed equivalence class with respect to \( H \) and \((c)\) follows from the contraction property of \( H \). Combining \eqref{eq:nablaMtdecomposition}. \eqref{eq:nablaMt+1decomposition2} and Lemma \ref{lem:zhanglemma6} with \eqref{eq:Mt+1decomposition}, we arrive as follows:
\begin{align} \label{eq:modified_proposition3}
    \mathbb{E}\Bigl[ 
     M_{\overline{E}}(x^{t+1} - x^*) | \mathcal{F}_t\Bigr] & \leq (1-2\alpha_2 \eta_t +\alpha_3\eta_t^2)  M_{\overline{E}}(x^{t} - x^*) + \alpha_4 \eta_t^2 \nonumber \\
    &+ \eta_t G\varepsilon_{\text{bias}}  \Bigl(1 + \|x^t - x^*\|_{\mathcal{N},\overline{E}}\Bigr)
\end{align}
Where $\alpha_2 \coloneqq (1-\gamma \sqrt{c_{{u}} / c_{{l}}})$, $\alpha_3 \coloneqq (8+2B)c_{{u}} \rho_2 L$ and $\alpha_4 \coloneqq A\rho_2L$. We now present the formal version of Theorem \ref{thm:informalbiasedSA} as follows:
\begin{theorem}[Formal version of Theorem \ref{thm:informalbiasedSA}] \label{thm:formalbiasedSA}
    let $\alpha_2$,  $\alpha_3$ and  $\alpha_4$ be defined in \eqref{eq:modified_proposition3}, if $x^t$ is generated by \eqref{eq:SA-updategeneral} with all assumptions in \ref{setups} satisfied, then if the stepsize $\eta_t \coloneqq \frac{1}{\alpha_2(t+K)}$ while $K \coloneqq \max\{ \alpha_3/\alpha_2, 3\}$,
    \begin{equation} \label{eq:biasedSAformal}
        \mathbb{E}\Bigl[\|x^T - x^*\|^2_{\mathcal{N},\overline{E}}\Bigr] \leq \frac{K^2 c_u}{(T+K)^2c_l} \|x^0 - x^*\|^2_{\mathcal{N},\overline{E}} + \frac{8 \alpha_4 c_u}{(T+K)\alpha_2^2} + \frac{2c_u C_1 C_2 \varepsilon_{\text{bias}} }{\alpha_2} 
    \end{equation}
    where $C_1 = G(1+2x_{sup})$, $C_2 = \frac{1}{K} + \log \big(\frac{T-1+K}{K}\big)$, $G$ is defined in \eqref{eq:G_value} and $x_{sup} \coloneqq \sup \|x\|_{\mathcal{N},\overline{E}}$ is the upper bound of the $\|\cdot\|_{\cp}$ semi-norm for all $x^t$.
\end{theorem}
\begin{proof}
    This choice $\eta_t$ satisfies $\alpha_3 \eta_t^2 \leq \alpha_2\eta_t$. Thus, by \eqref{eq:modified_proposition3} we have
    \begin{equation}
        \mathbb{E}\bigl[M_{\overline{E}}(x^{t+1} - x^*) | \mathcal{F}_t\Bigr] \leq (1-\alpha_2 \eta_t )  M_{\overline{E}}(x^{t} - x^*) + \alpha_4 \eta_t^2 + \eta_t C_1 \varepsilon_{\text{bias}}
    \end{equation}
    we define $\Gamma_t \coloneqq \Pi_{i=o}^{t-1}(1-\alpha_2\eta_t)$ and further obtain the $T$-step recursion relationship as follows:
    \begin{align}
        \mathbb{E}&\Bigl[M_{\overline{E}}(x^{T} - x^*)\Bigr]\leq \Gamma_T  M_{\overline{E}}(x^{0} - x^*) + \Gamma_T \sum_{t=0}^{T-1} (\frac{1}{\Gamma_{t+1}})[\alpha_4 \eta_t^2 + \eta_t C_1 \varepsilon_{\text{bias}}] \nonumber\\
        & = \Gamma_T  M_{\overline{E}}(x^{0} - x^*) + \Gamma_T \sum_{t=0}^{T-1} (\frac{1}{\Gamma_{t+1}})[\alpha_4 \eta_t^2] + \Gamma_T \sum_{t=0}^{T-1} (\frac{1}{\Gamma_{t+1}})[\eta_t C_1 \varepsilon_{\text{bias}}]\nonumber\\
        & = \underbrace{\Gamma_T  M_{\overline{E}}(x^{0} - x^*) + \frac{\alpha_4\Gamma_T}{\alpha_2} \sum_{t=0}^{T-1} (\frac{1}{\Gamma_{t+1}})[\alpha_2 \eta_t^2]}_{R_1} + \underbrace{\Gamma_T \sum_{t=0}^{T-1} (\frac{1}{\Gamma_{t+1}})[\eta_t C_1 \varepsilon_{\text{bias}}] }_{R_2} \label{eq:R1R2bound}
    \end{align}
    where the term $R_1$ is identical to the unbiased case in Theorem 3 of \cite{zhang2021finite} which leads to
    \begin{equation} \label{eq:R1bound}
        R_1 \leq \frac{K^2}{(T+K)^2}M_{\overline{E}}(x^{0} - x^*) + \frac{4\alpha_2}{(T+K)\alpha_2^2}
    \end{equation}
    also, $R_2$ can be bounded by a logarithmic dependence of $T$
    \begin{equation} \label{eq:R2bound}
        R_2 \leq  \sum_{t=0}^{T-1} [\eta_t C_1 \varepsilon_{\text{bias}}]  = C_1 \varepsilon_{\text{bias}} \sum_{t=0}^{T-1}\frac{1}{\alpha_2(t+K)} \leq \frac{C_1 C_2 \varepsilon_{\text{bias}} }{\alpha_2} 
    \end{equation}
    Combining \eqref{eq:R1bound} and \eqref{eq:R2bound} with \eqref{eq:R1R2bound} would obtain the following:
    \begin{equation} \label{eq:biasedSAM}
    \mathbb{E}\Bigl[M_{\overline{E}}(x^{T} - x^*)\Bigr]\leq  \frac{K^2}{(T+K)^2}M_{\overline{E}}(x^{0} - x^*) + \frac{4\alpha_2}{(T+K)\alpha_2^2}+ \frac{C_1 C_2 \varepsilon_{\text{bias}} }{\alpha_2} 
    \end{equation}
    Combining \eqref{eq:biasedSAM} with \eqref{eq:M2span} yields \eqref{eq:biasedSAformal}.
\end{proof}

\section{Uncertainty Set Support Function Estimators}
\subsection{Proof of Theorem \ref{thm:sample-complexity}} \label{proof:sample-complexity}
We have
\begin{align}
\mathbb{E}[M]&=\sum_{n=0}^{N_{\max}-1} 2^{n+1}\,\mathbb{P}(N'=n)+2^{N_{\max}+1}\mathbb{P}(N' = N_{\max})
\nonumber\\
&=\sum_{n=0}^{N_{\max}-1} 2^{n+1}\,\mathbb{P}(N=n)+2^{N_{\max}+1}\mathbb{P}(N\ge N_{\max})
\nonumber\\
&=\sum_{n=0}^{N_{\max}-1}\Bigl(\frac{ 2^{n+1}}{ 2^{n+1}}\Bigr)+2^{N_{\max}+1}\mathbb{P}(N\ge N_{\max})
\nonumber\\
&=N_{\max}+2^{N_{\max}+1}\mathbb{P}(N\ge N_{\max})
\nonumber\\
&=N_{\max}+\frac{2^{N_{\max}+1}}{2^{N_{\max}}}
\nonumber\\
&=N_{\max}+2 = \cO(N_{\max}).
\end{align}
\subsection{Proof of Theorem \ref{thm:exp-bias}} \label{proof:exp-bias}
    denote $\hat{\sigma}^{*}_{\cp^a_s}(V)$ as the untruncated MLMC estimator obtained by running Algorithm \ref{alg:sampling} when setting $N_{\max}$ to infinity. From \cite{wang2023model}, under both TV uncertainty sets and Wasserstein uncertainty sets, we have $\hat{\sigma}^{*}_{\cp^a_s}(V)$ as an unbiased estimator of ${\sigma}_{\cp^a_s}(V)$. Thus,
    \begin{align}
        \mathbb{E}&\left[\hat{\sigma}_{\cp^a_s}(V) - {\sigma}_{\cp^a_s}(V)\right]   = \mathbb{E}\left[\hat{\sigma}_{\cp^a_s}(V)\right] - \mathbb{E}\left[\hat{\sigma}^{*}_{\cp^a_s}(V)\right] \nonumber\\
        &=\mathbb{E}\left[\sigma_{\hat{\kp}^{a,1}_{s,N'+1}}(V)+\frac{\Delta_{N'}(V)}{  \mathbb{P}(N' = n) }\right] - \mathbb{E}\left[\sigma_{\hat{\kp}^{a,1}_{s,N+1}}(V)+\frac{\Delta_{N}(V)}{  \mathbb{P}(N = n) }\right]  \nonumber\\
        &=\mathbb{E}\left[\frac{\Delta_{N'}(V)}{  \mathbb{P}(N' = n) }\right] - \mathbb{E}\left[\frac{\Delta_{N}(V)}{  \mathbb{P}(N = n) }\right]  \nonumber\\
        &=\sum_{n=0}^{N_{\max}}\Delta_{n}(V) - \sum_{n=0}^{\infty}\Delta_{n}(V)   \nonumber\\
        &= \sum_{n=N_{\max}+1}^{\infty}\Delta_{n}(V)
    \end{align}
    For each $\Delta_n(V)$, the expectation of absolute value can be bounded as
    \begin{align}
        &\mathbb{E}\left[\abs{\Delta_n(V)}\right] =\mathbb{E}\left[\abs{\sigma_{\hat{\kp}^{a}_{s,n+1}}(V)-\sigma_{{\kp}^{a}_{s}}(V)}\right]\nonumber\\
        &+\frac{1}{2}\mathbb{E}\left[\abs{\sigma_{\hat{\kp}^{a,E}_{s,n+1}}(V)-\sigma_{{\kp}^{a}_{s}}(V)}\right]+\frac{1}{2}\mathbb{E}\left[\abs{\sigma_{\hat{\kp}^{a,O}_{s,n+1}}(V)-\sigma_{{\kp}^{a}_{s}}(V)}\right]
    \end{align}
    By the binomial concentration and the Lipschitz property of the support function as in Lemma \ref{lem:LipschitzTV}, we know for TV distance uncertainty, we have
    \begin{equation}  \label{eq:DeltaabsboundTV}
    \mathbb{E}\left[\abs{\Delta_n(V)}\right] \leq 6(1+\frac{1}{\delta}) 2^{-\frac{n}{2}}\|V\|_{\mathrm{sp}}
    \end{equation}  
    and for Wasserstein disance uncertainty, we have
    \begin{equation}  \label{eq:DeltaabsboundWasserstein}
    \mathbb{E}\left[\abs{\Delta_n(V)}\right]  \leq 6\cdot 2^{-\frac{n}{2}}\|V\|_{\mathrm{sp}}
    \end{equation}     
    Thus, for TV distance uncertainty, we have
    \begin{equation} \label{eq:sigmabiasboundTV}
         \abs{\mathbb{E}\left[\hat{\sigma}_{\cp^a_s}(V) - {\sigma}_{\cp^a_s}(V)\right] } \leq \sum_{n=N_{\max}+1}^{\infty}\mathbb{E}\left[\abs{\Delta_n(V)}\right]  \leq 6(1+\frac{1}{\delta}) 2^{-\frac{N_{\max}}{2}}\|V\|_{\mathrm{sp}}
    \end{equation}
    and for Wasserstein distance uncertainty, we have
    \begin{equation} \label{eq:sigmabiasboundWasserstein}
         \abs{\mathbb{E}\left[\hat{\sigma}_{\cp^a_s}(V) - {\sigma}_{\cp^a_s}(V)\right] } \leq \sum_{n=N_{\max}+1}^{\infty}\mathbb{E}\left[\abs{\Delta_n(V)}\right]  \leq 6\cdot2^{-\frac{N_{\max}}{2}}\|V\|_{\mathrm{sp}}
    \end{equation}
    
\subsection{Proof of Lemma \ref{lem:LipschitzTV}} \label{proof:LipschitzTV}
    For TV uncertainty sets, for a fixed $V$, for any $p\in \Delta(\mathcal{S})$, define $f_p(\mathbf{\mu}) \coloneqq p(V-\mathbf{\mu}) - \delta\|V-\mathbf{\mu}\|_{\mathrm{sp}}$ and $\mathbf{\mu}_p^* \coloneqq \arg\max_{\mathbf{\mu} \geq \mathbf{0}}f_p(\mathbf{\mu})$. Thus, we have
    \begin{equation} \label{eq:fp-fq}
        \sigma_{\mathcal{P}_{TV}} (V) -  \sigma_{\mathcal{Q}_{TV}} (V) = f_p(\mu_p^*)- f_q(\mu_q^*) 
    \end{equation} 
    since, $\mathbf{\mu}_p^*$ and $\mathbf{\mu}_q^*$ are maximizers of $f_p$ and $f_q$ respectively, we further have
    \begin{equation} \label{eq:fp-fqbound}
     f_p(\mathbf{\mu}_q^*)- f_q(\mathbf{\mu}_q^*)  \leq f_p(\mathbf{\mu}_p^*)- f_q(\mathbf{\mu}_q^*) \leq f_p(\mathbf{\mu}_p^*)- f_q(\mathbf{\mu}_p^*)
    \end{equation}
    Combing \eqref{eq:fp-fq} and \eqref{eq:fp-fqbound} we thus have:
    \begin{align} \label{eq:maxpq}
        |\sigma_{\mathcal{P}_{TV}} (V) -  \sigma_{\mathcal{Q}_{TV}} (V)| &\leq \max\{|f_p(\mathbf{\mu}_p^*)- f_q(\mathbf{\mu}_p^*)|,|f_p(\mathbf{\mu}_q^*)- f_q(\mathbf{\mu}_q^*) | \} \nonumber \\
        &= \max\{|(p-q)(V-\mathbf{\mu}_p^*)|,|(p-q)(V-\mathbf{\mu}_q^*)| \} 
    \end{align}
    Note that $\sigma_{\mathcal{P}_{TV}} (V)$ can also be expressed as $\sigma_{\mathcal{P}_{TV}} (V) = p\mathbf{x}^*-\delta \|\mathbf{x}^*\|_{\mathrm{sp}}$ where $\mathbf{x}^* \coloneqq\arg\max_{\mathbf{x} \leq V}(p\mathbf{x}-\delta \|\mathbf{x}\|_{\mathrm{sp}})$. Let $M\coloneqq \max_s\mathbf{x}^*(s)$ and $m\coloneqq \min_s\mathbf{x}^*(s)$, then $\|\mathbf{x}\|_{\mathrm{sp}} = M-m$. Denote $\mathbf{e}$ as the all-ones vector, then $\mathbf{x}=\min_s V(s) \cdot \mathbf{e}$ is a feasible solution. Thus,
    \begin{equation}
        p\mathbf{x}^*-\delta (M-m) \geq p(\min_s V(s) \cdot \mathbf{e})-\delta \|\min_s V(s) \cdot \mathbf{e}\|_{\mathrm{sp}} = \min_s V(s)
    \end{equation}
    Since $p$ is a probability vector, $p\mathbf{x}^*\leq M$, using the fact that $\delta>0$, we then obtain
    \begin{equation} \label{eq:M-m}
        M-\delta (M-m) \geq \min_s V(s) \Rightarrow M-m \leq \frac{M-\min_s V(s)}{\delta}
    \end{equation}
    Since $\mathbf{x^*}$ is a feasible solution, we have
    \begin{equation} \label{eq:Vspan}
        M \leq \max_s V(s) \Rightarrow M-\min_sV(s) \leq \max_s V(s) -\min_sV(s) =\|V\|_{\mathrm{sp}}
    \end{equation}
    Combining \eqref{eq:M-m} and \eqref{eq:Vspan} we obtain
    \begin{equation}
        M-m\leq\frac{\|V\|_{\mathrm{sp}}}{\delta} \Rightarrow m \geq M - \frac{\|V\|_{\mathrm{sp}}}{\delta} \geq \min_sV(s) - \frac{\|V\|_{\mathrm{sp}}}{\delta}
    \end{equation}
    Where the last inequality is from $M \geq \min_sV(s)$, which is a direct result of \eqref{eq:M-m} and the term $\delta(M-m)$ being positive. We finally arrive with
    \begin{equation}
        \mathbf{x}^*(j) \in [m,M] \subseteq \Big[\min_sV(s) - \frac{\|V\|_{\mathrm{sp}}}{\delta} , \max_sV(s) \Big] \quad \forall j \in \mathcal{S}
    \end{equation}
    Thus, $\|\mathbf{x}^*\|_{\mathrm{sp}} \leq (1+\frac{1}{\delta})\|V\|_{\mathrm{sp}}$, which leads to 
    \begin{equation} \label{eq:V-mubound}
        \|V-\mathbf{\mu}_p^*\|_{\mathrm{sp}} \leq (1+\frac{1}{\delta})\|V\|_{\mathrm{sp}}, \quad \|V-\mathbf{\mu}_q^*\|_{\mathrm{sp}} \leq (1+\frac{1}{\delta})\|V\|_{\mathrm{sp}}
    \end{equation}
    Combining \eqref{eq:V-mubound} with \eqref{eq:maxpq} we obtain the first part of \eqref{eq:TVWlipschitz}.

    For Wasserstein uncertainty sets, note that for any $p \in \Delta(\mathcal{S})$ and value function $V$,
    \begin{equation}
        \sigma_{\mathcal{P}_{W}}(V) = \sup_{\lambda\geq 0}\left(\overbrace{-\lambda\delta^l+\mathbb{E}_{p}\big[\underbrace{\inf_{y\in\mathcal{S}}\big(V(y)+\lambda d(S,y)^l \big)}_{\phi(s,\lambda)}\big]}^{g(\lambda,p)} \right).    
    \end{equation}
    Note that
    \begin{equation}
        \inf_{y\in\mathcal{S}}V(y) \leq \phi(s,\lambda) \leq V(s)+\lambda d(S,s)^l = V(s)
    \end{equation}
    where the first inequality is because $\lambda d(S,y)^l\geq0$ for any $d$ and $l$. We can then bound $\phi$ by the span of $V$ as
    \begin{equation} 
        |\phi(s,\lambda)| \leq \|V\|_{\mathrm{sp}} \quad \forall \lambda \geq 0
    \end{equation}
    We then further have that for any $p,q\in \Delta(\mathcal{S})$ and $\lambda\geq0$,
    \begin{equation} \label{eq:p-qbound}
        |g(\lambda,p)-g(\lambda,q)| \leq \sum_{s\in\mathcal{S}} |p(s)-q(s)||\phi(s,\lambda)| \leq \|p-q\|_1 \|V\|_{\mathrm{sp}}
    \end{equation}
    using \eqref{eq:p-qbound} and the fact that $|f(\lambda)-g(\lambda)| \leq \epsilon \Rightarrow |\sup_\lambda f(\lambda) - \sup_\lambda g(\lambda)| \leq \epsilon$, we obtain the second part of \eqref{eq:TVWlipschitz}.

\subsection{Proof of Theorem \ref{thm:linear-variance}} \label{proof:linear-variance}
For all $p\in\Delta(\mathcal{S})$, we have $\sigma_p(V)\leq \|V\|_{\mathrm{sp}}$, leading to 
\begin{equation}
    \mathrm{Var}(\hat{\sigma}_{\cp^a_s}(V)) \leq \mathbb{E}\left[\left(\hat{\sigma}_{\cp^a_s}(V)\right)^2\right] + \|V\|^2_{\mathrm{sp}}
\end{equation}
To bound the second moment, note that 
\begin{align}
    \mathbb{E}\left[\left(\hat{\sigma}_{\cp^a_s}(V)\right)^2\right] &= \mathbb{E}\left[\left(\sigma_{\hat{\kp}^{a,1}_{s,N'+1}}(V)+\frac{\Delta_{N'}(V)}{  \mathbb{P}(N' = n) }\right)^2\right] \nonumber\\
    &\leq\mathbb{E}\left[\left(\|V\|_{\mathrm{sp}}+\frac{\Delta_{N'}(V)}{  \mathbb{P}(N' = n) }\right)^2\right] \nonumber\\
    &\leq 2\|V\|_{\mathrm{sp}}^2 + 2\mathbb{E}\left[\left(\frac{\Delta_{N'}(V)}{  \mathbb{P}(N' = n) }\right)^2\right]\nonumber\\
    &\leq 2\|V\|_{\mathrm{sp}}^2 + 2\sum_{n=0}^{N_{\max}}\left(\frac{\mathbb{E}[|\Delta_{n}(V)|]}{  \mathbb{P}(N' = n) }\right)^2 \mathbb{P}(N' = n)\nonumber\\
    &= 2\|V\|_{\mathrm{sp}}^2 + 2\sum_{n=0}^{N_{\max}}\frac{\mathbb{E}[|\Delta_{n}(V)|]^2}{  \mathbb{P}(N' = n) }
    \end{align}
Under TV distance uncertainty set, by \eqref{eq:DeltaabsboundTV}, we further have
    \begin{align}
    \mathbb{E}\left[\left(\hat{\sigma}_{\cp^a_s}(V)\right)^2\right] &\leq 2\|V\|_{\mathrm{sp}}^2 + 2S\sum_{n=0}^{N_{\max}}\frac{36(1+\frac{1}{\delta})^2 2^{-n}\|V\|_{\mathrm{sp}}^2}{2^{-(n+1)}}  \nonumber\\
    &= 2\|V\|_{\mathrm{sp}}^2 + 144(1+\frac{1}{\delta})^2S\|V\|_{\mathrm{sp}}^2 N_{\max}
\end{align}
Under Wasserstein distance uncertainty set, by \eqref{eq:DeltaabsboundWasserstein}, we further have
    \begin{align}
    \mathbb{E}\left[\left(\hat{\sigma}_{\cp^a_s}(V)\right)^2\right] &\leq 2\|V\|_{\mathrm{sp}}^2 + 2S\sum_{n=0}^{N_{\max}}\frac{36 2^{-n}\|V\|_{\mathrm{sp}}^2}{2^{-(n+1)}}  \nonumber\\
    &= 2\|V\|_{\mathrm{sp}}^2 + 144S\|V\|_{\mathrm{sp}}^2 N_{\max}
\end{align}

\section{Convergence for Robust TD} \label{proof:VGresults}
\subsection{Formal Statement of Theorem \ref{thm:Vresult}}
The first half of Algorithm \ref{alg:RobustTD} (line 1 - line 7) can be treated as a special instance of the SA updates in \eqref{eq:SA-updategeneral} with the bias and variance of the i.i.d. noise term specified in Section \ref{QueriesSection}. To facilitate deriving the bounds of the noise terms, we first analyze the bounds in terms of the $l_\infty$ norm, and then translate the bounds in terms of the $\|\cdot\|_\cp$ semi-norm to obtain the final results.

We start with analyzing the bias and variance of $\hat{\mathbf{T}}_{g_0}(V_t)$ for each $t$. Recall the definition of $\hat{\mathbf{T}}_{g_0}(V_t)$ is as follows:
$$
\hat{\mathbf{T}}_{g_0}(V_t)(s) = \sum_{a} \pi(a|s) \big[ r(s,a) - g_0 +  \hat{\sigma}_{\cp^a_s}(V_t) \big] \quad \forall s \in \mcs
$$
Thus, we have for all $s\in\mcs$,
\begin{equation}
    \abs{\E\left[\hat{\mathbf{T}}_{g_0}(V_t)(s)\right] - {\mathbf{T}}_{g_0}(V_t)(s)} \leq  \sum_{a} \pi(a|s) \abs{\E[\hat{\sigma}_{\cp^a_s}(V_t)] - {\sigma}_{\cp^a_s}(V_t) } = \abs{\E[\hat{\sigma}_{\cp^a_s}(V_t)] - {\sigma}_{\cp^a_s}(V_t) } 
\end{equation}
Which further implies the bias of $\hat{\mathbf{T}}_{g_0}(V_t)$ is bounded by the bias of $\hat{\sigma}_{\cp^a_s}(V_t)$ as follows:
\begin{equation}
\bignorm{\E\left[\hat{\mathbf{T}}_{g_0}(V_t)\right] - {\mathbf{T}}_{g_0}(V_t)}_\infty \leq \abs{\E[\hat{\sigma}_{\cp^a_s}(V_t)] - {\sigma}_{\cp^a_s}(V_t)}
\end{equation}
Regarding the variance, note that
\begin{align}
    \E&\left[(\hat{\mathbf{T}}_{g_0}(V_t)(s)-{\mathbf{T}}_{g_0}(V_t)(s))^2\right] \nonumber \\
    &=  \left(\E\left[\hat{\mathbf{T}}_{g_0}(V_t)(s)\right] - {\mathbf{T}}_{g_0}(V_t)(s)\right)^2+ \mathrm{Var}\left(  \hat{\mathbf{T}}_{g_0}(V_t)(s)  \right) \nonumber\\
    &\leq\abs{\E[\hat{\sigma}_{\cp^a_s}(V_t)] - {\sigma}_{\cp^a_s}(V_t)}^2+ \mathrm{Var}\left( \sum_{a} \pi(a|s)  \hat{\sigma}_{\cp^a_s}(V_t)  \right) \nonumber\\
    &=\abs{\E[\hat{\sigma}_{\cp^a_s}(V_t)] - {\sigma}_{\cp^a_s}(V_t)}^2+ \sum_{a} \pi(a|s)^2 \mathrm{Var}\left(  \hat{\sigma}_{\cp^a_s}(V_t)  \right) 
\end{align}
To create an upper bound of $\|V\|_{\mathrm{sp}}$ for all possible $V$, define the mixing time of any $p\in\mathcal{P}$ to be
\begin{equation}
    t^p_{\mathrm{mix}}\coloneqq\arg\min_{t \geq 1} \left\{ \max_{\mu_0} \left\| (\mu_0 p_{\pi}^{t})^{\top} - \nu^{\top} \right\|_1 \leq \frac{1}{2} \right\}
\end{equation}
where $p_\pi$ is the finite state Markov chain induced by $\pi$, $\mu_0$ is any initial probability distribution on $\mcs$ and $\nu$ is its invariant distribution. By Assumption \ref{ass:sameg}, and Lemma \ref{lem:wanglemma9},and for any value function $V$, we have
    \begin{equation} \label{eq:boundofVsp}
         t^p_{\mathrm{mix}} < +\infty \quad \text{and} \quad\|V\|_{\mathrm{sp}} \leq 4t^p_{\mathrm{mix}} \leq 4 t_{\mathrm{mix}}
    \end{equation}
where we define $t_{\mathrm{mix}} \coloneqq \sup_{p\in\mathcal{P}} t^p_{\mathrm{mix}}$, then $t_{\mathrm{mix}}$ is also finite due to the compactness of $\mathcal{P}$. We now derive the bounds of biases and variances for the three types of uncertainty sets. Regarding contamination uncertainty sets, according to Lemma \ref{lem:wangthmD1}, $\hat{\sigma}_{\cp^a_s}(V)$ is unbiased and has variance bounded by $\|V\|^2$. Thu, define $t_{\mathrm{mix}}$ according to Lemma \ref{lem:wanglemma9} and combining the above result with Lemma \ref{lem:wanglemma9}, we obtain that $\hat{\mathbf{T}}_{g_0}(V_t)$ is also unbiased and the variance satisfies
\begin{equation} \label{eq:firstsupnormbias}
    \E\left[\bignorm{\hat{\mathbf{T}}_{g_0}(V_t)-{\mathbf{T}}_{g_0}(V_t)}^2_\infty\right] \leq \|V_t\|^2 \leq 16t^2_{\mathrm{mix}}
\end{equation}
Regarding TV distance uncertainty sets, using the property of the bias and variance of $\hat{\sigma}_{\cp^a_s}(V)$ in Theorem \ref{thm:exp-bias} and Theorem \ref{thm:linear-variance} while combining them with Lemma \ref{lem:wanglemma9}, we have
\begin{equation}
        \bignorm{\E\left[\hat{\mathbf{T}}_{g_0}(V_t)\right]-{\mathbf{T}}_{g_0}(V_t)}_\infty \leq 6(1+\frac{1}{\delta}) \sqrt{S}2^{-\frac{N_{\max}}{2}}\|V\|_{\mathrm{sp}}= 24(1+\frac{1}{\delta})\sqrt{S} 2^{-\frac{N_{\max}}{2}}t_{\mathrm{mix}}
\end{equation}
and
\begin{align}
        \E&\left[\bignorm{\hat{\mathbf{T}}_{g_0}(V_t)-{\mathbf{T}}_{g_0}(V_t)}^2_\infty\right] \nonumber \\
        &\leq \left( 24(1+\frac{1}{\delta})\sqrt{S} 2^{-\frac{N_{\max}}{2}}t_{\mathrm{mix}} \right)^2 + 3\|V\|_{\mathrm{sp}}^2 + 144(1+\frac{1}{\delta})^2S\|V\|_{\mathrm{sp}}^2 N_{\max} \nonumber\\
        &\leq \left( 24(1+\frac{1}{\delta}) \sqrt{S} 2^{-\frac{N_{\max}}{2}}t_{\mathrm{mix}} \right)^2 + 48t_{\mathrm{mix}}^2 + 2304(1+\frac{1}{\delta})^2St_{\mathrm{mix}}^2 N_{\max}
\end{align}
Similarly, for Wasserstein distance uncertainty sets, we have
\begin{equation}
        \bignorm{\E\left[\hat{\mathbf{T}}_{g_0}(V_t)\right]-{\mathbf{T}}_{g_0}(V_t)}_\infty \leq 6\sqrt{S} 2^{-\frac{N_{\max}}{2}}\|V\|_{\mathrm{sp}}= 24\sqrt{S} 2^{-\frac{N_{\max}}{2}}t_{\mathrm{mix}}
\end{equation}
and
\begin{align} \label{eq:lastsupnormbias}
        \E\left[\bignorm{\hat{\mathbf{T}}_{g_0}(V_t)-{\mathbf{T}}_{g_0}(V_t)}^2_\infty\right] &\leq \left( 24\sqrt{S} 2^{-\frac{N_{\max}}{2}}t_{\mathrm{mix}} \right)^2 + 3\|V\|_{\mathrm{sp}}^2 + 144\|V\|_{\mathrm{sp}}^2 N_{\max} \nonumber\\
        &\leq \left( 24\sqrt{S} 2^{-\frac{N_{\max}}{2}}t_{\mathrm{mix}} \right)^2 + 48t_{\mathrm{mix}}^2 + 2304St_{\mathrm{mix}}^2 N_{\max}
\end{align}
In order to translate the above bounds from the $l_\infty$ norm into the $\|\cdot\|_\cp$ norm, recall that in line 7 of Algorithm \ref{alg:RobustTD}, we chose an anchor state $s_0$ set $V_t(s_0)=0$ for all $t$ to avoid ambiguity. We thus can draw the following relationship:

\begin{lemma} \label{lem:normtranslations}
Let \(x\in \mathbb{R}^S\) satisfy \(x_i=0\) for some fixed index \(i\).  Then
\[
\|x\|_\infty\leq \|x\|_{\mathrm{sp}}
=\max_{1\le j\le n}x_j -\min_{1\le j\le n}x_j
\leq 2\|x\|_\infty.
\]
Moreover, since all semi-norms with the same kernel spaces are equivalent, there are constants \(c_{\cp},C_{\cp}>0\) so that
\[
c_\cp\|x\|_{\mathrm{sp}}\leq \|x\|_{\cp}\le\;C_\cp\|x\|_{\mathrm{sp}}
\quad\forall\,x\in\mathbb{R}^n,
\]
then
\begin{equation} 
c_\cp\|x\|_{\infty}\leq c_\cp\|x\|_{\mathrm{sp}}\leq\|x\|_{\cp}\leq C_\cp\|x\|_{\mathrm{sp}}\leq 2C_\cp \|x\|_\infty.
\end{equation}

\begin{proof}
Since \(x_i=0\), for every \(j\) we have \(-\|x\|_\infty \le x_j\le\|x\|_\infty\).  Hence
\[
\max_j x_j\;\le\;\|x\|_\infty,
\quad
\min_j x_j\;\ge\;-\,\|x\|_\infty,
\]
and so
\[
\|x\|_\infty = \max \{\max_j x_j - 0,0-\min_j x_j\} \leq \|x\|_{\mathrm{sp}}
=\max_j x_j - \min_j x_j
\le \|x\|_\infty -(-\|x\|_\infty)
=2\,\|x\|_\infty.
\]
Since $\|\cdot\|_{\mathrm{sp}}$ and \(\|\cdot\|_{\cp}\) both have the same kernel of $\{c\e : c\in\mathbb{R}\}$, by the equivalence of semi-norms, it follows that there exists $c_\cp$ and $C_\cp$ such that
\begin{equation}
c_\cp\|x\|_{\infty}\leq c_\cp\|x\|_{\mathrm{sp}}\leq\|x\|_{\cp}\leq C_\cp\|x\|_{\mathrm{sp}}\leq 2C_\cp \|x\|_\infty \nonumber
\end{equation}
as claimed.
\end{proof}
\end{lemma}

With the relationship established above, line 1 - line 7 of Algorithm \ref{alg:RobustTD} can be formally treated as a special instance of the SA updates in \eqref{eq:SA-updategeneral} with $B=0$. We now provide the bias and variance of the i.i.d. noise for the different uncertainty sets discussed using Lemma \ref{lem:normtranslations} and the estimation bounds in \eqref{eq:firstsupnormbias}-\eqref{eq:lastsupnormbias}. For contamination uncertainty sets, we have
\begin{equation}  \label{eq:firsteq}
    \varepsilon^{\text{Cont}}_{\text{bias}}=0 \quad \text{and} \quad A^{\text{Cont}}=32C_\cp^2t^2_{\mathrm{mix}}
\end{equation}
for TV distance uncertainty sets, we have
\begin{equation}
    \varepsilon^{\text{TV}}_{\text{bias}}=48C_\cp(1+\frac{1}{\delta}) \sqrt{S}2^{-\frac{N_{\max}}{2}}t_{\mathrm{mix}} = \cO\left(\sqrt{S} 2^{-\frac{N_{\max}}{2}}t_{\mathrm{mix}}   \right)
\end{equation}
and
\begin{align}
    A^{\text{TV}}&=2C_\cp^2\left( 24(1+\frac{1}{\delta}) \sqrt{S}2^{-\frac{N_{\max}}{2}}t_{\mathrm{mix}} \right)^2 + 96C_\cp^2t_{\mathrm{mix}}^2 + 4608C_\cp^2(1+\frac{1}{\delta})^2 S t_{\mathrm{mix}}^2 N_{\max} \nonumber\\
    &= \cO\left(  t_{\mathrm{mix}}^2 SN_{\max}  \right)
\end{align}
and for Wasserstein distance uncertainty sets, we have
\begin{equation}
    \varepsilon^{\text{Wass}}_{\text{bias}}=48C_\cp\sqrt{S} 2^{-\frac{N_{\max}}{2}}t_{\mathrm{mix}} = \cO\left( \sqrt{S}2^{-\frac{N_{\max}}{2}}t_{\mathrm{mix}}   \right)
\end{equation}
and
\begin{align}  \label{eq:lasteq}
    A^{\text{Wass}}&=2C_\cp^2\left( 24\sqrt{S} 2^{-\frac{N_{\max}}{2}}t_{\mathrm{mix}} \right)^2 + 96C_\cp^2t_{\mathrm{mix}}^2 + 4608 C_\cp^2St_{\mathrm{mix}}^2 N_{\max} \nonumber\\
    &= \cO\left(  t_{\mathrm{mix}}^2 S N_{\max}  \right)
\end{align}

\begin{theorem}[Formal version of Theorem \ref{thm:Vresult}]\label{thm:formalVresult}
    Let $\alpha_2 \coloneqq (1-\gamma \sqrt{c_{{u}} / c_{{l}}})$, $\alpha_3 \coloneqq 8c_{{u}} \rho_2 L$ and $\alpha_4 \coloneqq \rho_2L$, if $V_t$ is generated by Algorithm \ref{alg:RobustTD}. Define $V^*$ to be the anchored robust value function $V^* =  V^\pi_{\kp_V} + c \mathbf{e}$ for some $c$ such that $V^*(s_0)=0$, then under Assumption \ref{ass:sameg} and the radius conditions of Lemma \ref{lem:contaminationradius}-\ref{lem：Wradius}, if the stepsize $\eta_t \coloneqq \frac{1}{\alpha_2(t+K)}$ while $K \coloneqq \max\{ \alpha_3/\alpha_2, 3\}$, then for contamination uncertainty sets,
    \begin{equation} 
        \mathbb{E}\Bigl[\|V_T - V^*\|^2_{\infty}\Bigr] \leq \frac{4 K^2 c_u C_\cp^2}{(T+K)^2c_l c_\cp^2} \|V_0 - V^*\|^2_{\infty} + \frac{8 A^{\mathrm{Cont}}\alpha_4 c_u}{(T+K)\alpha_2^2 c_\cp^2}=\cO\left(\frac{1}{T^2}+\frac{t^2_{\mathrm{mix}}}{T(1-\gamma)^2}\right)
    \end{equation}
    for TV distance uncertainty sets,
    \begin{align} 
        \mathbb{E}\Bigl[\|V_T - V^*\|^2_{\infty}\Bigr] &\leq \frac{4 K^2 c_u C_\cp^2}{(T+K)^2c_l c_\cp^2} \|V_0 - V^*\|^2_{\infty} + \frac{8A^{\mathrm{TV}} \alpha_4 c_u}{(T+K)\alpha_2^2 c_\cp^2} + \frac{c_u C_3 C_2 \varepsilon^{\mathrm{TV}}_{\mathrm{bias}} }{\alpha_2 c_\cp^2} \\
        &=\cO\left(\frac{1}{T^2}+\frac{t^2_{\mathrm{mix}}N_{\max}}{T(1-\gamma)^2} + \frac{t^2_{\mathrm{mix}}  2^{-\frac{N_{\max}}{2}}\log T}{(1-\gamma)^2}\right)
    \end{align}
    for Wasserstein distance uncertainty sets,
    \begin{align} 
        \mathbb{E}\Bigl[\|V_T - V^*\|^2_{\infty}\Bigr] &\leq \frac{4 K^2 c_u C_\cp^2}{(T+K)^2c_l c_\cp^2} \|V_0 - V^*\|^2_{\infty} + \frac{8 A^{\mathrm{Wass}} \alpha_4 c_u}{(T+K)\alpha_2^2c_\cp^2} + \frac{c_u C_3 C_2 \varepsilon^{\mathrm{Wass}}_{\mathrm{bias}} }{\alpha_2 c_\cp^2} \\
        &=\cO\left(\frac{1}{T^2}+\frac{t^2_{\mathrm{mix}}N_{\max}}{T(1-\gamma)^2}+\frac{t^2_{\mathrm{mix}}  2^{-\frac{N_{\max}}{2}}\log T}{(1-\gamma)^2}\right)
    \end{align}
    where the $\varepsilon$ and $A$ terms are defined in \eqref{eq:firsteq}-\eqref{eq:lasteq}, $C_2 = \frac{1}{K} + \log \big(\frac{T-1+K}{K}\big)$, $C_3 = G(1+8C_\cp t_{\mathrm{mix}})$, $\gamma$ is defined in \eqref{eq:contractiongamma}, $c_u,c_l$ are defined in \eqref{eq:M2span}, $\rho_2$ is defined in \eqref{eq:normequivalence}, $G$ is defined in \eqref{eq:G_value}, and $C_\cp, c_\cp$ are defined in Lemma \ref{lem:normtranslations}.
\end{theorem}
\begin{proof}
    By Lemma \ref{lem:normtranslations} and \eqref{eq:boundofVsp}, we have that for any value function $V$ its $\|\cdot\|_\cp$ norm is bounded as follows:
    \begin{equation} \label{eq:boundPnormofV}
        \|V\|_\cp \leq 4C_\cp t_{\rm mix}
    \end{equation}
    Substituting the terms of \eqref{eq:firsteq}-\eqref{eq:lasteq}, \eqref{eq:boundPnormofV}, and Theorem \ref{thm:robust_seminorm-contraction} to Theorem \ref{thm:formalbiasedSA}, we would have for contamination uncertainty sets,
    \begin{equation} 
        \mathbb{E}\Bigl[\|V_T - V^*\|^2_{\cp}\Bigr] \leq \frac{K^2 c_u}{(T+K)^2c_l} \|V_0 - V^*\|^2_{\cp} + \frac{8 A^{\mathrm{Cont}}\alpha_4 c_u}{(T+K)\alpha_2^2}=\cO\left(\frac{1}{T^2}+\frac{t^2_{\mathrm{mix}}}{T(1-\gamma)^2}\right)
    \end{equation}
    for TV distance uncertainty sets,
    \begin{align} 
        \mathbb{E}\Bigl[\|V_T - V^*\|^2_{\cp}\Bigr] &\leq \frac{K^2 c_u}{(T+K)^2c_l} \|V_0 - V^*\|^2_{\cp} + \frac{8A^{\mathrm{TV}} \alpha_4 c_u}{(T+K)\alpha_2^2} + \frac{c_u C_3 C_2 \varepsilon^{\mathrm{TV}}_{\mathrm{bias}} }{\alpha_2} \\
        &=\cO\left(\frac{1}{T^2}+\frac{St^2_{\mathrm{mix}}N_{\max}}{T(1-\gamma)^2} + \frac{St^2_{\mathrm{mix}}  2^{-\frac{N_{\max}}{2}}\log T}{(1-\gamma)^2}\right)
    \end{align}
    for Wasserstein distance uncertainty sets,
    \begin{align} 
        \mathbb{E}\Bigl[\|V_T - V^*\|^2_{\cp}\Bigr] &\leq \frac{K^2 c_u}{(T+K)^2c_l} \|V_0 - V^*\|^2_{\cp} + \frac{4 A^{\mathrm{Wass}} \alpha_4 c_u}{(T+K)\alpha_2^2} + \frac{c_u C_3 C_2 \varepsilon^{\mathrm{Wass}}_{\mathrm{bias}} }{\alpha_2} \\
        &=\cO\left(\frac{1}{T^2}+\frac{St^2_{\mathrm{mix}}N_{\max}}{T(1-\gamma)^2}+\frac{St^2_{\mathrm{mix}}  2^{-\frac{N_{\max}}{2}}\log T}{(1-\gamma)^2}\right)
    \end{align}
    where the $\varepsilon$ and $A$ terms are defined in \eqref{eq:firsteq}-\eqref{eq:lasteq},  $C_2 = \frac{1}{K} + \log \big(\frac{T-1+K}{K}\big)$, $C_3 = G(1+8C_\cp t_{\mathrm{mix}})$, $\gamma$ is defined in \eqref{eq:contractiongamma}, $c_u,c_l$ are defined in \eqref{eq:M2span}, $\rho_2$ is defined in \eqref{eq:normequivalence}, $G$ is defined in \eqref{eq:G_value}, and $C_\cp$ is defined in Lemma \ref{lem:normtranslations}. We now translate the result back to the standard $l_\infty$ norm by applying Lemma \ref{lem:normtranslations} again to the above, we obtain the desired results.
\end{proof}

\subsection{Proof of Theorem \ref{thm:Vresult}}
We use the result from Theorem \ref{thm:formalVresult}, to set $\mathbb{E}\Bigl[\|V_T - V^*\|^2_{\infty}\Bigr] \leq \epsilon^2$. For contamination uncertainty set we set 
 $T=\cO\left(\frac{t^2_{\mathrm{mix}}}{\epsilon^2(1-\gamma)^2} \right)$, resulting in $\cO\left(\frac{SAt^2_{\mathrm{mix}}}{\epsilon^2(1-\gamma)^2} \right)$ sample complexity. For TV and Wasserstein uncertainty set, we set 
$N_{\max}=\cO\left(\log \frac{\sqrt{S}t_{\mathrm{mix}}}{\epsilon(1-\gamma)}\right)$ and $T=\cO\left(\frac{t^2_{\mathrm{mix}}}{\epsilon^2(1-\gamma)^2} \log\frac{\sqrt{S}t_{\mathrm{mix}}}{\epsilon(1-\gamma)}\right)$, combining with Theorem \ref{thm:sample-complexity}, this would result in $\cO\left(\frac{SAt^2_{\mathrm{mix}}}{\epsilon^2(1-\gamma)^2} \log^2\frac{\sqrt{S}t_{\mathrm{mix}}}{\epsilon(1-\gamma)}\right)$ sample complexity.

To show order-optimality, we provide the standard mean estimation as a hard example. Consider the TD learning of the MDP with only two states $\mcs=\{s_1,s_2\}$,  and \(Pr(s\to s_1)=p\), \(Pr(s\to s_2)=1-p\) for each \(s\in S\) with $p \in (0,1)$. Thus, this MDP is indifferent from the actions chosen and we further define  \(r(s_1)=1,\;r(s_2)=0\). Thus, estimating the relative value functions is equivalent of estimating $p$. By the Cramér–Rao or direct variance argument for Bernoulli(\(p\)) estimation, we have that to achieve \(|\hat{p_N}-p|^2\le\epsilon^2\) requires 
\(\displaystyle N\ge\ 1/ {2\epsilon^2}=\Omega(\epsilon^{-2})\).

\subsection{Formal Statement of Theorem \ref{thm:gresult}}
To analyze the second part (line 8 - line 14) of Algorithm \ref{alg:RobustTD} and provide the provide the complexity for $g_t$, we first define the noiseless function $\bar{\delta}(V)$ as
\begin{equation} \label{eq:bardeltat}
    \bar{\delta}(V) \coloneqq  \frac{1}{S}\sum_s   \left(\sum_{a}\pi(a|s) \big[ r(s,a) +  {\sigma}_{\cp^a_s}(V) \big]- V(s) \right)
\end{equation}
Thus, we have
\begin{equation}
    \bar{\delta}_t = \bar{\delta}(V_T) + \nu_t
\end{equation}
where $\nu_t$ is the noise term with bias equal to the bias $\hat{\sigma}_{\cp^a_s}(V_T)$
\begin{equation}
    \E[\abs{\nu_t}] =  \frac{1}{S}\sum_s  \sum_{a} \left(\pi(a|s) \E\big[\abs{  {\sigma}_{\cp^a_s}(V_T) - \hat{\sigma}_{\cp^a_s}(V_T)} \big] \right) =  \abs{\mathbb{E}\left[\hat{\sigma}_{\cp^a_s}(V_T) - {\sigma}_{\cp^a_s}(V_T)\right] }
\end{equation}
By the Bellman equation in Theorem \ref{thm:robust Bellman}, we have $g^\pi_\cp = \bar{\delta}(V^*)$, which implies
\begin{align}
    \abs{\bar{\delta}(V_T) - g^\pi_\cp} &= \abs{\bar{\delta}(V_T) - \bar{\delta}(V^*)} \nonumber\\
    &\leq \frac{1}{S}\sum_s \left( \sum_{a} \pi(a|s)   \abs{{\sigma}_{\cp^a_s}(V_T) - {\sigma}_{\cp^a_s}(V^*)}+\abs{V_T(s)-V^*(s)} \right)\nonumber\\
    &\leq \frac{1}{S}\sum_s  \left( \sum_{a} \pi(a|s)   \abs{{\sigma}_{\cp^a_s}(V_T) - {\sigma}_{\cp^a_s}(V^*)}+\abs{V_T(s)-V^*(s)} \right)\nonumber\\
    &\leq \frac{1}{S}\sum_s 2\|V_T-V^*\|_{\mathrm{sp}}\nonumber\\
    &=  2\|V_T-V^*\|_{\mathrm{sp}} \nonumber\\
    &\leq  4\|V_T-V^*\|_{\infty}
\end{align}

Where the last inequality is by Lemma \ref{lem:normtranslations}. Thus, the following recursion can be formed
\begin{align}
    \abs{g_{t+1}-g^\pi_\cp}&=\abs{g_{t}+\beta_t(\bar{\delta}_t-g_t)-g^\pi_\cp }\nonumber\\
    &=\abs{g_{t}-g^\pi_\cp  +\beta_t(\bar{\delta}_t-g^\pi_\cp+g^\pi_\cp -g_t)}\nonumber\\
    &=\abs{g_{t}-g^\pi_\cp  +\beta_t(\bar{\delta}(V^T)-g^\pi_\cp + \nu_t+g^\pi_\cp -g_t)}\nonumber\\
    &\leq(1-\beta_t)\abs{g_{t}-g^\pi_\cp}  +\beta_t(\abs{\bar{\delta}(V^T)-g^\pi_\cp} + \abs{\nu_t})\nonumber\\
    &\leq(1-\beta_t)\abs{g_{t}-g^\pi_\cp}  +\beta_t(4\|V_T-V^*\|_{\infty} + \abs{\nu_t})
\end{align}
Thus, taking expectation conditioned on the filtration $\mathcal{F}^t$ yields
\begin{align}
    \E&\left[\abs{g_{t+1}-g^\pi_\cp}\right]\leq(1-\beta_t)\abs{g_{t}-g^\pi_\cp}  +\beta_t\left(4\E\left[\norm{V_T-V^*}_{\infty}\right] + \E[\abs{\nu_t}]\right)\nonumber\\
    &\leq(1-\beta_t)\abs{g_{t}-g^\pi_\cp}  +\beta_t\left(4\E\left[\norm{V_T-V^*}_{\infty}\right] + \abs{\mathbb{E}\left[\hat{\sigma}_{\cp^a_s}(V_T) - {\sigma}_{\cp^a_s}(V_T)\right] }\right)
\end{align}
By letting $\zeta_t \coloneqq \Pi_{i=0}^{t-1}(1-\beta_t)$, we obtain the $T$-step recursion as follows:
\begin{align}
    \E&\left[\abs{g_{T}-g^\pi_\cp}\right]\leq\zeta_T\abs{g_0-g^\pi_\cp}\nonumber\\
    &+ \zeta_T \sum_{t=0}^{T-1} (\frac{1}{\zeta_{t+1}})\beta_t\left(4\E\left[\norm{V_T-V^*}_{\infty}\right]  + \abs{\mathbb{E}\left[\hat{\sigma}_{\cp^a_s}(V_T) - {\sigma}_{\cp^a_s}(V_T)\right] }\right) \nonumber\\
        & = \zeta_T\abs{g_0-g^\pi_\cp} +  \sum_{t=0}^{T-1} (\frac{\zeta_T}{\zeta_{t+1}})\beta_t\left(4\E\left[\norm{V_T-V^*}_{\infty}\right]  + \abs{\mathbb{E}\left[\hat{\sigma}_{\cp^a_s}(V_T) - {\sigma}_{\cp^a_s}(V_T)\right] }\right) \nonumber\\
          & \leq \zeta_T\abs{g_0-g^\pi_\cp} +  \sum_{t=0}^{T-1} \beta_t{\left(4\E\left[\norm{V_T-V^*}_{\infty}\right]  + \abs{\mathbb{E}\left[\hat{\sigma}_{\cp^a_s}(V_T) - {\sigma}_{\cp^a_s}(V_T)\right] }\right) }\nonumber\\
          & = \zeta_T\abs{g_0-g^\pi_\cp} + \left(4\E\left[\norm{V_T-V^*}_{\infty}\right]  + \abs{\mathbb{E}\left[\hat{\sigma}_{\cp^a_s}(V_T) - {\sigma}_{\cp^a_s}(V_T)\right] }\right)  \sum_{t=0}^{T-1} \beta_t \label{eq:gTbound}
\end{align}
By setting $\beta_t\coloneqq\frac{1}{t+1}$, we have $\zeta_T= \frac{1}{T+1}\leq\frac{1}{T}$ and $ \sum_{t=0}^{T-1} \beta_t\leq 2\log T$, \eqref{eq:gTbound} implies
\begin{equation} \label{eq:gresult}
    \E\left[\abs{g_{T}-g^\pi_\cp}\right]\leq \frac{1}{T}\abs{g_0-g^\pi_\cp} + \left(8\E\left[\norm{V_T-V^*}_{\infty}\right] + 2\abs{\mathbb{E}\left[\hat{\sigma}_{\cp^a_s}(V_T) - {\sigma}_{\cp^a_s}(V_T)\right] }\right)  \log T
\end{equation}

\begin{theorem}[Formal version of Theorem \ref{thm:gresult}] \label{thm:formalgresult}
    Following all notations and assumptions in Theorem \ref{thm:formalVresult}, then for contamination uncertainty sets,
    \begin{align} 
        \mathbb{E}&\Bigl[\abs{g_{T}-g^\pi_\cp}\Bigr] \leq\frac{1}{T}\abs{g_0-g^\pi_\cp} + \frac{16K C_\cp\sqrt{c_u}\log T}{(T+K)c_\cp\sqrt{c_l}} \|V_0 - V^*\|_{\infty} \nonumber\\
        &+ \frac{16 \sqrt{2A^{\mathrm{Cont}}\alpha_4 c_u}\log T}{\alpha_2c_\cp\sqrt{T+K}} =\cO\left(\frac{1}{T}+\frac{\log T}{T}+\frac{t_{\mathrm{mix}}\log T}{\sqrt{T}(1-\gamma)}\right).
    \end{align}
    For TV distance uncertainty sets,

    \begin{align} 
        \mathbb{E}&\Bigl[\abs{g_{T}-g^\pi_\cp}\Bigr] \leq\frac{1}{T}\abs{g_0-g^\pi_\cp} + \frac{16K C_\cp \sqrt{c_u}\log T}{(T+K)c_\cp\sqrt{c_l}} \|V_0 - V^*\|_{\infty} \nonumber\\
        &+ \frac{16\sqrt{2A^{\mathrm{TV}} \alpha_4 c_u}\log T}{\alpha_2c_\cp\sqrt{T+K}} + \frac{8\sqrt{c_u C_3 C_2 \varepsilon^{\mathrm{TV}}_{\mathrm{bias}}}\log T }{c_\cp\sqrt{\alpha_2}} + 48(1+\frac{1}{\delta}) \sqrt{S}2^{-\frac{N_{\max}}{2}}t_{\mathrm{mix}} \log T \nonumber\\
        &=\cO\left(\frac{1}{T}+\frac{\log T}{T}+\frac{t_{\mathrm{mix}}\sqrt{SN_{\max}}\log T}{\sqrt{T}(1-\gamma)} + \frac{\sqrt{S}t_{\mathrm{mix}}  2^{-\frac{N_{\max}}{4}}\log^{\frac{3}{2}} T}{\sqrt{1-\gamma}} +  \sqrt{S}2^{-\frac{N_{\max}}{2}}t_{\mathrm{mix}} \log T\right).
    \end{align}

    For Wasserstein distance uncertainty sets,
    \begin{align} 
        \mathbb{E}&\Bigl[\abs{g_{T}-g^\pi_\cp}\Bigr] \leq \frac{1}{T}\abs{g_0-g^\pi_\cp} +\frac{16K C_\cp\sqrt{c_u} \log T}{(T+K)c_\cp\sqrt{c_l}} \|V_0 - V^*\|_{\infty}  \nonumber\\
        &+ \frac{16 \sqrt{2A^{\mathrm{Wass}} \alpha_4 c_u} \log T}{\alpha_2c_\cp\sqrt{T+K}} + \frac{8\sqrt{c_u C_3 C_2 \varepsilon^{\mathrm{Wass}}_{\mathrm{bias}}} \log T}{c_\cp\sqrt{\alpha_2}} + 48\sqrt{S}2^{-\frac{N_{\max}}{2}}t_{\mathrm{mix}} \log T \nonumber\\
        &=\cO\left(\frac{1}{T}+\frac{\log T}{T}+\frac{t_{\mathrm{mix}}\sqrt{SN_{\max}}\log T}{\sqrt{T}(1-\gamma)} + \frac{\sqrt{S}t_{\mathrm{mix}}  2^{-\frac{N_{\max}}{4}}\log^{\frac{3}{2}} T}{\sqrt{1-\gamma}} +  \sqrt{S}2^{-\frac{N_{\max}}{2}}t_{\mathrm{mix}} \log T\right).
    \end{align}
where all the above variables are defined the same as in Theorem \ref{thm:formalVresult}.

\begin{proof}
    By Theorem \ref{thm:formalVresult}, taking square root on both side and utilizing the concavity of square root function, we have for contamination uncertainty sets,
    \begin{equation} \label{eq:firstequation}
        \mathbb{E}\Bigl[\|V_T - V^*\|_{\infty}\Bigr] \leq \frac{2KC_\cp \sqrt{c_u}}{(T+K)c_\cp\sqrt{c_l}} \|V_0 - V^*\|_{\infty} + \frac{2\sqrt{ 2A^{\mathrm{Cont}}\alpha_4 c_u}}{\alpha_2c_\cp\sqrt{T+K}}
    \end{equation}
    for TV distance uncertainty sets,
    \begin{equation} 
        \mathbb{E}\Bigl[\|V_T - V^*\|_{\infty}\Bigr] \leq \frac{2KC_\cp \sqrt{c_u}}{(T+K)c_\cp\sqrt{c_l}} \|V_0 - V^*\|_{\infty} + \frac{2\sqrt{2A^{\mathrm{TV}} \alpha_4 c_u}}{\alpha_2c_\cp\sqrt{T+K}} + \sqrt{\frac{c_u C_3 C_2 \varepsilon^{\mathrm{TV}}_{\mathrm{bias}} }{\alpha_2 c_\cp^2} }
    \end{equation}
    for Wasserstein distance uncertainty sets,
    \begin{equation} 
        \mathbb{E}\Bigl[\|V_T - V^*\|_{\infty}\Bigr] \leq \frac{2KC_\cp \sqrt{c_u}}{(T+K)c_\cp\sqrt{c_l}} \|V_0 - V^*\|_{\infty} + \frac{2 \sqrt{2A^{\mathrm{Wass}} \alpha_4 c_u}}{\alpha_2 c_\cp\sqrt{T+K}} + \sqrt{\frac{c_u C_3 C_2 \varepsilon^{\mathrm{Wass}}_{\mathrm{bias}} }{\alpha_2 c_\cp^2}}
    \end{equation}
    Regarding the bound for the absolute bias of $\hat{\sigma}_{\cp^a_s}$, from Lemma \ref{lem:wangthmD1}, we have for contamination uncertainty,
    \begin{equation} 
         \abs{\mathbb{E}\left[\hat{\sigma}_{\cp^a_s}(V) - {\sigma}_{\cp^a_s}(V)\right] } =0
    \end{equation}
    In addition, combining \eqref{eq:sigmabiasboundTV}-\eqref{eq:sigmabiasboundWasserstein} with Lemma \ref{lem:wanglemma9}, we have for  for TV distance uncertainty, 
    \begin{equation} 
         \abs{\mathbb{E}\left[\hat{\sigma}_{\cp^a_s}(V) - {\sigma}_{\cp^a_s}(V)\right] } \leq 24\sqrt{S}(1+\frac{1}{\delta}) 2^{-\frac{N_{\max}}{2}}t_{\mathrm{mix}}
    \end{equation}
    and for Wasserstein distance uncertainty, we have
    \begin{equation} \label{eq:lastequation}
         \abs{\mathbb{E}\left[\hat{\sigma}_{\cp^a_s}(V) - {\sigma}_{\cp^a_s}(V)\right] } \leq 24\sqrt{S} 2^{-\frac{N_{\max}}{2}}t_{\mathrm{mix}}
    \end{equation}
Combining \eqref{eq:firstequation}-\eqref{eq:lastequation} with \eqref{eq:gresult} gives the desired result.
\end{proof}
\end{theorem}

\subsection{Proof of Theorem \ref{thm:gresult}}
We use the result from Theorem \ref{thm:formalgresult}, to set $\mathbb{E}\Bigl[\abs{g_{T}-g^\pi_\cp}\Bigr]  \leq \epsilon$. For contamination uncertainty sets we set 
 $T=\cO\left(\frac{t^2_{\mathrm{mix}}}{\epsilon^2(1-\gamma)^2} \log \frac{t_\mathrm{mix}}{\epsilon(1-\gamma)}\right)$, resulting in $\cO\left(\frac{SAt^2_{\mathrm{mix}}}{\epsilon^2(1-\gamma)^2} \log \frac{t_\mathrm{mix}}{\epsilon(1-\gamma)}\right)$ sample complexity. For TV and Wasserstein uncertainty set, we set 
$N_{\max}=\cO\left(\log \frac{\sqrt{S}t_{\mathrm{mix}}}{\epsilon(1-\gamma)}\right)$ and $T=\cO\left(\frac{t^2_{\mathrm{mix}}}{\epsilon^2(1-\gamma)^2} \log^3\frac{\sqrt{S}t_{\mathrm{mix}}}{\epsilon(1-\gamma)}\right)$, combining with Theorem \ref{thm:sample-complexity}, this would result in $\cO\left(\frac{SAt^2_{\mathrm{mix}}}{\epsilon^2(1-\gamma)^2} \log^4\frac{\sqrt{S}t_{\mathrm{mix}}}{\epsilon(1-\gamma)}\right)$ sample complexity.

\section{Numerical Validations for Semi-Norm Contractions}
In this section, we provide numerical examples that directly verify the one-step strict contraction across the settings studied. These results empirically support the key structural claims used by our analysis.

\subsection{Evaluations of Lemma \ref{lem:seminorm-contraction}} \label{nonrobustvalidation}

Lemma \ref{lem:validseminorm} is the technical backbone for Lemma \ref{lem:seminorm-contraction}, as Lemma \ref{lem:validseminorm} constructs the fixed-kernel semi-norm and provides the one-step contraction for a given $P$. Therefore, we perform numerical evaluations on Lemma \ref{lem:validseminorm} to demonstrate the one-step contraction property for Lemma \ref{lem:seminorm-contraction}.

For a kernel $P$ with stationary distribution $d$, we follow the steps in Appendix \ref{proofspan-contraction} and construct $\lVert\cdot\rVert_{\kp}$ in \eqref{eq:constructionofPnorm} with
\begin{equation}
\alpha=\min\{0.99,\tfrac{1+\rho(Q)}{2}\},\qquad
\epsilon=0.25 (1-\alpha),
\end{equation}
where $Q=P-\mathbf{e}\,d^\top$. Then the one-step contraction factor is $\beta = \alpha+\epsilon < 1$.

To generate ergodic matrices with dimension $n$, let $I_n$ be the identity matrix and $S_n$ be the cyclic shift matrix defined by $S_n e_i = e_{i+1 \bmod n}$. We provide the following four examples:
\begin{itemize}
  \item $P_1 = 0.5\,I_5 + 0.5\,S_5$
  \item $P_2 = 0.6\,I_6 + 0.4\,S_6$
  \item $P_3 = 0.55\,I_7 + 0.45\,S_7$
  \item $P_4 = 0.6\,I_8 + 0.3\,S_8 + 0.1\,S_8^2$
\end{itemize}

We generate $1000$ random unit vectors $x$ and compute each ratio $\displaystyle \frac{\lVert P_i x\rVert_{\kp}}{\lVert x\rVert_{\kp}}$. The empirical results are summarized below.

\begin{table}[h]
\centering
\footnotesize
\setlength{\tabcolsep}{4pt}
\resizebox{\textwidth}{!}{%
\begin{tabular}{lrrrrrrrrrr}
\hline
\textbf{matrix} & $n$ & \textbf{max span ratio} & $\rho(Q)$ & $\alpha$ & $\epsilon$ & $\beta$ &
$\mathrm{ratio}_{\min}$ & $\mathrm{ratio}_{\mathrm{median}}$ &
$\mathrm{ratio}_{\mathrm{p90}}$ & $\mathrm{ratio}_{\max}$ \\
\hline
P1 & 5 & 1 & 0.8090 & 0.9045 & 0.0239 & 0.9284 & 0.3824 & 0.7950 & 0.8077 & 0.8090 \\
P2 & 6 & 1 & 0.8718 & 0.9359 & 0.0160 & 0.9519 & 0.5197 & 0.8510 & 0.8687 & 0.8718 \\
P3 & 7 & 1 & 0.9020 & 0.9510 & 0.0122 & 0.9633 & 0.5861 & 0.8799 & 0.8976 & 0.9014 \\
P4 & 8 & 1 & 0.8700 & 0.9350 & 0.0162 & 0.9513 & 0.4855 & 0.8226 & 0.8604 & 0.8685 \\
\hline
\end{tabular}}
\caption{Empirical one-step contraction ratios for ergodic kernels using the fixed-kernel semi-norm $\lVert\cdot\rVert_{\kp}$.}
\end{table}

\subsection{Evaluations of Theorem \ref{thm:robust_seminorm-contraction}}

Lemma \ref{lem:cpnormcontraction} is the key step for Theorem \ref{thm:robust_seminorm-contraction}, as Lemma \ref{lem:cpnormcontraction} proves a uniform one-step contraction across all $P$ in the uncertainty set. We therefore perform numerical evaluations on Lemma \ref{lem:cpnormcontraction} to demonstrate the one-step contraction property for Theorem \ref{thm:robust_seminorm-contraction} under contamination, total variation (TV), and Wasserstein-1 uncertainty. We select the same $P_1$, $P_2$, $P_3$, and $P_4$ in Appendix \ref{nonrobustvalidation} as four examples of the nominal model.

To numerically approximate $\lVert\cdot\rVert_{\mathcal P}$ defined in \eqref{eq:robustseminormconstruction}, we approximate $\lVert\cdot\rVert_{\mathcal P}$ by (i) discretizing the uncertainty set and (ii) using a finite product to approximate the extremal norm. First, we sample a family $\{P^{(i)}\}_{i=1}^m\subset\mathcal P$ of size $m$ and form
\[
Q_i \;:=\; P^{(i)}-\mathbf e\,d_{P^{(i)}}^\top, \qquad
\hat r \;=\; \max_i \rho(Q_i).
\]
We set $\alpha=\min\{0.99,(1+\hat r)/2\}$ and choose $\epsilon\in(0,1-\alpha)$. To approximate the extremal norm, we build a library of scaled products of the $Q_i$’s up to maximum length $K$: for each $k=0,1,\dots,K$ we draw products $M_{k,j}=Q_{i_k}\cdots Q_{i_1}$; the number of such draws at each $k$ is the “products per length” (denoted \texttt{samples\_per\_k} in the tables). This defines the surrogate
\begin{equation}
\lVert z\rVert_{\mathrm{ext}}^{(K)} = \max_{0\le k\le K,j}\ \alpha^{-k}\lVert M_{k,j}z\rVert_2.
\end{equation}
We then set
\begin{equation}
\lVert x\rVert_{\mathcal P}^{(K)}
=
\max_{i=1,\dots,m} \lVert Q_i x\rVert_{\mathrm{ext}}^{(K)}
+
\epsilon \min_{c\in\mathcal C}\lVert x-c\,\mathbf e\rVert_{\mathrm{ext}}^{(K)}.
\end{equation}

We generate $50$ random unit vectors $x$ for each sampled uncertainty matrix in $\{P^{(i)}\}_{i=1}^m\subset\mathcal P$, and compute the ratios $\displaystyle \frac{\lVert P^{(i)} x\rVert_{\mathcal P}}{\lVert x\rVert_{\mathcal P}}$. The empirical results of the uncertainty sets studied in our settings are summarized below.

\begin{table}[h]
\centering
\footnotesize
\setlength{\tabcolsep}{4pt}
\resizebox{\textwidth}{!}{%
\begin{tabular}{lrrrrrrrrrrrrrrr}
\toprule
\textbf{nominal $\tilde{\kp}$} & $n$ & $\delta$ & $m$ & $K$ & \texttt{samples\_per\_k} & \textbf{max span ratio} & $\hat r$ & $\alpha$ & $\epsilon$ & $\gamma$ & $\mathrm{ratio}_{\min}$ & $\mathrm{ratio}_{\mathrm{median}}$ & $\mathrm{ratio}_{\mathrm{p90}}$ & $\mathrm{ratio}_{\max}$ \\
\midrule
P1 & 5 & 0.15 & 30 & 3 & 25  & 1 & 0.8138 & 0.9069 & 0.0233 & 0.9302 & 0.2210 & 0.6396 & 0.8057 & 0.8134 \\
P2 & 6 & 0.15 & 30 & 3 & 25  & 1 & 0.8807 & 0.9403 & 0.0149 & 0.9553 & 0.2957 & 0.6581 & 0.8630 & 0.8785 \\
P3 & 7 & 0.15 & 30 & 3 & 25 & 1 & 0.9067 & 0.9534 & 0.0117 & 0.9650 & 0.3217 & 0.6683 & 0.8700 & 0.8901 \\
P4 & 8 & 0.15 & 30 & 3 & 25  & 1 & 0.8812 & 0.9406 & 0.0148 & 0.9555 & 0.3739 & 0.5964 & 0.8207 & 0.8624 \\
\bottomrule
\end{tabular}}
\caption{Empirical one-step contraction ratios under contamination uncertainty.}
\end{table}

\begin{table}[h]
\centering
\footnotesize
\setlength{\tabcolsep}{4pt}
\resizebox{\textwidth}{!}{%
\begin{tabular}{lrrrrrrrrrrrrrrr}
\toprule
\textbf{nominal $\tilde{\kp}$} & $n$ & $\delta$ & $m$ & $K$ & \texttt{samples\_per\_k}  & \textbf{max span ratio} & $\hat r$ & $\alpha$ & $\epsilon$ & $\gamma$ & $\mathrm{ratio}_{\min}$ & $\mathrm{ratio}_{\mathrm{median}}$ & $\mathrm{ratio}_{\mathrm{p90}}$ & $\mathrm{ratio}_{\max}$ \\
\midrule
P1 & 5 & 0.15 & 30 & 3 & 25 & 1 & 0.8280 & 0.9140 & 0.0215 & 0.9355 & 0.3239 & 0.7609 & 0.8239 & 0.8280 \\
P2 & 6 & 0.15 & 30 & 3 & 25  & 1 & 0.9013 & 0.9507 & 0.0123 & 0.9630 & 0.4021 & 0.7904 & 0.8862 & 0.9006 \\
P3 & 7 & 0.15 & 30 & 3 & 25  & 1 & 0.9175 & 0.9588 & 0.0103 & 0.9691 & 0.4457 & 0.7918 & 0.8962 & 0.9162 \\
P4 & 8 & 0.15 & 30 & 3 & 25  & 1 & 0.8805 & 0.9403 & 0.0149 & 0.9552 & 0.4497 & 0.7503 & 0.8449 & 0.8739 \\
\bottomrule
\end{tabular}}
\caption{Empirical one-step contraction ratios under total variation (TV) uncertainty.}
\end{table}

\begin{table}[h]
\centering
\footnotesize
\setlength{\tabcolsep}{4pt}
\resizebox{\textwidth}{!}{%
\begin{tabular}{lrrrrrrrrrrrrrrr}
\toprule
\textbf{nominal $\tilde{\kp}$} & $n$ & $\delta$ & $m$ & $K$ & \texttt{samples\_per\_k}  & \textbf{max span ratio} & $\hat r$ & $\alpha$ & $\epsilon$ & $\gamma$ & $\mathrm{ratio}_{\min}$ & $\mathrm{ratio}_{\mathrm{median}}$ & $\mathrm{ratio}_{\mathrm{p90}}$ & $\mathrm{ratio}_{\max}$ \\
\midrule
P1 & 5 & 0.15 & 30 & 3 & 25  & 1 & 0.8184 & 0.9092 & 0.0227 & 0.9319 & 0.3698 & 0.7569 & 0.8141 & 0.8188 \\
P2 & 6 & 0.15 & 30 & 3 & 25  & 1 & 0.8900 & 0.9450 & 0.0138 & 0.9587 & 0.4257 & 0.7889 & 0.8774 & 0.8894 \\
P3 & 7 & 0.15 & 30 & 3 & 25 & 1 & 0.9110 & 0.9555 & 0.0111 & 0.9666 & 0.4262 & 0.7818 & 0.8827 & 0.9080 \\
P4 & 8 & 0.15 & 30 & 3 & 25  & 1 & 0.8758 & 0.9379 & 0.0155 & 0.9534 & 0.4413 & 0.7224 & 0.8518 & 0.8689 \\
\bottomrule
\end{tabular}}
\caption{Empirical one-step contraction ratios under Wasserstein-1 uncertainty.}
\end{table}

\subsection{Interpretations}

Note that for the above tables, \textbf{max span ratio} denotes the largest one‑step span contraction coefficient over the sampled families. This value equals to $1$ for all the settings, meaning no strict one‑step contraction in the span. In contrast, the quantities $\mathrm{ratio}_{\min}$, $\mathrm{ratio}_{median}$, $\mathrm{ratio}_{\mathrm{p90}}$, and $\mathrm{ratio}_{\max}$ summarize the empirical one‑step ratios
$
\frac{||Px||_{\mathcal P}}{||x||_{\mathcal P}}
$ (in the robust case) and
$
\frac{||Px||_{\kp}}{||x||_{\kp}}
$ (in the non-robust case)
computed under the constructed semi-norms over all sampled kernels $P$ and random unit directions $x$. They report, respectively, the minimum, median, 90th percentile, and maximum observed value across those tests. In every table we have $\mathrm{ratio}_{\max}<1$ , so we observe empirical one‑step contraction under our semi-norms even when span does not contract. Moreover, $\mathrm{ratio}_{\max} \leq \gamma$ (robust case) or $\mathrm{ratio}_{\max} \leq \beta$ (non-robust case), which is consistent with the corresponding theoretical contraction factor guaranteed by our constructions.

\section{Some Auxiliary Lemmas for the Proofs}

\begin{lemma}[Theorem IV in \cite{berger1992bounded}]\label{lem:bergerlemmaIV}
    Let $\mathcal{Q}$ be a bounded set of square matrix such that $\rho(Q) < \infty$ for all $Q\in \mathcal{Q}$ where $\rho(\cdot)$ denotes the spectral radius. Then the joint spectral radius of $\mathcal{Q}$ can be defined as 
    \begin{equation}
        \hat{\rho}(\mathcal{Q}) \coloneqq \lim_{k\rightarrow \infty} \sup_{Q_i \in \mathcal{Q}}\rho(Q_k \ldots Q_1)^{\frac{1}{k}} = \lim_{k\rightarrow \infty} \sup_{Q_i \in \mathcal{Q}}\|Q_k \ldots Q_1\|^{\frac{1}{k}}
    \end{equation}
    where $\|\cdot\|$ is an arbitrary norm.
\end{lemma}
\begin{lemma}[Lemma 6 in \cite{zhang2021finite}] \label{lem:zhanglemma6}
    Under the setup and notation in Appendix \ref{setups}, if assuming the noise has bounded variance of \(\mathbb{E}[\,\|w^t\|_{\mathcal{N},\overline{E}}^2 | \mathcal{F}^t] \le A + B\|x^t - x^*\|_{\mathcal{N},\overline{E}}^2\), we have
    \begin{equation}
        \E\left[\|x^{t+1}-x^t\|^2_{s,\overline{E}} | \mathcal{F}^t\right] \leq (16+4B)c_u\rho_2\eta_t^2M_{\overline{E}}(x^t-x^*) + 2A\rho_2\eta_t^2.
    \end{equation}
\end{lemma}
\begin{lemma}[Theorem D.1 in \cite{wang2023model}] \label{lem:wangthmD1}
    The estimator $\hat{\sigma}_{\cp^a_s}(V)$ obtained by \eqref{eq:contaminationquery} for contamination uncertainty sets is unbiased and has bounded variance as follows:
    \begin{equation}
        \E\left[\hat{\sigma}_{\cp^a_s}(V)\right] = {\sigma}_{\cp^a_s}(V), \quad \text{and} \quad \mathrm{Var}(\hat{\sigma}_{\cp^a_s}(V)) \leq  \|V\|^2
    \end{equation}
\end{lemma}
\begin{lemma}[Ergodic case of Lemma 9 in \cite{wang2022near}] \label{lem:wanglemma9}
    For any average-reward MDP with stationary policy $\pi$ and the mixing time defined as 
    \begin{equation} \label{eq:tmix}
    \tau_{\mathrm{mix}}\coloneqq\arg\min_{t \geq 1} \left\{ \max_{\mu_0} \left\| (\mu_0 P_{\pi}^{t})^{\top} - \nu^{\top} \right\|_1 \leq \frac{1}{2} \right\}
    \end{equation}
    where $P_\pi$ is the finite state Markov chain induced by $\pi$, $\mu_0$ is any initial probability distribution on $\mcs$ and $\nu$ is its invariant distribution. If $P_\pi$ is irreducible and aperiodic, then $ \tau_{\mathrm{mix}} < +\infty$ and for the value function $V$, we have
    \begin{equation}
        \|V\|_{\mathrm{sp}} \leq 4\tau_{\mathrm{mix}}
    \end{equation}
\end{lemma}

\end{document}